%% file: main.tex
\documentclass{article}

\usepackage[final]{neurips_2025}

\usepackage[utf8]{inputenc} % allow utf-8 input
\usepackage[T1]{fontenc}    % use 8-bit T1 fonts
\usepackage{hyperref}       % hyperlinks
\usepackage{url}            % simple URL typesetting
\usepackage{booktabs}       % professional-quality tables
\usepackage{amsfonts}       % blackboard math symbols
\usepackage{nicefrac}       % compact symbols for 1/2, etc.
\usepackage{microtype}      % microtypography

\usepackage{hyperref}
\usepackage{url}

\usepackage{graphicx}
\usepackage{subfigure}
\usepackage{wrapfig}
\usepackage{tikz}
\usetikzlibrary{arrows.meta, positioning, quotes}
\input{tiks_figs}

\input{acronyms}
\usepackage{amsmath}
\usepackage{amssymb}
\usepackage{mathtools}
\usepackage{amsthm}
\usepackage{bm}

\usepackage[capitalize,noabbrev]{cleveref}

\usepackage[dvipsnames]{xcolor}
\usepackage{pifont}% http://ctan.org/pkg/pifont
\newcommand{\construction}[1]{\textcolor{green!60!black!90}{#1}}
\newcommand{\ccbm}[1]{\textcolor{blue!80!black!60}{#1}}

\theoremstyle{plain}
\newtheorem{theorem}{Theorem}[section]

\theoremstyle{definition}

\theoremstyle{remark}
\newtheorem{remark}[theorem]{Remark}
\theoremstyle{problem}
\newtheorem{problem}[theorem]{Problem}

\usepackage[textsize=tiny]{todonotes}
\usepackage{multirow}

\definecolor{customgray}{gray}{0.7} % 0 is black, 1 is white -> For tables
\definecolor{backcolour}{rgb}{0.95,0.95,0.92}
\usepackage{comment}
\usepackage{listings}
\usepackage{caption}

\lstnewenvironment{prompt}{%
  \thicklines
  \lstset{
    backgroundcolor=\color{backcolour},
    rulecolor=\color{backcolour!40!black},
    frameround=tttt,
    frame=single,
    basicstyle=\ttfamily\scriptsize,
    basewidth=0.50em,
    keywordstyle=\bfseries,
    xleftmargin=.02\linewidth,
    xrightmargin=.02\linewidth,
    escapeinside={<@}{@>},
    numbers=none
  }}{}

\usepackage[page]{appendix}
\usepackage{etoolbox}
\renewcommand{\appendixtocname}{Table of Contents}

\makeatletter
\let\oldappendix\appendices

\g@addto@macro\tableofcontents{%
  \let\tf@toc@orig\tf@toc
}
\renewcommand{\appendices}{%
  \clearpage
  \renewcommand{\thesection}{\Roman{section}}
  \let\tf@toc\tf@app
  \addtocontents{app}{\protect\setcounter{tocdepth}{3}}
  \immediate\write\@auxout{%
    \string\let\string\tf@toc\string\tf@app
  }
  \oldappendix
}%

\g@addto@macro\endappendices{%
  \let\tf@toc\tf@toc@orig
  \immediate\write\@auxout{%
    \string\let\string\tf@toc\string\tf@toc@orig
  }%
}  

\newcommand{\listofappendices}{%
  \begingroup
  \renewcommand{\contentsname}{\appendixtocname}
  \let\@oldstarttoc\@starttoc
  \def\@starttoc##1{\@oldstarttoc{app}}
  \tableofcontents% Reusing the code for \tableofcontents with different \contentsname and different file handle app
  \endgroup
}

\usepackage[subfigure]{tocloft}
\definecolor{tocblue}{RGB}{0, 122, 204} % Custom blue

\makeatother

\title{Causally Reliable Concept Bottleneck Models}

\author{Giovanni De Felice\thanks{Equal contribution.}\\
Università della Svizzera Italiana\\
\texttt{giovanni.de.felice@usi.ch} \\
\And
Arianna Casanova Flores$^*$\\
University of Liechtenstein\\
\And
Francesco De Santis$^*$\\
Politecnico di Torino\\
\And
Silvia Santini \\
Università della Svizzera Italiana\\
\And
Johannes Schneider\\
University of Liechtenstein\\
\And
Pietro Barbiero\thanks{Equal senior authors.}\\
IBM Research\\
\And
Alberto Termine$^\dagger$\\
Scuola Universitaria Professionale della Svizzera Italiana, 
IDSIA\\
}

\begin{document}
\maketitle

\begin{abstract}
Concept-based models are an emerging paradigm in deep learning that constrains the inference process to operate through human-interpretable variables, facilitating explainability and human interaction. However, these architectures, on par with popular opaque neural models, fail to account for the true causal mechanisms underlying the target phenomena represented in the data. This hampers their ability to support causal reasoning tasks, limits out-of-distribution generalization, and hinders the implementation of fairness constraints. To overcome these issues, we propose \emph{Causally reliable Concept Bottleneck Models} (C$^2$BMs), a class of concept-based architectures that enforce reasoning through a bottleneck of concepts structured according to a model of the real-world causal mechanisms. We also introduce a pipeline to automatically learn this structure from observational data and \emph{unstructured} background knowledge (e.g., scientific literature). Experimental evidence suggests that C$^2$BMs are more interpretable, causally reliable, and improve responsiveness to interventions w.r.t. standard opaque and concept-based models, while maintaining their accuracy. 
\end{abstract}

\section{Introduction}\label{introduction}
In recent years, interpretable neural models have become more popular, achieving performance similar to powerful opaque Deep Neural Networks (DNNs)~\citep{alvarez2018towards,chen2019looks,chen2020concept}. Among these, Concept Bottleneck Models (CBMs)~\citep{koh2020,zarlenga2022concept,yuksekgonul2022post,barbiero2023interpretable} guarantee high expressivity and interpretability by enforcing DNNs to reason through a layer of high-level, human-interpretable variables called \textit{concepts} (e.g., the ``color'' and ``shape'' of an object)~\citep{kim2018interpretability, achtibat2023attribution, fel2023craft}. In CBMs, a neural encoder first maps the raw input to concepts, forming a semantically transparent intermediate representation that is used by a simple decoder for downstream predictions. Beyond transparency, this design allows human experts to intervene on mispredicted concepts at test time to improve downstream task predictions~\citep{espinosa2024learning}.

However, like standard DNN architectures, CBMs remain pure \emph{associative} models~\citep{Pearl2019}: their decision-making process reflects statistical correlations within the data rather than real-world causal mechanisms. As a result, they fail to distinguish between spurious correlations and true causal relationships. Recognizing this distinction is fundamental to achieving a \textit{reliable} scientific understanding, supporting causal reasoning for intervention~\citep{Pearl2009, peters2017elements}, enabling \textit{robust} generalization under distributional shifts, and the implementation of fairness constraints~\citep{scholkopf2021toward, wang2022out}.

To address these limitations, we propose \emph{Causally reliable Concept Bottleneck Models} (C$^2$BMs): a class of concept-based architectures that enforce reasoning through a ``\ccbm{Causal Bottleneck}'' (Fig.~\ref{fig:abstract}) of concepts structured according to a model of the real-world causal mechanisms underlying data generation. C$^2$BMs process information as follows. First, a neural encoder extracts a set of latent representations from raw data. Then, information flows from latent representations through a given causal graph where each node represents an interpretable variable (e.g., ``smoker'', ``bronchitis''). At inference time, the value of each variable is predicted from its causal parents through an interpretable structural equation, parametrized adaptively by a hypernetwork. 

\begin{wrapfigure}{r}{0.6\textwidth}
    \centering
    \vspace{-0.4cm}
    \includegraphics[width=1\linewidth]{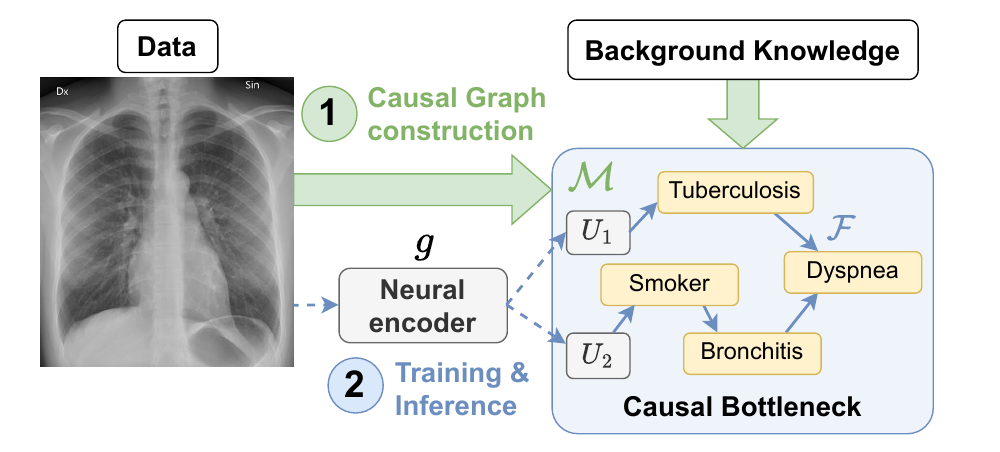}
    \vspace{-0.4cm}
    \caption{\emph{Causally reliable Concept Bottleneck Models} (C$^2$BMs) enforce reasoning through a ``\ccbm{Causal Bottleneck}'' aligned with a model of real-world causal mechanisms obtained from data and background knowledge.} 
    \vspace{-0.3cm}
    \label{fig:abstract}
\end{wrapfigure}
Designing a C$^2$BM requires identifying domain-relevant concepts and specifying their causal relationships, a process that depends heavily on expert knowledge, which could be scarce, costly, or entirely unavailable in practice. To mitigate this reliance and favor agile deployment across domains, we propose a fully automated pipeline (Fig.~\ref{fig:abstract}, \construction{Causal Graph Construction}) for instantiating a C$^2$BM, in which the set of relevant concepts and the causal graph are automatically learned from a mixture of data and \emph{unstructured} background knowledge.

Experimental evidence shows that C$^2$BMs: (i) improve on \textbf{consistency} with real-world causal mechanisms, \textbf{without compromising accuracy} w.r.t. standard DNN models, CBMs, and their extensions~(Sec.~\ref{sec:exp-accuracy-and-rel}); (ii) \textbf{improve interventional accuracy} on downstream concepts with fewer interventions (Sec.~\ref{sec:exp-interventions}); (iii) mitigate reliance on spurious correlations (\textbf{debiasing}, Sec.~\ref{sec:exp-ood}); (iv) permit interventions to remove unethical model behavior and \textbf{meet fairness requirements} (Sec.~\ref{sec:exp-fairness}).

\section{Preliminaries}\label{sec:preliminaries}
We introduce the notation and key formalizations underlying standard CBMs and causal modeling. A more detailed background on causality is provided in App.~\ref{app:background}.

\paragraph{Concept Bottleneck Models.}
CBMs~\citep{koh2020} are interpretable-by-design architectures that explain their predictions using high-level interpretable variables called \textit{concepts}. Standard CBMs decompose prediction into two stages: a neural encoder maps the input $X$ to a set of intermediate concepts $\mathcal{V} = \{V_i\}_{i=1}^C$, and a decoder predicts the target $Y$ from $\mathcal{V}$. This yields:
\begin{equation}
    P(Y, \mathcal{V} \mid X) = \underbrace{P(Y \mid \mathcal{V})}_{\text{decoder}} \, \underbrace{P(\mathcal{V} \mid X)}_{\text{concept encoder}}.
\end{equation}
Concept Embedding Models (CEMs)~\citep{zarlenga2022concept} enhance CBMs by pairing concepts with high-dimensional embeddings of the form $P(\mathcal{U} \mid \mathcal{V}, X)$, where $\mathcal{U} = \{U_i\}_{i=1}^C$. These embeddings are provided to the decoder to predict the target variable $Y$, enabling the model to achieve performance comparable to standard DNN approaches while maintaining semantic interpretability. Critically, traditional decoders rely on a \textit{bipartite structure} assumption, wherein all concepts are treated as direct causes of the target, e.g., $Y = f(V_1, \dots, V_C)$ for CBMs. This assumption is often overly simplistic for real-world problems. Bringing the reasoning of concept-based architectures closer to real-world mechanisms constitutes the main focus of this work.

\paragraph{Causal Reliability.} A model $\mathcal{M}$ is \emph{causally reliable} w.r.t. a target phenomenon $T$ if and only if the structure of $\mathcal{M}$'s decision-making process is consistent with the causal mechanisms underlying $T$ \citep{Termine2023Causal}. Although state-of-the-art DNNs and concept-based models offer high expressivity, they lack causal reliability.

\paragraph{Structural Causal Models.}
The standard framework for modeling causal mechanisms is the \emph{structural causal model} (SCM) \citep{bareinboim2022pearl}. An SCM $\mathcal{M}$ is a tuple $\langle \mathcal{V}, \, \mathcal{U}, \, \mathcal{F}, \, P \rangle$, where:
%\vspace{-0.1cm}
    \begin{itemize}
        \item $\mathcal{V}$ is a set of $C$ \emph{endogenous}
        variables, modeling observable magnitudes of interest;
        \item $\mathcal{U}$ is a set of \emph{exogenous} variables, modeling unobservable magnitudes determined by factors external to $\mathcal{V}$;
        \item $\mathcal{F} = \{f_i\}_{i=1}^C$ is a set of functions such that 
        \begin{equation}
            V_i = f_i( \textsf{PA}_i, \, \mathcal{U}_i) \quad \forall i=1, \dots  , C
        \end{equation}
        %where ${\textsf{PA}_i \subseteq \mathcal{V} \setminus V_i}$, ${\mathcal{U}_i \subseteq \mathcal{U}}$, and the entire set $\mathcal{F}$ forms a mapping from  $\mathcal{U}$ to $\mathcal{V}$. 
        %\textbf{structural equations} describing the causal mechanisms that relate each endogenous variable with its causal parents.
        where ${\textsf{PA}_i \subseteq \mathcal{V} \setminus V_i}$ is the set of the \emph{endogenous} parents of $V_i$, ${\mathcal{U}_i \subseteq \mathcal{U}}$ is an exogenous parent summarizing all the information influencing $V_i$ that is not explicitly represented in $\mathcal{V}$, and the entire set $\mathcal{F}$ forms a mapping from  $\mathcal{U}$ to $\mathcal{V}$. 
        \item $P(\mathcal{U})$ is a joint probability distribution over $\mathcal{U}$.
    \end{itemize}
Each SCM can be associated with a graphical representation in which nodes correspond to the variables, %$\mathcal{U}$ and $\mathcal{V}$,
and edges encode the functional relationships specified by $\mathcal{F}$. Here, we focus on SCMs whose associated graph is a \emph{directed acyclic graph} (DAG)~\citep{pearl1995causal, zaffalon2020structural}. In most cases, the underlying DAG is unknown and must be inferred from observational data, a process known as \textit{causal discovery}~\citep{peters2017elements,zanga2022}. However, methods based solely on observational data cannot generally guarantee the identification of a unique DAG~\citep{peters2017elements}. The set of candidate DAGs can be refined by incorporating additional information, which we refer to as \emph{background knowledge}~\citep{andrews2020completeness, abdulaal2023causal}. This can be drawn from a range of sources, such as human experts, structured repositories of information (e.g., domain ontologies), or ``unstructured'' samples of information (e.g., scientific papers or other documentation).
%SCMs supports \textit{do-interventions}~\citep{Pearl2009, peters2017elements} (App.~\ref{app:pearl-ladder}-~\ref{app:CMAppendix}), which allow to simulate the effects of a \emph{change} in the system, e.g., manually fixing a variable to a specific value.

\section{Related works}
Traditional concept-based architectures impose a strict bipartite structure in which concept neuron activations are assumed to directly cause task outputs~\citep{koh2020,yuksekgonul2022post,kim2023probabilistic,oikarinen2023,yang2023language,barbiero2023interpretable,vandenhirtz2024stochastic}. This strong, often unrealistic assumption can lead to misleading explanations. For example, attributing a lung cancer diagnosis to both a `cough' and `smoker' concept could risk the false interpretation that reducing coughing could reduce cancer risk. Moreover, most CBMs assume independence among concepts, which is unrealistic, as it ignores natural co-occurrences (e.g., `smoke' and `fire') and prevents improvements in one concept from propagating to related concepts during interventions. Stochastic CBM (SCBM)~\citep{vandenhirtz2024stochastic} and Concept Graph Models~\citep{dominici2025causal} attempt to relax this assumption. However, these approaches capture only associations rather than causal relations, making it vulnerable to spurious correlations in the data. To date, no methodology exists for structuring the concept bottleneck according to a reliable causal model.

Recent approaches like DiConStruct~\citep{moreira2024diconstruct}, aim to improve this aspect by generating causal graphs linking concepts to opaque DNN predictions. However, DiConStruct is a \textit{post-hoc} method that may misalign with the original DNN's outputs and relies solely on observational data, neglecting background knowledge and resulting in under-determined causal structures. Other architectures, such as Neural Causal Models~\citep{ke2019learning} and Neural Causal Abstractions~\citep{xia2024neural}, impose even stronger assumptions, requiring access to either the true causal graph or a low-resolution structural causal model, which are impractical in many cases.

\section{Method}\label{sec:Architecture}
\begin{figure}[t]
    \centering
    \includegraphics[width=\linewidth]{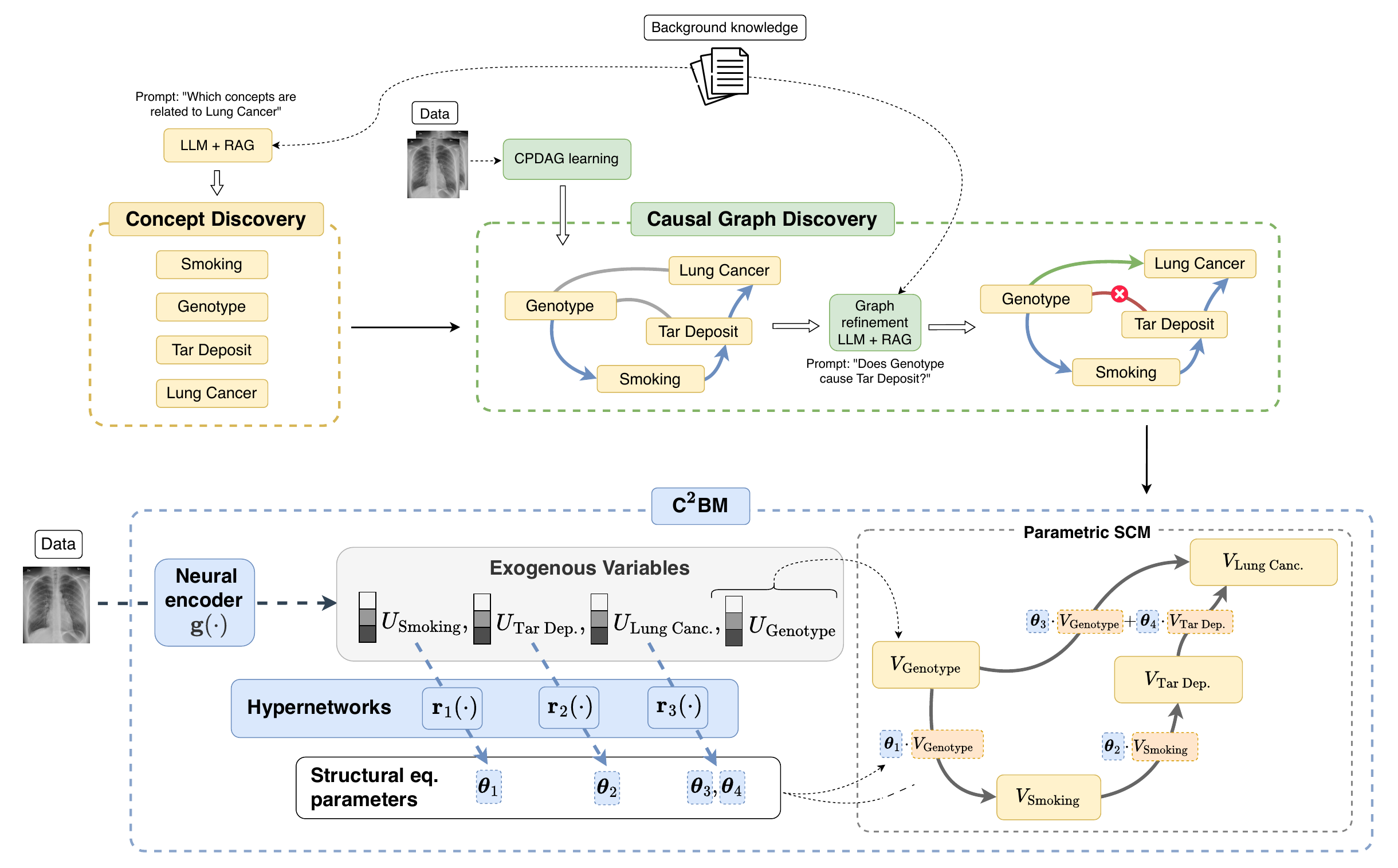}
    %\vspace{-0.3cm}
    \caption{\textbf{Overview of the C$^2$BM fully automated pipeline.} The pipeline consists of three key blocks: (i) Concept discovery: discovery and labeling of the relevant variables $\mathcal{V}$ from background knowledge; (ii) \construction{Causal graph discovery}: discovery of the causal graph by integrating data and background knowledge; (iii) the \ccbm{C$^2$BM model}, comprising a neural encoder and an adaptively parametrized SCM. Once the model is trained, it can support forward and interventional queries about any endogenous variable (e.g., predicting \textit{dyspea}).}
    \label{fig:pipeline}
\end{figure}
In this section, we introduce \emph{Causally reliable Concept Bottleneck Models} (C$^{2}$BMs) and the pipeline we propose to fully automate its instantiation, learning, and functioning (Fig.~\ref{fig:pipeline}).

\subsection{Causally reliable Concept Bottleneck Models}\label{sec:Blueprint}
A C$^2$BM is a concept-based architecture that leverages the formalism of SCMs to structure a ``\ccbm{causal bottleneck of concepts}''. More formally, let: (i) $X$ denoting a random variable modeling (possibly noisy) input features; (ii)  $\mathcal{V} = \{V_i\}_{i=1}^C$ be a set of $C$ semantically meaningful variables modeled as \emph{endogenous} variables; (iii) $\mathbf{G}$ be a DAG connecting variables in $\mathcal{V}$. A C$^2$BM is a neural architecture implementing the tuple $\langle \mathbf{g}, \mathcal{M}_{\boldsymbol\Theta} \rangle$ where:
%\vspace{-0.2cm}
\begin{itemize}  %\setlength\itemsep{0.1em}
    \item $\mathbf{g}(\cdot)$ is a \textit{neural encoder} modeling a probability distribution $P(\mathcal{U}|X)$ over a set of latent, high-dimensional embeddings $\mathcal{U} = \{U_i\}_{i=1}^C$, representing the \emph{exogenous} variables;
    \item $\mathcal{M}_{\boldsymbol\Theta}$ is a \textit{parametric} SCM $\langle \mathcal{V}, \mathcal{U}, \mathcal{F}_{\boldsymbol\Theta}, P(\mathcal{U} | X) \rangle$ (see Sec.~\ref{sec:preliminaries}), where we assume a parametric form for the functions' set. Specifically, the structure of the functions is determined by the connectivity of $\mathbf{G}$, and the parameters $\boldsymbol\Theta$ are predicted from $\mathcal{U}$ by a hypernetwork.
\end{itemize}
The information flowing along a C$^2$BM can be described as follows (Fig.~\ref{fig:pipeline}, right side). First, the values of the exogenous variables $\mathcal{U}$ are predicted using the exogenous encoder $\mathbf{g}(\cdot)$ from $X$. Then, the information flows along the SCM $\mathcal{M}_{\boldsymbol\Theta}$ starting from the endogenous sources (predicted from $\mathcal{U}$) down to the sinks. At each subsequent level of the causal graph, the values of each $V_i$ are predicted from the values of its \emph{parents} $\textsf{PA}_i$ based on the relative structural equation $f_i\in \mathcal{F}_{\boldsymbol{\Theta}}$. 
%Both the encoder and the meta-model parameters are learned end-to-end from the input data.

\subsection{Model instantiation}\label{sec:ModelConstruction}
 To instantiate a C$^2$BM, one requires a labeled dataset $\mathcal{D}$ annotated for all variables in $\mathcal{V}$, as well as a DAG capturing the causal relationships among $\mathcal{V}$. However, such resources may be inaccessible, problem-specific, or heavily dependent on human expertise. To address this challenge, we propose a fully automated pipeline that enables the use of C$^2$BM also in such complex scenarios. Our approach extracts the necessary components from: (i) a potentially unlabeled dataset $\mathcal{D}_x$; (ii) a potentially unstructured repository of background knowledge $\mathcal{K}$.

Our pipeline (see Fig.~\ref{fig:pipeline}) addresses the following sub-problems: \construction{(i) causal graph construction}, which includes concept discovery, concept labeling (Sec.~\ref{sec:concept-discovery}), and causal graph discovery (Sec.~\ref{sec:graph-discovery}); and \ccbm{(ii) training of the neural parameters} of the encoder and the hypernetwork determining the structural equations (Sec.~\ref{sec:equation-learning}). In the following, we outline our implementation for each sub-problem. Specifically, building C$^2$BM's individual prerequisites will be mostly based on prior work. Note that integrating them into a coherent, automated pipeline is instead part of this paper's contributions.
\begin{remark}
Concepts and causal graph constitute an input for C$^2$BM, which could remain agnostic to how they are obtained, e.g., provided by human experts. Notably, alternative or novel approaches may be employed, provided they solve the same problems~\citep{loula2025syntactic}.
\end{remark}

\subsubsection{Concept discovery and labeling}\label{sec:concept-discovery}
\begin{problem}[Concept Discovery]
Given a dataset of i.i.d. samples $\mathcal{D}_x = \{\mathbf{x}_i\}_{i=1}^N$, and a background knowledge repository $\mathcal{K}$ relative to a task, identify a set of relevant variables $\mathcal{V}$.
\end{problem}
In the CBM community, automated concept discovery and labeling using Large Language Models (LLMs) has become a standard solution when human supervision is unavailable~\citep{oikarinen2023, yang2023language, srivastava2024vlg, yamaguchi2025explanation}. In our implementation, we follow the label-free CBM approach from~\citet{oikarinen2023}, where concepts are discovered by querying an LLM for those most relevant to the task. We then apply a filtering procedure to retain only concepts that meet criteria such as brevity, distinctiveness (i.e., not too similar to each other or the target), and presence in the training data. 

Once $\mathcal{V}$ are selected, we label the dataset with variable annotations $\mathcal{D} = \{(\mathbf{x}_i, \mathbf{v}_i)\}_{i=1}^N$ to supervise concept learning. To do so, we adopt a strategy similar to~\citet{oikarinen2023}, leveraging a pre-trained contrastive vision-language model, such as CLIP~\citep{radford2021learning}. This projects both data samples and the discovered concept names into a shared embedding space and computes their alignment to generate concept labels. Full implementation details are provided in App.~\ref{app:concept_discovery}.

\subsubsection{Causal graph discovery} 
\label{sec:graph-discovery}
\begin{problem}[Causal Discovery]
Let $\mathbf{G}^*$ be the true, unknown, graph over a set of variables $\mathcal{V}$ from which a dataset $\mathcal{D}$ was generated. The causal discovery problem consists in recovering $\mathbf{G}^*$ from the observed dataset $\mathcal{D}$~\citep{zanga2022}.
\end{problem}
As anticipated in Sec.~\ref{sec:preliminaries}, a promising direction for addressing this problem is to combine standard causal discovery algorithms with knowledge-base querying. In our pipeline, we focus on a well-known class of methods that recover an equivalence class of graphs from data, referred to as the \emph{Markov equivalence class} (MEC)~\citep{Spirtes2001, Pearl2009, zanga2022}.%, which encode the same conditional independencies among variables
This class can be compactly represented as a \emph{Completed Partially Directed Acyclic Graph} (CPDAG), an extension of a DAG where edges remain unoriented when there is insufficient evidence in the data to infer causal direction. This is a desirable property as it prevents incorrect (spurious) orientations based on the data alone. Specifically, we apply the \emph{Greedy Equivalence Search} (GES)~\citep{chickering2002} algorithm, which we found performed well empirically (see App.~\ref{app:ablation_cd}). Then, we leverage a pre-trained LLM to assess each undirected edge in the CPDAG, orienting the ones corresponding to true causal relationships while discarding those originating from spurious correlations. To improve the robustness and generalizability of this approach, we pair the LLM with a \textit{Retrieval Augmented Generation} (RAG) technique, which is known to reduce hallucinations and can provide problem-specific knowledge. To further improve robustness and performance, we repeat each query 10 times, selecting the most frequent outcome~\citep{wang2022self}. Implementation details are provided in App.~\ref{app:causal_discovery}.

\subsection{Structural equations and model training}
\label{sec:equation-learning}
\begin{problem}[Learning structural equations]
    Let $\mathcal{E}$ be the set of edges describing causal connections in a $\mathrm{DAG}$ $\mathbf{G}$ connecting variables in $\mathcal{V}$. Given $\mathcal{D} = \{(\mathbf{x}_i, \mathbf{v}_i)\}_{i=1}^N$ and $\mathcal{E}$, predict the parameters $\boldsymbol\Theta$ of the structural equations $\mathcal{F}_{\boldsymbol{\Theta}}$. 
\end{problem}
We model the structural functions $f_i \in \mathcal{F}_{\boldsymbol{\Theta}}$ describing the causal mechanisms relating each endogenous variable$\,$\footnote{Root variables are predicted from the exogenous variables $\mathcal{U}$ using a neural network. Further details are provided in App.~\ref{app:detailedarchitecture}.} to its endogenous parents as weighted linear sums, i.e., for each $i$: 
\begin{equation}\label{eq:model_streq}
    V_i = \sum_{V_j \in \textsf{PA}_i} [\boldsymbol{\theta}_{f_i}]_j V_j
\end{equation}
where ${\textsf{PA}_i}$ denotes the set of endogenous parents of $V_i$$\,$\footnote{We refer to $V_i$ as a variable to allow for a more general formalization. In a classification setting, concepts assume categorical values; hence $V_i$ represents the probability of a concept state activation.}. For the parameters ${\boldsymbol\theta_{f_i}}$, we do not learn a single parameterization; instead, these are adaptively inferred for each different realization of $X$, by a \emph{hypernetwork} $\mathbf{r}(\cdot)$~\citep{ha2017hypernetworks,barbiero2023interpretable,debot2024interpretable}, based on the values of the exogenous variables $\mathcal{U}$ and the graph connectivity. In our implementation, we consider separate hypernetworks $\mathbf{r}_i(\cdot)$ (e.g., independent DNNs), each taking as input a separate exogenous variable:
\begin{equation}\label{eq:meta-model}
    \boldsymbol\theta_{f_i}=\mathbf{r}(\mathcal{E}, \mathcal{U})_i \, := \, \mathbf{r}_i(U_i) = \mathbf{r}_i(\mathbf{g}(X)_i).
\end{equation}
The design in Eq.~\ref{eq:model_streq} and \ref{eq:meta-model} improves both interpretability and expressivity. To aid (mechanistic) interpretability, structural equations take a linear form (Eq.~\ref{eq:model_streq}): the value of each is a weighted linear combination of its parents. The adaptive re-parameterization of the equation’s weights performed by the hypernetwork $\mathbf{r}(\cdot)$ allows the model to also approximate non-linear relationships among endogenous variables (see App.~\ref{app:detailedarchitecture}, which also includes a proof that C$^2$BM is a universal approximator, regardless of the underlying causal graph). This idea is in line with existing literature on interpretability~\citep{ribeiro2016should, alvarez2018towards} and concept-based methods~\citep{barbiero2023interpretable}~\footnote{This idea also aligns with~\citet{balke1994counterfactual} and~\citet{zaffalon2020}, where exogenous variables are used to represent relationships between endogenous variables when the structural equations are unknown.} %In our case, the meta-model adaptively infers the parameters of the structural equations based on the exogenous variables.}. In our implementation, $\mathbf{g}()$ is modeled using a DNN tailored to the input data (e.g., a CNN combined with an MLP). Additionally, we considered separate structural equation meta-models $\mathbf{r}_i()$, one for each endogenous variable, each implemented as an independent DNN (e.g., an MLP).
%Within the SCM framework, once the exogenous variables are fixed, the structural equations deterministically define the values of the endogenous variables from their endogenous parents. Analogously, in our model, the meta-model uses the exogenous variables to generate a specific realization of the weights $[\boldsymbol{\theta}_{f_i}]_j$, thereby instantiating a deterministic relationship among the endogenous variables.

%\begin{remark}
%    Notice that the overall decision boundary is a piecewise linear function formed as a mixture of locally linear decision boundaries. This idea is in line with existing XAI~\citep{ribeiro2016should} and concept-based methods~\citep{barbiero2023interpretable} and guarantees that structural equations are both highly interpretable and expressive.
%\end{remark}

\paragraph{Model training.}

\begin{wrapfigure}{r}{0.30\textwidth}
    \vspace{-0.4cm}
    \centering
    \includegraphics[width=\linewidth]{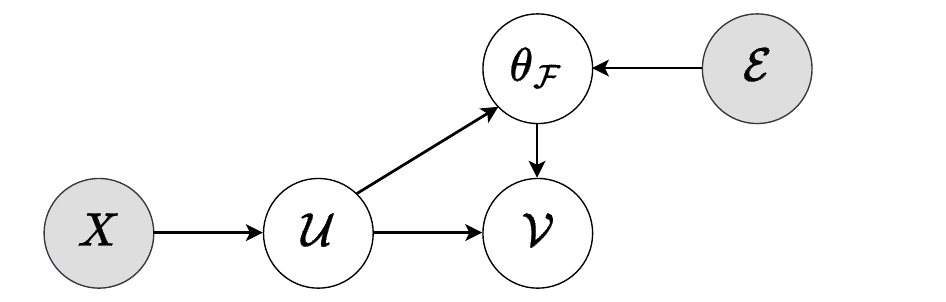}
    \caption{Probabilistic graphical model of C$^2$BM inference.}
    \vspace{-0.8cm}
    \label{fig:pgm}
\end{wrapfigure}

The training of C$^2$BM consists of learning the neural parameters of the encoder $\mathbf{g}(\cdot)$ and hypernetwork $\mathbf{r}(\cdot)$ \textit{end-to-end} from the input data. We formalize this by modeling the joint conditional distribution $P(\mathcal{V}, \mathcal{U}, \boldsymbol{\Theta} \mid X, \mathcal{E})$ which factorizes as:
{\small
\begin{equation}
P(\mathcal{V}, \mathcal{U}, \boldsymbol{\Theta} \mid X, \mathcal{E}) =  \overbrace{P(\mathcal{V} \mid \mathcal{U}; \boldsymbol{\Theta})}^{\text{endogenous}} \, \overbrace{P(\boldsymbol{\Theta}\mid \mathcal{E}, \mathcal{U})}^{\text{structural equation}} \, \overbrace{P(\mathcal{U} \mid X)}^{\text{exogenous}}
\end{equation}
}

where: $P( \mathcal{U} \mid X)$ represents the exogenous encoder $\mathbf{g}(\cdot)$; %: implemented using a set of  NNs (e.g., CNNs + MLP); %mapping low-level features into exogenous variables.
$P(\boldsymbol{\Theta} \mid \mathcal{E}, \mathcal{U})$ represents the hypernetwork $\mathbf{r}(\cdot)$ predicting the structural equations' parameters using the given causal connections and the exogenous variables; $P(\mathcal{V} \mid \mathcal{U};\boldsymbol{\Theta})$ represents a causally reliable classifier leveraging the structural equations $\mathcal{F}_{\boldsymbol\Theta}$ to predict the values of the endogenous variables. Under the Markov condition imposed by the C$^2$BM causal graph, the causally-reliable classifier can be re-written as a product of independent distributions i.e., 
\begin{equation}
    P(\mathcal{V} \mid \mathcal{U}; \boldsymbol{\Theta}) = \prod_{i} P(V_i \mid \textsf{PA}_i, U_i; \, \mathbf{r}(\mathcal{E}, \mathcal{U})_i)
\end{equation}
 where $U_i = \mathbf{g}(X)_i$ and $\mathbf{r}(\mathcal{E}, \mathcal{U})_i = \boldsymbol{\theta}_{f_i}$. From the above factorization, we derive the C$^2$BM's training objective, which corresponds to maximizing the empirical log-likelihood of the training data:
{\small
\begin{equation}
    \boldsymbol{\phi}^* 
    % &= \arg\max_{\theta} \sum_{(x_j, v_j) \in \mathcal{D}} \mathcal{L}(\theta; \mathcal{D}) =\\
    = \arg\max_{\boldsymbol\phi} \sum_{\mathcal{D}} \sum_{i =1}^{C} \log P(V_i \mid \textsf{PA}_i, U_i; \, \mathbf{r}(\mathcal{E}, \mathcal{U})_i)
\end{equation}}
\begin{remark}
We clarify that we do not claim to identify the true structural functions. Instead, C$^2$BM numerically approximates the outcomes as if they were generated by the underlying (unknown) structural equations. This approximation, along with C$^2$BM's DAG, is sufficient to compute reliable interventions, which is a core objective in the concept-based community~\citep{Poeta2023, steinmann2024learning}.
\end{remark}

%\textbf{Summary.}
%To summarize, our pipeline enables the retrieval of all the necessary elements for constructing and training a C$^2$BM without necessarily relying on human expertise. It is important to notice that the pipeline is flexible---it can incorporate human expertise when available and can continuously be updated to reflect advancements in the research domains relevant to each component. We provide further details about the architecture in App.~\ref{app:detailedarchitecture}.

\section{Experimental evaluations}\label{sec:experiments}
We evaluate the performance of the proposed C$^2$BM pipeline. Experiments are conducted across different datasets and settings, allowing for the investigation of the following aspects: classification accuracy~(Sec.~\ref{sec:exp-accuracy-and-rel}), causal reliability~(Sec.~\ref{sec:exp-accuracy-and-rel}), accuracy under ground-truth interventions~(Sec.~\ref{sec:exp-interventions}), debiasing~(Sec.~\ref{sec:exp-ood}), and fairness~(Sec.~\ref{sec:exp-fairness}). App.~\ref{app:extraexp} provides additional results and ablations.
%\begin{itemize}
%    \item \textbf{Task accuracy and Interpretability}: C$^2$BMs match the performances 
%    of standard neural networks and standard concept-based architectures, while ensuring causal reliability.
%    \item \textbf{Ground-Truth Interventions in In-Distribution and Out-of-Distribution Settings}: C$^2$BMs improve the performances of standard concept-based architectures in both in-distribution and out-of-distribution settings.
%    \item \textbf{Fairness}: C$^2$BMs enable causal reasoning and fairness analyses.
%\end{itemize}

The considered datasets include both synthetic and real-world benchmarks. 
%All the datasets will be used for evaluating the task and concept accuracy of our model with respect to standard baselines. 
%Additionally, different subsets of datasets will be used for the other experiments. 
%This modification allow us to introduce a simple form of bias, enabling the evaluation of our architecture's performance in out-of-distribution settings.
As synthetic datasets, we sample $10^4$ points from each of the five following discrete Bayesian networks available from the \texttt{bnlearn} repository~\citep{scutari2010}: \textbf{Asia}~\citep{asia}, \textbf{Sachs}~\citep{sachs2005causal}, \textbf{Insurance}~\citep{binder1997}, \textbf{Alarm}~\citep{alarm}, and \textbf{Hailfinder}~\citep{abramson1996}. We include \textbf{cMNIST}, a variant of the original dataset \citep{lecun2010} in which the image data are colored according to custom rules. Additionally, we consider three real-world datasets: \textbf{CelebA}~\citep{liu2015}, a facial recognition dataset labeled with different binary facial attributes; \textbf{CUB}$_\mathbf{C}$, a custom version of the original bird image dataset~\citep{he2019fine} from which we select a subset of concepts and define new ones to introduce deeper causal relationships; \textbf{Siim-Pneumothorax}~\citep{you2023}, containing chest X-ray images annotated with a single label indicating the presence of pneumothorax, without additional concepts or their annotations. To generate them, we follow the label-free approach outlined in Sec.~\ref{sec:concept-discovery}. Exhaustive details on all datasets are given in App.~\ref{app:datasets}.

The performance of the proposed pipeline is investigated alongside an opaque neural baseline predicting the task variable only (\textbf{OpaqNN}) and established state-of-the-art (SOTA) concept-based architectures, namely: \textbf{CBM}~\citep{koh2020}, with linear and non-linear decoder; \textbf{CEM}~\citep{zarlenga2022concept}; and \textbf{SCBM}~\citep{vandenhirtz2024stochastic}. 
Hyperparameters have been selected via an independent search for each dataset–model pair based on performance on the validation set. Further details on each model's architecture and hyperparameters can be found in App.~\ref{app:exp-details}. Note that all baselines, except for OpaqNN, provide
concept-based explanations for their predictions and allow concept interventions at test-time. This excludes architectures such as Self-Explainable Neural Networks~\citep{alvarez2018towards} and Concept Whitening~\citep{chen2020concept} as they do not offer a clear mechanism for intervening on their concept bottlenecks. We also excluded other CBM baselines such as Probabilistic CBMs~\citep{kim2023probabilistic}, Post-hoc CBMs~\citep{yuksekgonul2022post}, Label-free CBMs~\citep{oikarinen2023,yang2023language}, as all of them share the same limitation of vanilla CBMs and CEMs: the causal graph is fixed and bipartite.  Python code for reproducing all experiments is provided alongside the submission as supplementary material.

%\textbf{Metrics.} We evaluate the different parts of the pipeline using diifferent metrics.
%\ari{Se la parte label-free e' interna alla nostra pipeline, non dovremmo anche misurare la qualita' dei concetti?}
%First, we measure the quality of the graph generated by the causal discovery algorithm with the application of llm and rag using a modification of the Structural Hamming distance ***, which assigns different weights to different types of errors --- e.g., penalizing an undirected arc in place of a directed one less than an arc with an incorrect orientation.

%\vspace{-0.4cm}
\subsection{Task accuracy and causal reliability}\label{sec:exp-accuracy-and-rel}
\input{tables/accuracy}
Our initial experiment evaluates task accuracy. For each dataset, we designate a predefined single variable as the prediction \textit{task}. All models except OpaqNN are trained to predict the task while simultaneously learning to fit the remaining concepts. Tab.~\ref{table:accuracy} presents the task accuracy for all evaluated models (see App.\ref{app:concept_accuracy} for concept accuracy). To further assess model expressiveness, we also evaluate task accuracy on modified versions of the \textit{Asia} and \textit{Alarm} datasets, where selected concepts (App.~\ref{app:datasets}) are intentionally removed to create a stronger bottleneck.

%To ensure consistency, the average is restricted to concepts deemed relevant for C$^2$BM, i.e., those included in the subgraph of the task's ancestors, as determined by the causal discovery block in our pipeline.

\paragraph{C$^2$BM achieves comparable or higher accuracy to non-causally reliable models (Tab.~\ref{table:accuracy}).}
Our evaluation shows that C$^2$BM achieves robust accuracy across datasets, matching the performance of the expressive models OpaqueNN and CEM. Notably, as the concept bottleneck is reduced, C$^2$BM retains expressivity by leveraging exogenous variables to propagate residual information from the input. 
This is in contrast with CBMs implementing a hard bottleneck. 
%Importantly, C$^2$BM's interpretability and causal reliability are achieved without sacrificing on expressivity.
%this is achieved without compromising, and in fact enhancing, the interpretability typically associated with concept-based architectures, see App.\ref{app:interpretability}.

%Specifically, it performs comparably to OpaqNN and CEM, while outperforming CBMs on datasets with the highest number of concepts (\textit{Insurance}, \textit{Hailfinder}, and \textit{Pneumothorax}).

\paragraph{C$^2$BM improves on causal reliability (Tab.~\ref{table:hamming}).}
C$^2$BM captures a rich causal structure that aligns well with real-world dependencies. We quantitatively assess this alignment by comparing the learned and true causal graphs in synthetic datasets. Tab.~\ref{table:hamming} reports two metrics: a structural Hamming distance (detailed in App.~\ref{app:hamming}) and the number of incorrect edges, computed after causal discovery (CD) and refinement via LLM queries. Metrics for the simplistic graphs from CBMs (all concepts are treated as mutually independent and direct causes of the task) are reported for reference. Results indicate that integrating CD with background knowledge produces a causal graph that is more accurately aligned with the true structure. Notably, on the \textit{Sachs} dataset, %for the Hailfinder dataset, this approach reduces the number of mistaken edges to 22, compared to the 117 errors observed in the graph derived from concept-based models (CB), achieving a relative reduction in Hamming distance from 69 to 22.
\input{tables/hamming}
the integration of background knowledge enables to correctly identify 10 additional edges w.r.t. CD alone. Detailed ablation studies on causal graph discovery methods and LLM types are provided in App.~\ref{app:ablation_cd}-\ref{app:ablation-llm}-\ref{app:ablation-rag}. To further validate the quality of the learned causal graph, App.~\ref{app:true-graph} demonstrates that C$^2$BM achieves comparable task accuracy using either the learned or the true graph. 
%App.~\ref{app:ablation-llm}-\ref{app:ablation-rag} provide a short ablation study on the employed LLM and RAG.
%Beyond performance, C$^2$BM provides a significantly higher level of interpretability compared to baseline models. Traditional CBMs and CEMs rely on an overly simplistic causal structure where all concepts are treated as mutually independent and direct causes of the task. This assumption fails to capture the intricate causal dependencies underlying the magnitudes represented by concepts. In contrast, C$^2$BM models a richer causal structure, better aligned with real-world dependencies. This alignment can be quantitatively assessed by comparing the learned causal graph to the true one in synthetic datasets. Tab.~\ref{table:hamming} presents two metrics: a custom structural Hamming distance (detailed in App.\ref{app:hamming}) and the ratio of mistaken edges. These metrics are computed after the causal discovery (CD) step and following the querying of the LLM to examine and refine the unoriented edges. The results highlight how incorporating background knowledge with data enhances the discovery of these causal relationships. \gio{improve discussion} To further validate the quality of the learned causal graph, App.~\ref{app:true-graph} demonstrates that C$^2$BM achieves comparable task accuracy using either the learned or the true graph.
When considered together, the results in Tab.~\ref{table:accuracy}-\ref{table:hamming} highlight C$^2$BM's ability to improve on causal reliability without compromising expressivity and performance.

\subsection{Ground-truth interventions}\label{sec:exp-interventions}
\begin{figure*}[t]
    \centering
    \includegraphics[width=\linewidth]{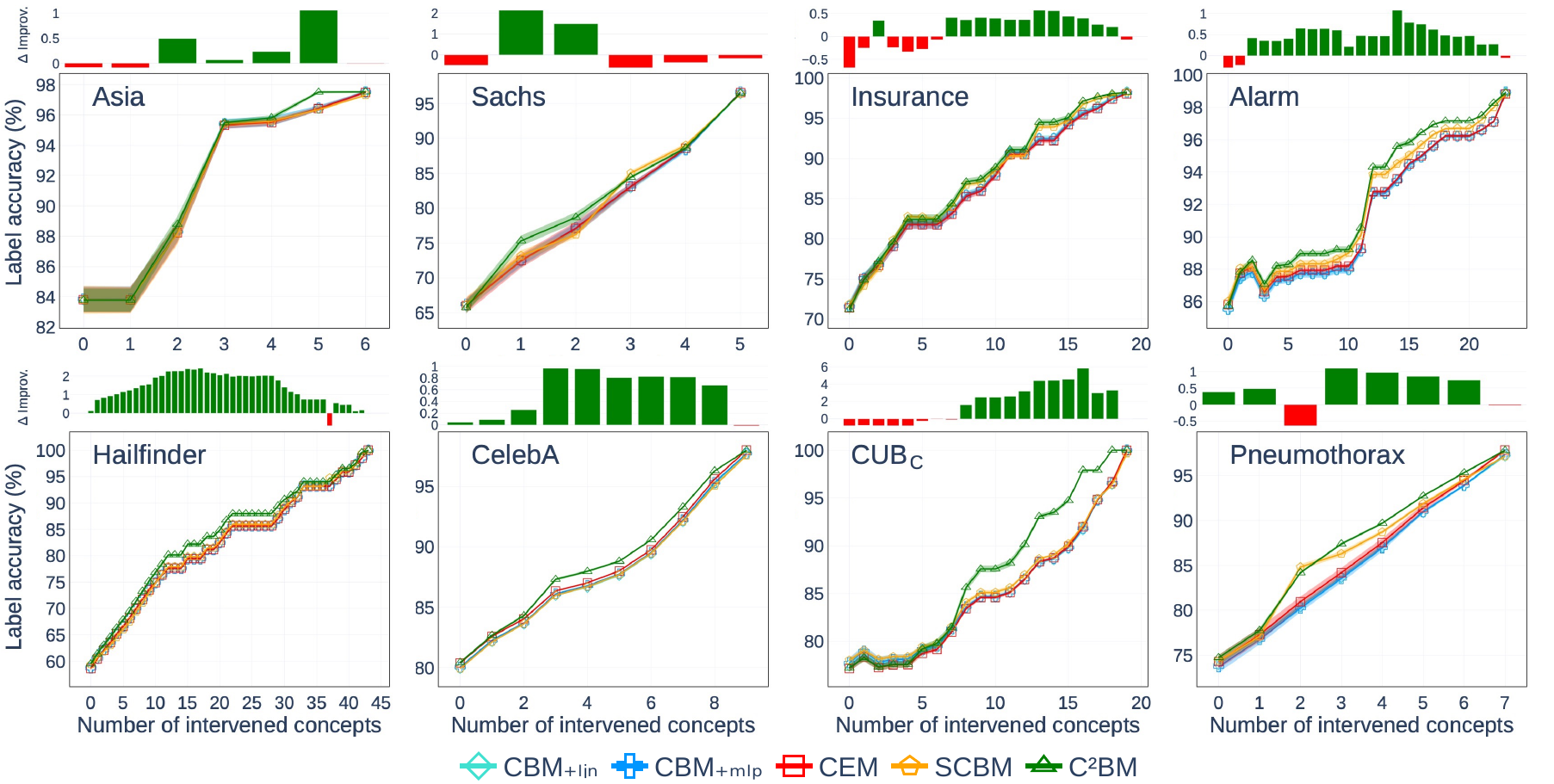}
    %\vspace{-0.2cm}
    \caption{Label accuracy (\%) on downstream variables (task included) after intervening on concepts up to progressively deeper levels in the graph hierarchy. Summit plots show the difference of C$^2$BM's accuracy w.r.t. the best-performing baseline. Uncertainties represent $2$ sample mean $\sigma$ across 5 runs.}
    \label{fig:interventions}
\end{figure*}
After training all models on the same classification task as in Sec.~\ref{sec:exp-accuracy-and-rel}, we test their responsiveness to ground-truth interventions, i.e., replacing predicted concepts with ground-truth values~\footnote{Ground-truth interventions can be seen as a special case of causal \emph{do}-interventions (see App. \ref{app:pearl-ladder}), where variables are set to their ground-truth values.}. This simulates a form of human intervention in a deployed model. Following each intervention, we compute the average accuracy over all variables (concepts and task) prediction. As for the policy, we intervene on random concepts within progressively deeper levels in the hierarchy defined by the true graph. This constitutes the only intervention policy aligned with real-world causal-effect relationships. When the true graph is unavailable, we use the one generated by our pipeline.

%\vspace{-0.1cm}
\paragraph{C$^2$BM improves accuracy on downstream concepts with fewer interventions~(Fig.~\ref{fig:interventions}).}
Our findings, reported in Fig.~\ref{fig:interventions}, demonstrate that C$^2$BM achieves higher accuracy improvements with fewer interventions compared to alternative models. This advantage stems from two key properties of C$^2$BM: (i) unlike other baselines that do not account for connections among concepts, interventions on an upstream concept in C$^2$BM directly influence \textbf{all} downstream nodes, potentially enhancing the predictions of their values; (ii) unlike SCBM, the effects of interventions in C$^2$BM are restricted to concepts that are causally related, rather than altering concept values due to spurious correlations.

\subsection{Debiasing}\label{sec:exp-ood}
We hypothesize that a real-world-aligned causal bottleneck can reduce reliance on spurious correlations. To test this, we use \textit{cMNIST}, where digit \textit{Color} is correlated with \textit{Parity} during training (all the odd digits are green). 
At test time, the correlation among \textit{Color} and \textit{Parity} is reversed (all the even digits are green), introducing a distribution shift that challenges generalization. The neural encoders in all models still struggle with out-of-distribution (OOD) generalization. Therefore, as expected, all models including C$^2$BM, fail to extrapolate correctly, capturing the artificial correlation. However, their enforced reasoning is different. All baselines retain the concept-task connections, perpetuating the color-parity shortcut. In contrast, C$^2$BM detects the color-parity edge through causal discovery but correctly removes it via graph refinement with the LLM block. This is reflected in large differences in performance after concept interventions, which alleviate (or even remove) the impact of the encoder.

\begin{wrapfigure}{r}{0.4\textwidth}
    \vspace{-0.3cm}
    \centering
    \includegraphics[width=0.8\linewidth]{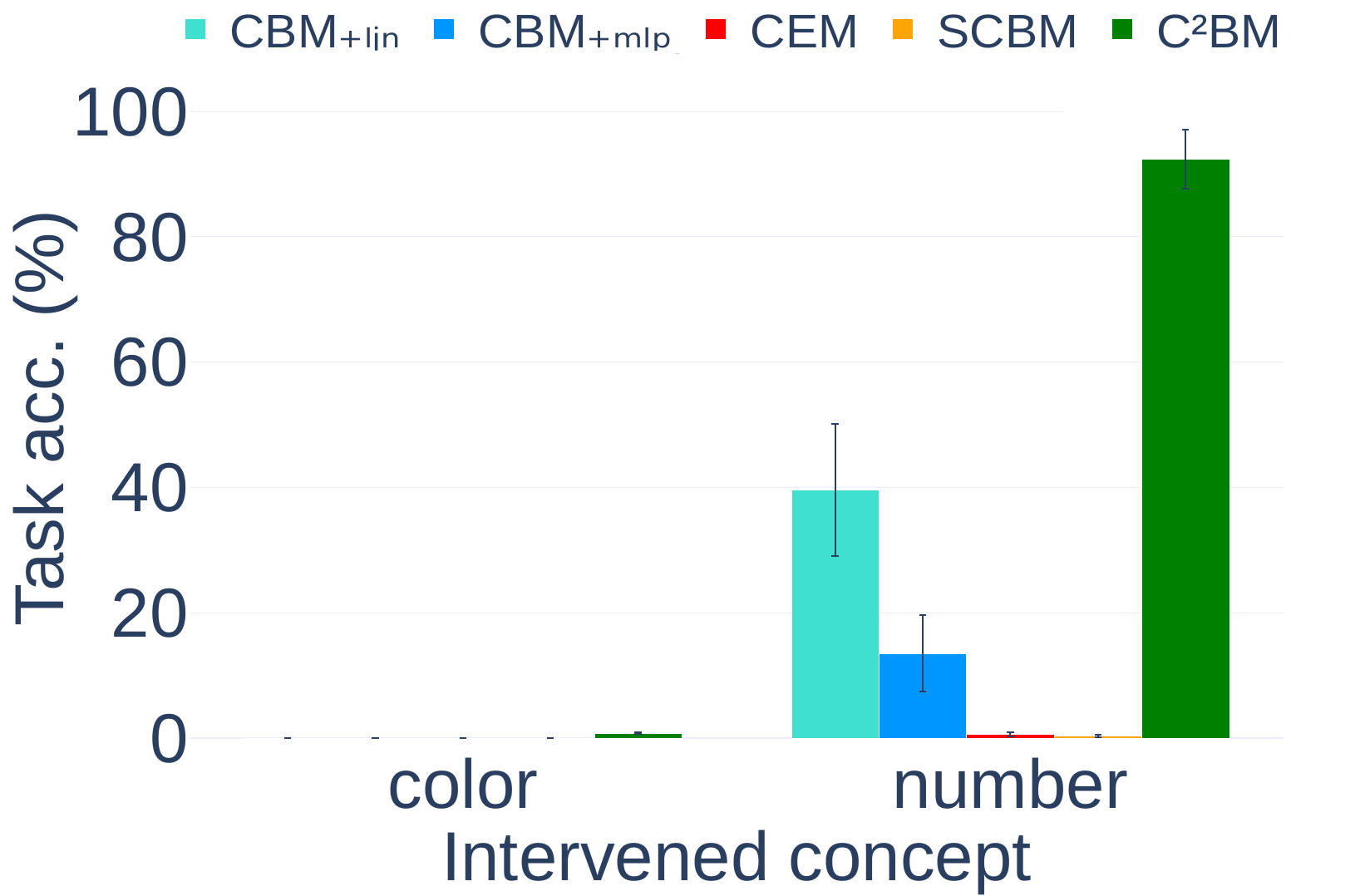}
%    \vspace{-0.4cm}
    \caption{Biased ColorMNIST dataset. Task accuracy on \textit{Parity} after ground-truth
    interventions.}
    %\vspace{-0.1cm}
    \label{fig:colormnist_ood_SI_on_y}
\end{wrapfigure}
\paragraph{Causal bottlenecks mitigate reliance on spurious correlations~(Fig.~\ref{fig:colormnist_ood_SI_on_y}).}
Fig.~\ref{fig:colormnist_ood_SI_on_y} shows the accuracy on \textit{Parity} after ground-truth interventions on each concept. As expected, color has no effect across all models, confirming the learned bias. 
Notably, C$^2$BM exhibits the largest improvement when intervening on the \emph{number} concept (achieving $\sim90$\% accuracy), as its causal structure isolates color and strengthens training on the correct feature. 
%While ColorMNIST provides a simple and insightful case study, future work should investigate datasets with deeper complex causal structures to further strengthen these findings. 
Although a comprehensive analysis of OOD robustness is beyond the scope of this paper, our results suggest that the C$^2$BM pipeline holds promise in improving generalization in biased settings.

\subsection{Fairness}\label{sec:exp-fairness}
We create a customized \textit{CelebA} dataset to evaluate the influence of sensitive attributes on the decision-making of the model. Specifically, we consider a hypothetical scenario in which an actor with a specific physical attribute is required for a specific role. However, the hiring manager has a strong bias toward \emph{Attractive} applicants. To model this, we define two custom attributes: \emph{Qualified}, indicating whether an applicant meets the biased hiring criteria, and \emph{Should be Hired}, which depends on both \emph{Qualified} and the task-specific requirement (\emph{Pointy Nose}). In our fairness analysis, we aim to intervene and remove any such unfair bias. %This is possible with the framework of do-interventions.
\paragraph{C$^2$BMs permit interventions to meet fairness requirements~(Tab.~\ref{table:fairness}).}  
We measure the Causal Concept Effect (CaCE)~\citep{goyal2019} of \emph{Attractive} on \emph{Should be Hired} before and after blocking the only path between them, i.e., performing an intervention on \emph{Qualified}. Tab.~\ref{table:fairness} shows that C$^2$BM is the only model able to successfully remove the influence, achieving a post-intervention CaCE of 0.0\%. 
This difference stems from the model architectures: 
\begin{wraptable}{r}{0.60\columnwidth}
    \input{tables/fairness_wrap}
\end{wraptable}
in CBM, CEM, and SCBM, all concepts are directly connected to the task, meaning interventions on one concept cannot block the influence of others. In contrast, C$^2$BM enforces a structured causal bottleneck allowing for interventions to fully override the effects of parent nodes and block information propagation through the intervened node. This highlights C$^2$BM’s ability to enforce causal fairness by eliminating biased pathways.

% \note{Compare with DNNs (assumption that all magnitudes of interest are independent), CBMs (assumption that magnitudes of interest have a bipartite structure): in both cases poor causal reliability and interpretability. DiConStruct (no access to unstructured knowledge): theoretically impossible to retrieve the true causal graph (identifiability). Neural Causal Models (assumption that they have access to the true causal graph as an inductive bias), Neural Causal Abstractions (assumption to have access to low-level SCM), Vanilla causal models (assumption to have access to structural equations and do not handle unstructured input data): in these cases assumptions are too strong in practice}

\section{Conclusions}\label{sec:conclusions}
We presented C$^2$BM, a concept-based model advancing prior research by structuring the bottleneck of concepts according to a model of causal relations between human-interpretable variables. By combining observational data with background knowledge, C$^2$BMs improves on causal reliability without compromising performance. This offers several additional benefits, e.g., improved interventional accuracy, robustness to spurious correlations, and fairness.

\paragraph{Applications and broader impact.}
We speculate C$^2$BM has the potential to significantly narrow the hypothesis space in complex scientific domains where even a panel of human experts might struggle to identify or exclude plausible hypotheses worth testing. For instance, constructing the hypothesis space to design clinical trials accounting for the influence of environmental conditions on gastrointestinal biochemistry requires deep interdisciplinary knowledge---not only in environmental science and biochemistry, but also in genetics, microbiology, nutrition science, and epidemiology, among others. In such settings, C$^2$BM may integrate human scientific knowledge across these diverse fields to construct a comprehensive causal graph, thereby supporting experts in systematically excluding hypotheses that are inconsistent with the integrated body of evidence. This can help accelerate interdisciplinary scientific discovery, while reducing experts' cognitive burden.

\paragraph{Limitations.}
C$^2$BM requires a robust prior knowledge base, access to pre-trained models (LLMs for causal discovery), and well-crafted prompts for querying these LLMs. Biases within the knowledge base or in the observational data can reduce the system’s reliability (though C$^2$BM still outperforms the selected baselines). Furthermore, the SOTA in causal structural learning currently faces scalability limitations, which also constrain C$^2$BM. As these techniques become more scalable and robust, C$^2$BM stands to benefit. Finally, encoder embeddings used to construct exogenous variables can move out of distribution, and since the encoder’s OOD performance is not guaranteed, the SCM’s OOD performance may likewise be affected.

\paragraph{Future works.}
Future directions include a deeper investigation of OOD generalization with C$^2$BMs, an extensive exploration of their role in causal inference (e.g, \emph{counterfactual queries} extending~\citep{}, see App.~\ref{app:background}), and the identification of optimal intervention policies. Moreover, incorporating PAGs \citep{Zhang2008} would enable modeling of hidden confounders. %Lastly, extending our approach to spatio-temporal causal structures could open new possibilities, particularly in \emph{Physics-informed machine learning} \cite{Karniadakis2021}.

\section*{Acknowledgments}
This work is supported by the Swiss National Science Foundation (SNSF) through the grant 205121\_197242 for the project ``PROSELF: Semi-automated Self-Tracking Systems to Improve Personal Productivity'' and the Hasler Foundation under the Project ID: 2024-05-15-70. AC and JS acknowledge support from FFF of the University of Liechtenstein grant lbs\_24\_08. PB acknowledges support from the Swiss National Science Foundation Postdoctoral Fellowships IMAGINE (No. 224226) and has received funding from the Research Foundation Flanders (FWO, G033625N). AT acknowledges the support by the Hasler Foundation grant Malescamo (No. 22050), and the Horizon Europe grant Automotif (No. 101147693).

\bibliography{neurips_2025_references}
\bibliographystyle{neurips_2025}

%%%%%%%%%%%%%%%%%%%%%%%%%%%%%%%%%%%%%%%%%%%%%%%%%%%%%%%%%%%%%%%%%%%%%%%%%%%%%%%
%%%%%%%%%%%%%%%%%%%%%%%%%%%%%%%%%%%%%%%%%%%%%%%%%%%%%%%%%%%%%%%%%%%%%%%%%%%%%%%
% NEURIPS CHECKLIST
%%%%%%%%%%%%%%%%%%%%%%%%%%%%%%%%%%%%%%%%%%%%%%%%%%%%%%%%%%%%%%%%%%%%%%%%%%%%%%%
%%%%%%%%%%%%%%%%%%%%%%%%%%%%%%%%%%%%%%%%%%%%%%%%%%%%%%%%%%%%%%%%%%%%%%%%%%%%%%%

\section*{NeurIPS Paper Checklist}

\begin{enumerate}

\item {\bf Claims}
    \item[] Question: Do the main claims made in the abstract and introduction accurately reflect the paper's contributions and scope?
    \item[] Answer: \answerYes{} % Replace by \answerYes{}, \answerNo{}, or \answerNA{}.
    \item[] Justification: The main claims of our paper have been highlighted in bold in the introduction (Sec.~\ref{introduction}) and summarized in the abstract. Sec.~\ref{sec:experiments} provide experimental evidence that support the claims presented in the abstract and introduction.

\item {\bf Limitations}
    \item[] Question: Does the paper discuss the limitations of the work performed by the authors?
    \item[] Answer: \answerYes{} % Replace by \answerYes{}, \answerNo{}, or \answerNA{}.
    \item[] Justification: The limitations of our work are clearly stated and discussed in the Conclusions (Sec.~\ref{sec:conclusions}).

\item {\bf Theory assumptions and proofs}
    \item[] Question: For each theoretical result, does the paper provide the full set of assumptions and a complete (and correct) proof?
    \item[] Answer: \answerYes{} % Replace by \answerYes{}, \answerNo{}, or \answerNA{}.
    \item[] Justification: Although the main contribution of our work is not theoretical, we also included a simple proof regarding the expressivity of our model in App.~\ref{app:detailedarchitecture}.

    \item {\bf Experimental result reproducibility}
    \item[] Question: Does the paper fully disclose all the information needed to reproduce the main experimental results of the paper to the extent that it affects the main claims and/or conclusions of the paper (regardless of whether the code and data are provided or not)?
    \item[] Answer: \answerYes{} % Replace by \answerYes{}, \answerNo{}, or \answerNA{}.
    \item[] Justification: We provide extensive details for reproducing the experiments in Appendices \ref{app:concept_discovery}, \ref{app:causal_discovery}, \ref{app:llm}, \ref{app:datasets}, and \ref{app:exp-details}. Moreover, we uploaded a .zip file containing our code.

\item {\bf Open access to data and code}
    \item[] Question: Does the paper provide open access to the data and code, with sufficient instructions to faithfully reproduce the main experimental results, as described in supplemental material?
    \item[] Answer: \answerYes{} % Replace by \answerYes{}, \answerNo{}, or \answerNA{}.
    \item[] Justification: We use freely available datasets and provide instructions on how to download and preprocess them in App.~\ref{app:datasets}. Moreover, we have uploaded a .zip file containing our code. YAML configurations are available within the code to reproduce experiments.

\item {\bf Experimental setting/details}
    \item[] Question: Does the paper specify all the training and test details (e.g., data splits, hyperparameters, how they were chosen, type of optimizer, etc.) necessary to understand the results?
    \item[] Answer: \answerYes{} % Replace by \answerYes{}, \answerNo{}, or \answerNA{}.
    \item[] Justification: Details on dataset splitting can be found in App.~\ref{app:datasets}, while information on training is provided in App.~\ref{app:exp-details}.

\item {\bf Experiment statistical significance}
    \item[] Question: Does the paper report error bars suitably and correctly defined or other appropriate information about the statistical significance of the experiments?
    \item[] Answer: \answerYes{}
    % Replace by \answerYes{}, \answerNo{}, or \answerNA{}.
    \item[] Justification: All relevant experiments are executed with 5 random seeds. The reported results are averaged, and uncertainty is expressed for all experimental results as 2 standard errors of the sample mean. This is stated in the main results' captions.

\item {\bf Experiments compute resources}
    \item[] Question: For each experiment, does the paper provide sufficient information on the computer resources (type of compute workers, memory, time of execution) needed to reproduce the experiments?
    \item[] Answer: \answerYes{} % Replace by \answerYes{}, \answerNo{}, or \answerNA{}.
    \item[] Justification: Information about the computational resources is provided in the App.~\ref{app:exp-details}.
    
\item {\bf Code of ethics}
    \item[] Question: Does the research conducted in the paper conform, in every respect, with the NeurIPS Code of Ethics \url{https://neurips.cc/public/EthicsGuidelines}?
    \item[] Answer: \answerYes{} % Replace by \answerYes{}, \answerNo{}, or \answerNA{}.
    \item[] Justification: We have read the NeurIPS Code of Ethics and ensured that our paper conforms to them. Specifically, our experiments do not include human subjects and the content of our paper does not contain personally identifiable information.
    

\item {\bf Broader impacts}
    \item[] Question: Does the paper discuss both potential positive societal impacts and negative societal impacts of the work performed?
    \item[] Answer: \answerYes{} % Replace by \answerYes{}, \answerNo{}, or \answerNA{}.
    \item[] Justification: The broader impact of our paper is discussed in the Conclusions (Sec.~\ref{sec:conclusions}).
    
\item {\bf Safeguards}
    \item[] Question: Does the paper describe safeguards that have been put in place for responsible release of data or models that have a high risk for misuse (e.g., pretrained language models, image generators, or scraped datasets)?
    \item[] Answer: \answerNA{}
    % Replace by \answerYes{}, \answerNo{}, or \answerNA{}.
    \item[] Justification: In our paper, we propose a methodological extension to a class of models that are predominantly used in scientific research. All datasets constitute standard choices in the community.

\item {\bf Licenses for existing assets}
    \item[] Question: Are the creators or original owners of assets (e.g., code, data, models), used in the paper, properly credited and are the license and terms of use explicitly mentioned and properly respected?
    \item[] Answer:  \answerYes{} % Replace by \answerYes{}, \answerNo{}, or \answerNA{}.
    \item[] Justification: We cite the original creators or owners of every asset used in the paper, for example, see the references for the datasets in App.~\ref{app:datasets}. All our employed datasets and baselines are open-source. Datasets licenses are as follows: bnlearn datasets (CC-BY-SA), MNIST (CC BY-SA), CUB (CC0: Public Domain), CelebA (Creative Commons NonCommercial license), Pneumothorax (CC BY 4.0).
    As for the baselines: we implemented CEM and CBM from scratch and provide code alongside the submission. For SCBM we use the implementation available at \url{https://github.com/mvandenhi/SCBM/tree/main}.
    

\item {\bf New assets}
    \item[] Question: Are new assets introduced in the paper well documented and is the documentation provided alongside the assets?
    \item[] Answer: \answerYes{} % Replace by \answerYes{}, \answerNo{}, or \answerNA{}.
    \item[] Justification: We propose a new predictive algorithm and a pipeline to instantiate it. Detailed descriptions of how they work are provided throughout the main text and in the appendices. Moreover, we have uploaded a .zip file containing our code.

\item {\bf Crowdsourcing and research with human subjects}
    \item[] Question: For crowdsourcing experiments and research with human subjects, does the paper include the full text of instructions given to participants and screenshots, if applicable, as well as details about compensation (if any)? 
    \item[] Answer: \answerNA{} % Replace by \answerYes{}, \answerNo{}, or \answerNA{}.
    \item[] Justification: The paper does not involve crowdsourcing nor research with human subjects.

\item {\bf Institutional review board (IRB) approvals or equivalent for research with human subjects}
    \item[] Question: Does the paper describe potential risks incurred by study participants, whether such risks were disclosed to the subjects, and whether Institutional Review Board (IRB) approvals (or an equivalent approval/review based on the requirements of your country or institution) were obtained?
    \item[] Answer: \answerNA{} % Replace by \answerYes{}, \answerNo{}, or \answerNA{}.
    \item[] Justification: The paper does include experiments involving human subjects.

\item {\bf Declaration of LLM usage}
    \item[] Question: Does the paper describe the usage of LLMs if it is an important, original, or non-standard component of the core methods in this research? Note that if the LLM is used only for writing, editing, or formatting purposes and does not impact the core methodology, scientific rigorousness, or originality of the research, declaration is not required.
    %this research? 
    \item[] Answer: \answerNA{} % Replace by \answerYes{}, \answerNo{}, or \answerNA{}.
    \item[] Justification: We used LLMs just for editing purposes and grammar checks. Notably, in our proposed framework, an LLM is employed as a secondary element to assist a causal discovery algorithm for causal discovery refinement.

\end{enumerate}

%%%%%%%%%%%%%%%%%%%%%%%%%%%%%%%%%%%%%%%%%%%%%%%%%%%%%%%%%%%%%%%%%%%%%%%%%%%%%%%
%%%%%%%%%%%%%%%%%%%%%%%%%%%%%%%%%%%%%%%%%%%%%%%%%%%%%%%%%%%%%%%%%%%%%%%%%%%%%%%
% APPENDIX
%%%%%%%%%%%%%%%%%%%%%%%%%%%%%%%%%%%%%%%%%%%%%%%%%%%%%%%%%%%%%%%%%%%%%%%%%%%%%%%
%%%%%%%%%%%%%%%%%%%%%%%%%%%%%%%%%%%%%%%%%%%%%%%%%%%%%%%%%%%%%%%%%%%%%%%%%%%%%%%
\newpage
\begin{appendices}
\noindent\rule{\textwidth}{0.5pt}\par
\vspace{-0.2cm}
\listofappendices
\noindent\rule{\textwidth}{0.5pt}

\section{Extended background on causality}\label{app:background}
\subsection{Pearl's framework of causality}\label{app:pearl-ladder}
Contemporary research in causal inference and causal machine learning predominantly builds on the framework introduced by \citet{Pearl2009} (see also \cite{Pearl2019}). This framework centers on an agent’s ability to reason about underlying causal mechanisms, going beyond mere statistical associations observed in data. Pearl formalizes this capacity through the notion of answering different types of \emph{what-if} questions, structured into a three-level hierarchy (see Figure~\ref{fig:Hierarchy}).

%This framework analyses causality as an \emph{epistemological problem} related to the ability of an agent to \emph{understand} the causal-generating mechanisms beyond regularities one observes in data. This understanding comes in degrees and can be measured in terms of the agent's ability to answer \emph{what-if} kind of questions concerning a given target-phenomenon of interest. 
%Pearl identifies three typologies of what-if questions organised in a hierarchy of levels (see, Fig.~\ref{fig:Hierarchy}) such that questions at level $j,\,\, (j = 1,2,3)$ can be answered only if information from level $i,\,\, (i\, <\, j)$  is available \citep[~p.1]{Pearl2019}:

\begin{figure}[h!]
    \centering
    \includegraphics[width=0.4\linewidth]{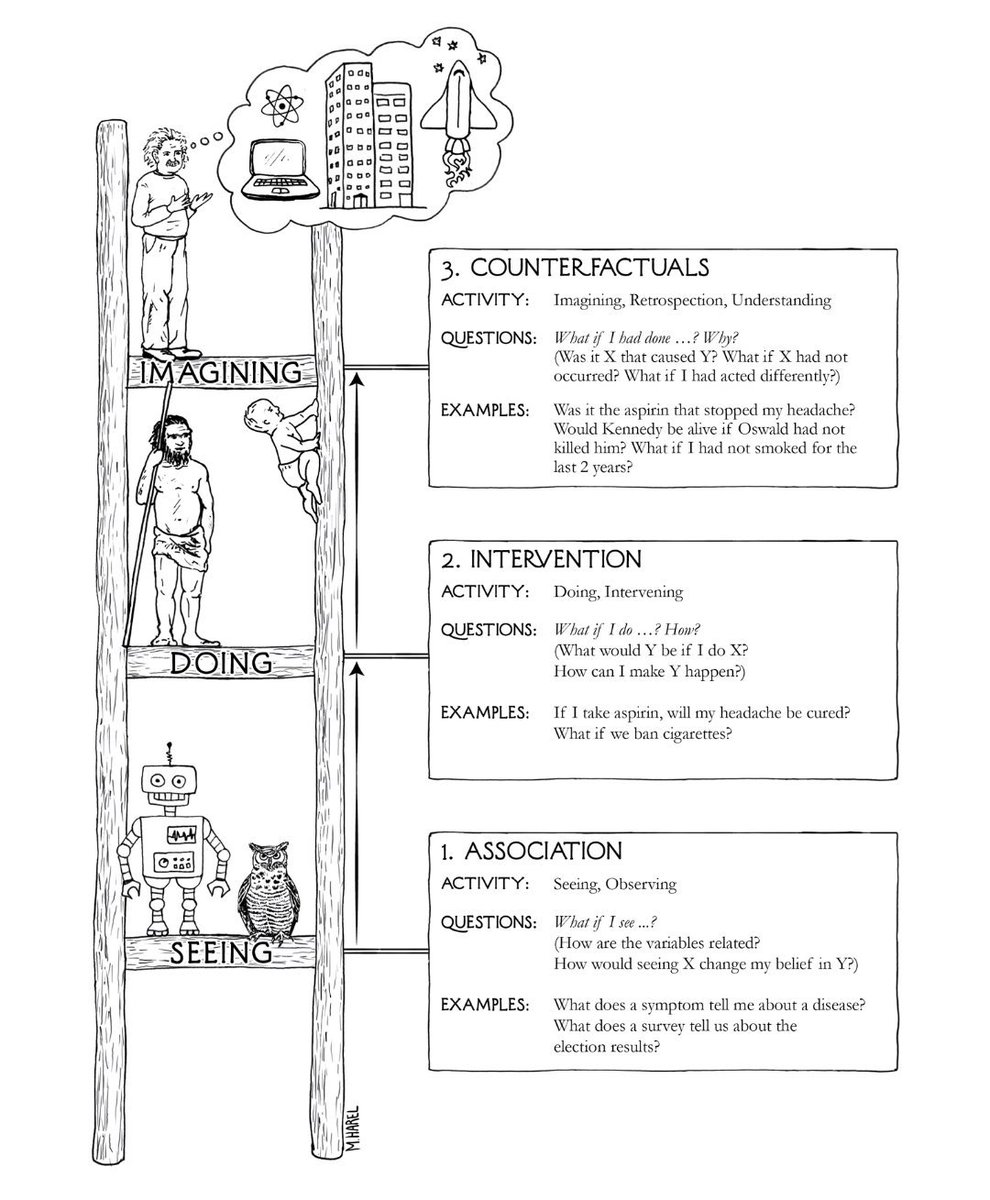}
    \caption{\footnotesize Pearl's Hierarchy as represented in \citet{Pearl2018}.}
    \label{fig:Hierarchy}
\end{figure}

\begin{itemize}
    \item \textbf{Level 1: Observational questions} (``what is the value of $Y$ if I \emph{observe} $X=x$?''). These questions rely purely on statistical associations within the data. They can be answered using standard tools such as conditional probabilities or expectations, with no need for causal assumptions.
    %These questions constitute the bottom layer of Pearl's hierarchy and 
    %require no causal understanding of the problem at stake.
    \item \textbf{Level 2: Interventional questions} (``what is the value of $Y$ if I \emph{do} $X=x$?''). These questions focus on the effects of actively intervening on variables, rather than passively observing them. The expression $\text{do}(X = x)$ denotes such an intervention, where $X$ is externally set to $x$, overriding its natural causes. 
    
    To compute such an intervention, one typically relies on a graphical model representing causal dependencies among variables, e.g., a directed acyclic graph (DAG). The intervention is then modeled by removing all incoming edges to $X$, effectively simulating the setting of $X$ independently of its original causes. This modified graph is then used to compute the post-intervention distribution $P(Y \mid \text{do}(X = x))$.
    
    \item \textbf{Level 3: Counterfactual questions} (``What would have been the value of $Y$ if I had observed $X=x'$ instead of $X=x$?''). These questions focus on states of affairs alternative to the actual reality.
    They constitute the upper layer of Pearl's hierarchy and their computation requires a detailed knowledge of the causal mechanisms relating the variables of the target-problem.
\end{itemize}

While standard machine learning models~\footnote{This excludes models specifically designed for causal inference, such as those in the field of \emph{causal machine learning} \citep{kaddour2022causal}.} can generally handle only observational questions, reasoning at interventional and counterfactual levels necessitates dedicated causal models and inference methods~\citep{Pearl2009, peters2017elements}. 

\subsection{Causal opacity}
\begin{problem}[Causal Opacity]\label{def:CausalOpacity}
    Given a DNN model $\mathcal{M}$ and an user $A$, we say that $\mathcal{M}$ is \emph{causally-opaque} with respect to $A$ whenever $A$ is not capable to understand the inner causal structure of $\mathcal{M}$'s decision-making process. The opposite of causal opacity is \emph{causal transparency}.
\end{problem}

Note that \textbf{causal transparency does not presuppose or imply causal reliability}, the two issues, although related, are indeed very different.
Consider for example a concept-based model $\mathcal{M}_1$, where the graph representing how the concepts are connected to the final task is shown in the left panel of the figure below (a). This model is trained to predict whether a given colored image represents an even or odd number (variable $P$, for ``parity'') passing through a bottleneck of two interpretable concepts, namely ``number'' ($N$) and ``color'' ($C$).
The structure of $\mathcal{M}_1$'s decision-making process is causally-transparent, including two edges connecting the final task $P$ with $N$ and $C$ respectively.
However, this causal structure is not consistent with the causal structure of the world (b), for which $C$ and $P$ are clearly independent concepts 
%\ari{perche'? dipende dai dati che consideri. Puoi costruirti un dataset in cui il colore influenza la parita'. Secondo me va trovato un altro esempio, qui parli di immagini colorate quindi teoricamente le puoi costruire come vuoi}.
Therefore, $\mathcal{M}_1$ cannot be considered \emph{causally reliable}, although it is \emph{causally transparent}.

\[
\resizebox{.5\textwidth}{!}{\TWOGRAPHS}
\]

Causal opacity partially depends on another opacity issue of central relevance for our analysis, i.e., \emph{semantic opacity}.

\begin{problem}[Semantic Opacity]
    Given a model $\mathcal{M}$ and a user $A$, we say that $\mathcal{M}$ is \emph{semantically opaque} to $A$ if and only if $\mathcal{M}$'s decision-making process is based on features that do not possess any interpretable meaning for $A$.
\end{problem}

Semantic opacity is addressed by concept-based architectures, such as concept-bottleneck models and their extensions, which we discuss in the main paper.

\section{Concept discovery details}\label{app:concept_discovery}
CBMs-like architectures tipically rely on labeled data for each concept, a costly process that can hinder their practical adoption. Label-free Concept Bottleneck Models (Label-free CBMs)~\citep{oikarinen2023} address this issue by automatically generating concepts and assign concept labels using pre-trained models. %The process begins with generating candidate concepts via LLM prompts tailored to the task. Then, the concepts are filtered for relevance and quality. Next, internal features of a standard neural network are projected onto the concept space using CLIP (a vision-language model), followed by training a sparse final layer that predicts output classes based on the concept activations.

In our paper, we apply an approach similar to the one followed in~\cite{oikarinen2023} to the Siim-Pneumothorax dataset, which lacks concepts and concept annotations. Using GPT-4o, we first generate candidate concepts with a specific prompt (see App.~\ref{app:prompts}), and apply a multi-stage filtering process:
\begin{itemize}
    \item discard too long concepts (>50 character);
    \item filter out concepts too similar to class labels or to each other. Specifically, we use the CXR-CLIP model -- a CLIP-based model pretrained on medical imaging datasets~\citep{johnson2019, irvin2019, wang2017, you2023} -- to encode both concepts and class labels, and discard any concept with cosine similarity > 0.9 to either a class label or another concept;
    \item discard concepts that are not sufficiently present in the training data. To this end, we also compute image embeddings using CXR-CLIP, and remove concepts whose maximum cosine similarity with images in the training set is < 0.2.
\end{itemize}
We then annotate images by computing their similarity to each remaining concept using CXR-CLIP embeddings, and binarize the resulting scores via 2-means clustering.

\section{Causal graph discovery details}\label{app:causal_discovery}

\subsection{Causal discovery algorithms}
%As detailed in Section \ref{sec:preliminaries}, the core framework for modeling causal mechanisms is based on Structural Causal Models (SCMs). 

%In particular, we specifically focus on SCMs where the associated graphs are DAGs. In particular, a common assumption is that there is a one-to-one correspondence between conditional independence in the data and graphical separation conditions, known as \emph{d-separations}, among the variables in the graph \citep{peters2017elements}.

Algorithms for addressing the causal discovery problem can be broadly classified into two main categories~\citep{peters2017elements}:\footnote{Here, we restrict our discussion to causal discovery algorithms that assume the set of observed variables sufficiently captures the relevant causal influences. For more complex scenarios, see, e.g., \citep{peters2017elements}.}

%\emph{independence-based methods}, which assume a correspondence between conditional independence in data and graphical separation conditions among variables and employ it to infer the graph structure, and \emph{score-based methods}, which define a scoring function and find the class of graphs that maximize it. More specifically:

\begin{itemize}
\item \textbf{Independence-based methods}. These methods assume a correspondence between conditional independence in the data and graphical separation among variables, leveraging this relationship to infer the underlying graph structure.  Typically, they recover a class of DAGs that are equivalent with respect to conditional independencies, which can be compactly represented by a CPDAG.
\item \textbf{Score-based methods.} These methods define a scoring function over potential graph structures and search for the graph that maximizes it, often using criteria such as the \emph{Bayesian Information Criterion} (BIC) \citep{peters2017elements}. The results from score-based methods are often comparable to those from independence-based approaches, as graphs that violate conditional independencies tend to result in poor model fits.
\end{itemize}

For our experiments, we adopted the \emph{Greedy Equivalence Search} (GES) algorithm~\citep{chickering2002}, a score-based method that performed well in our setting, as shown in Table~\ref{table:cd_ablation} in App.~\ref{app:ablation_cd}. 

GES is a two-phase greedy algorithm that searches over equivalence classes of DAGs (CPDAGs). It begins with the empty graph and iteratively adds edges that yield the greatest improvement in the scoring function. Once a local optimum is reached --- where no addition improves the score --- the algorithm enters a backward phase, greedily removing edges that most increase the score. The core idea is to navigate the space of CPDAGs through local transformations (edge additions and deletions), using the score as a guide to optimize structure learning.

In our implementation, we use the GES algorithm provided by the \texttt{causal-learn} Python library~\citep{zheng2024}, employing the \emph{Bayesian Dirichlet equivalent uniform} (BDeu) scoring criterion for discrete variables~\citep{heckerman1995learning, chickering2002}.

%The $\mathtt{\backslash onecolumn}$ command above can be kept in place if you prefer a one-column appendix, or can be removed if you prefer a two-column appendix.  Apart from this possible change, the style (font size, spacing, margins, page numbering, etc.) should be kept the same as the main body.
%%%%%%%%%%%%%%%%%%%%%%%%%%%%%%%%%%%%%%%%%%%%%%%%%%%%%%%%%%%%%%%%%%%%%%%%%%%%%%%
%%%%%%%%%%%%%%%%%%%%%%%%%%%%%%%%%%%%%%%%%%%%%%%%%%%%%%%%%%%%%%%%%%%%%%%%%%%%%%%

\subsection{LLM and RAG}\label{app:llm}

%The introduction of transformers~\citep{vaswani2017attention} has paved the way for the development of LLMs~\citep{brown2020language, jiang2024mixtral}, which are capable of answering complex queries without the need for additional training or fine-tuning. However, recent research has underscored the unreliability of LLMs, as they are prone to generating hallucinations~\citep{ji2023survey, zhang2023siren}. To mitigate this issue, Retrieval Augmented Generation (RAG)~\citep{lewis2020retrieval} has been proposed. RAG employs a smaller language model to evaluate the relevance of textual information with respect to a given query, subsequently appending the most pertinent information to the query. This enriched context conditions the LLM to produce more accurate and reliable responses. 
LLMs~\citep{brown2020language, jiang2024mixtral} are capable of answering complex queries without additional training. However, they can be unreliable and prone to hallucinations~\citep{ji2023survey, zhang2023siren}. Retrieval Augmented Generation (RAG)~\citep{lewis2020retrieval} addresses this by retrieving relevant textual information and appending it to the query, thereby improving the accuracy and reliability of LLM outputs. In the context of structural learning, LLMs have been utilized to construct causal graphs by employing ad-hoc prompts specifically designed to condition the model for answering causal queries~\citep{antonucci2023zero, long2023can, zhang2024causal}. To enhance the causal graph discovery process, we utilize an LLM integrated with a RAG to either direct or eliminate edges that remain undirected by the causal discovery algorithm. The information retrieval process is outlined as follows:
\begin{enumerate}
    \item \textit{Document Retrieval}: Given a causal query (e.g., ``Is lung cancer influenced by smoking?''), we first retrieve a set of documents from the web. In our experiments, we employed both the DuckDuckGo search engine for web pages and Arxiv for relevant paper abstracts. From each source, we retrieve the top 10 documents. For certain datasets (cMNIST, CelebA, and Sachs), local documents were used due to the unavailability or inaccessibility of information online (e.g., the Sachs paper). Each retrieved document is then segmented into smaller pieces, referred to as chunks, using a sliding window approach that samples a 512-token chunk every 128 tokens.

    \item \textit{Ranking}: The causal query is transformed by the LLM using a \textit{query transformation} approach~\citep{gao2023retrieval}, which improves the semantic alignment of the query with the relevant document chunks. After transformation, all the retrieved chunks and the modified causal query are processed by a sentence transformer~\citep{reimers2019sentence}, specifically the \texttt{multi-qa-mpnet-base-dot-v1} model. At this stage, cosine similarity is computed between the embedded transformed causal query and each embedded chunk.

    \item \textit{Context}: The LLM is then tasked with answering the causal query using the additional context retrieved in the previous steps. This context is derived from the top 5 chunks with the highest cosine similarity to the transformed causal query, providing the LLM with relevant and supportive information for generating a more accurate answer.
\end{enumerate}

All the prompts mentioned in this section can be found in App.~\ref{app:prompts}.

\subsection{Prompts}
\label{app:prompts}

In this subsection we list all the prompts used in both the causal discovery part and label free concept generation.

\textsc{DuckDuckGO search prompt}
\begin{prompt}
    Your task is to create the most effective search query to find information that answers 
    the user's question. 
    Your query will be used to search the web using a web engine (e.g. google, duckduckgo).
    NOTE: be short and concise.
    
    This is the question: {question}. 
    Provide the final query without brackets.
\end{prompt}

\textsc{Arxiv search prompt}
\begin{prompt}
    Your task is to create the most effective search query to find information that answers 
    the user's question. 
    Your query will be used to search scientific articles from the web.
    From the given query, produce a query that will help to find the most relevant articles.
    NOTE: be short and concise.
    
    This is the question: {question}. 
    Provide the final query without brackets.
\end{prompt}

\textsc{Transformation query prompt}
\begin{prompt}
    Rephrase the query to align semantically with similar target texts while maintaining 
    its core meaning.
    Output the expanded query enclosed within 
    <expanded_query> tags (e.g. <expanded_query>[example_query]</expanded_query>).
    NOTE: be very short and concise.
    
    Query: {query}
    Expanded Query: 
\end{prompt}

\textsc{Causal Prompt}
\begin{prompt}
    You are an expert in causal inference and logical analysis. 
    I will provide you with two concepts and you have to infer the causal relationship between them.
    **Concept 1:** {concept_1} - {concept_1_description}
    **Concept 2:** {concept_2} - {concept_2_description}

    Now, use your knowledge and, if available, the context provided, to determine 
    which of the following options is the correct one:
    (A) changing {concept_1} to certain values result in a change in {concept_2};
    (B) changing {concept_2} to certain values result in a change in {concept_1};
    (C) there is no causal relationship or reciprocal influence between {concept_1} and {concept_2}.

    The following information are extracted from recent and reliable sources:
    {context}

    The answer has to be enclosed within <answer> tags (e.g. <answer>A</answer>).
    Analyze the situation step-by-step to ensure the final conclusion is accurate.
\end{prompt}

\textsc{Concepts generation prompt}
\begin{prompt}
    You are an expert of {context}. 
    You need to list the most important features to recognize {class_label} from {input}. 
    List also the variables that are most likely to be associated with {class_label} as well as
    the variables that are most likely to be associated with the absence of {class_label}. 
    You need also to give a list of superclasses for the word {class_label}.
    Combine all the lists in a single one and separate the single terms with a comma. 
    If a term is composed by more than one word, use an underscore to separate the words.
\end{prompt}

\section{C$^2$BMs detailed architecture}\label{app:detailedarchitecture}
In this appendix, we provide a detailed description of the proposed C$^2$BM model training and functioning, using the \textit{Asia} dataset as an illustrative example and \emph{Dyspnea} as  the task (Fig.~\ref{fig:asia-pipeline}). We assume the following information is available:
\begin{itemize}
\item A set of human-understandable variables  relevant to determining the task's value. Specifically, the binary concepts: $\{$\emph{Smoker, Bronchitis, Lung cancer, Either, Tubercolisis, Been in Asia, Xray anomalies, and Dyspnea}$\}$.
\item A training dataset $D = \{\textbf{x}_i, \textbf{v}_i\}_{i=1}^n$, where each sample is annotated with the values of all endogenous variables, i.e., the target variable \emph{Dyspnea} and all preceding binary concepts.
\item A DAG $\mathbf{G}$ outlining the causal relationships between the concepts and the task.
\end{itemize}
\begin{figure*}[h]
    \centering    \includegraphics[width=\linewidth]{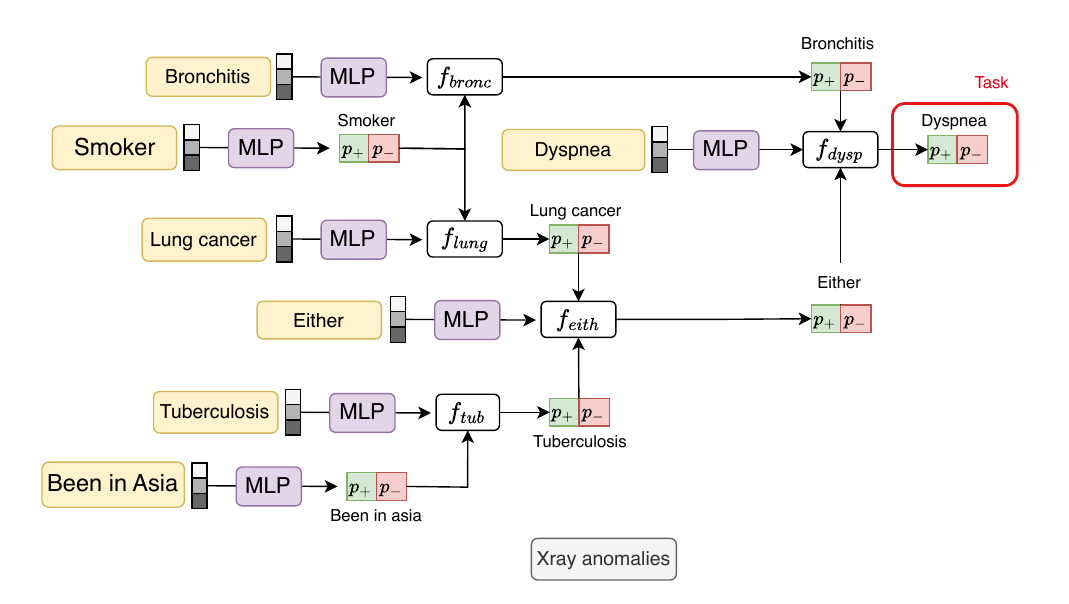}
    \caption{Detailed C$^2$BM architecture applied to the Asia dataset.}\label{fig:asia-pipeline}
\end{figure*}
These elements can be either provided by human experts or generated using the pipeline we propose in the paper, as described in Fig. \ref{fig:pipeline}. The proposed approach follows four main steps (Fig.~\ref{fig:asia-pipeline}), as detailed below:

%\vspace{-0.3cm}
\begin{enumerate}
\item \textbf{Sub-causal Graph Selection.} We extract from the DAG $\mathbf{G}$ only the variables that are ancestors of the task, that is, variables for which there exists a causal path in $\mathbf{G}$ connecting the variable node to the task. In the case of the Asia dataset, the variable \emph{X-ray anomalies} is discarded because it is not an ancestor of the task. %For this task we used common methods for recovering a CPDAG, see App. \ref{app:ablation_cd}. %including: (i) \emph{constraint-based algorithms}, e.g., PC algorithm~\citep{spirtes2000causation} exploiting conditional independencies in data; (ii) \emph{score-based algorithms}, e.g., GES~\citep{chickering2002}, which optimize a score over a set of possible graphs.

\item \textbf{Exogenous embeddings.} Each variable (including the task) is assumed to have an associated latent factor, which is represented as an embedding (the grey encoder symbols in Fig.~\ref{fig:asia-pipeline}) learned from the input using a dedicated neural encoder. In our implementation, these are implemented as MLPs, preceded by a dataset-specific feature extractor, e.g., a CNN for image data (as detailed in the App.~\ref{app:datasets}).

\item \textbf{Structural Equation Modeling.} 
We model the structural relationships between parent nodes and their child nodes using linear equations. The weights of these equations are predicted by separate hypernetworks, implemented as MLPs, which take as input the embedding of the child node produced in the previous step. Applying these functions we can derive the normalized logits of a child node from the ones of its parents. For the nodes that lie in the roots of the causal graph (\textit{Been in Asia}, \textit{Smoker}), their logits are obtained directly from the corresponding exogenous embeddings.

For instance, we can calculate normalized logits for \emph{Dyspnea} as follows:
\begin{equation}
\mathbf{p}_{Dyspnea} = \sigma(\boldsymbol{\theta}_1 \mathbf{p}_{Bronchitis} + \boldsymbol{\theta}_2 \mathbf{p}_{Either})
\end{equation}
where $\mathbf{p}_{Bronchitis}$ and $\mathbf{p}_{Either}$ are the normalized logits for the parents, $\boldsymbol{\theta}_{f_{Dyspnea}} = [\boldsymbol{\theta}_1,\boldsymbol{\theta}_2 ]$ their corresponding weights and $\sigma$ denotes a transformation function (such as a softmax) applied to the weighted sum of the parent node logits, ensuring the final output is in a suitable range.

While the structural equations are linear, the fact that the weights can be adaptively inferred from exogenous variables allows us to capture complex dependencies between variables. This idea is analogous to locally approximating complex (smooth) functions. For example, consider an exponential relationship between two endogenous variables, $V_2 = e^{V_1}$ (Fig.~\ref{fig:nonlinearity}). This function can be locally approximated by a linear form $V_2 = \theta_1 V_1$, where the weight $\theta_1$ is adjusted based on the value of $V_1$.
\begin{figure*}[h]
    \centering
\includegraphics[width=0.8\linewidth]{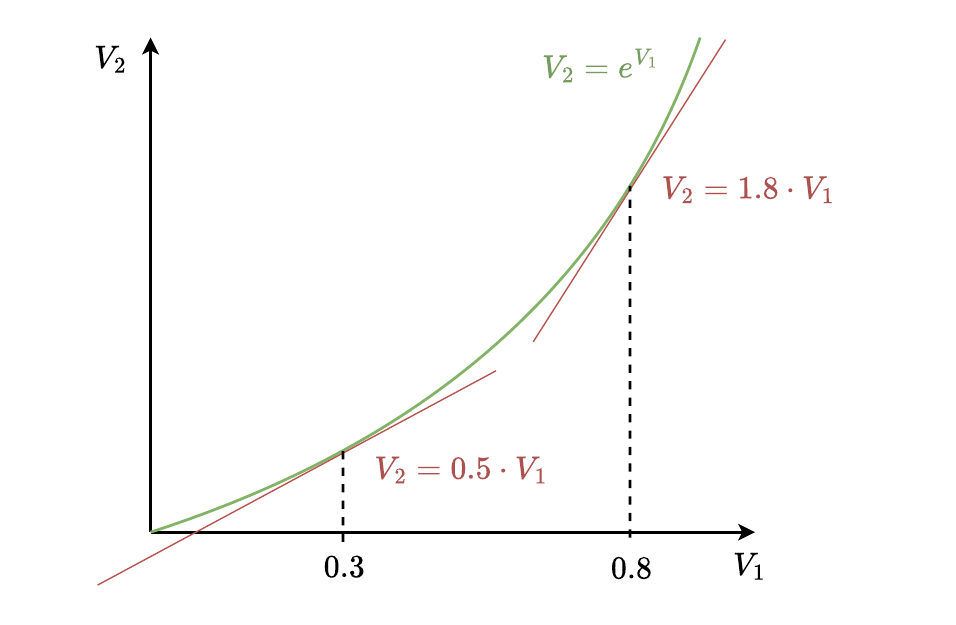}
    \caption{Non-linear functions can be modeled by adaptive re-parametrizable linear models.}\label{fig:nonlinearity}
\end{figure*}
\end{enumerate}

It is important to notice that such a model supports queries about any specific endogenous variable. Specifically, 
after training, C$^2$BM can be used to predict a target variable (task), while using the other variables as concepts to explain the reasoning process. This provides greater flexibility compared to other concept-based architectures, which instead assume a fixed task. Notably, only the task's ancestors are relevant in this case and, in our implementation, all other concepts are discarded. 

\subsection{C$^2$BM as an universal approximator}\label{app:universal-approx}
We establish the following result regarding the expressivity of C$^2$BM.
\begin{theorem}
    C$^2$BM is a universal approximator regardless of the underlying causal graph.
\end{theorem}
\begin{proof}
    We show that C$^2$BM can predict any endogenous variable as any DNN. Assume the exogenous encoder $\mathbf{g}(\cdot)$ is a DNN (or any universal approximator) and that endogenous variables are represented as logits (the extension to fuzzy or Boolean values is trivial). We consider four cases:
    \begin{enumerate}
    \item \textbf{No endogenous parent:} If $V_i$ is a root, $V_i=\text{MLP}_i(\mathbf{g}(X)_i)$. Since both $\mathbf{g}(\cdot)$ and $\text{MLP}_i$ are universal approximators, their composition is also a universal approximator.
    \item \textbf{Single endogenous root parent:} If $V_i$ has exactly one root parent $V_j$, then $V_i = [\boldsymbol\theta_{f_i}]_j V_j=[\mathbf{r}_i(\mathbf{g}(X)_i)]_j \cdot V_j$ and since both $\mathbf{r}_i(\cdot)$ and $\mathbf{g}(\cdot)$ are universal approximators, their composition is also a universal approximator. Multiplying this by $V_j$, which itself is produced by a universal approximator, preserves the ability to approximate any function of the input.

    \item \textbf{Multiple endogenous root parents:} If $V_i$ has more than one parent, $\boldsymbol\theta_{f_i}$ can assign zero weights to all but one parent, reducing the case to a single-parent scenario.
    \item \textbf{Non-root endogenous parents.}
   The reasoning above can be applied recursively following the topological ordering of the causal graph. Each variable is computed as a function of its parents, and universal approximation is preserved layer by layer.

\end{enumerate}
Hence, by recursively composing universal approximators along the causal graph, C$^2$BM can approximate any mapping from the input to endogenous variables. This holds for any graph structure, establishing C$^2$BM as a universal approximator. 

\end{proof}

\subsection{C$^2$BM interpretability}\label{app:interpretability} 

In this section, we present an explanation generated by C$^2$BM on the \textit{Asia} dataset. As shown in Tab.~\ref{table:hamming}, the causal graph retrieved by the causal discovery mechanism is almost equal to the real one, despite a missing edge between \textit{Been in Asia} and \textit{Tuberculosis}. Starting from the source endogenous variables, it is possible to see the weight associated to each descending endogenous variable and the corresponding activation probability. For instance, \textit{Lung cancer} is `True' because the corresponding probability is peaked toward it ($P(1)=0.99$). In particular, the decision-making process for the classification of \textit{Dyspnea} as `False' is completely unveiled. Although the parameter on the edge from \textit{Bronchitis} to \textit{Dyspnea} promotes a positive prediction, the stronger, negatively weighted connection from \textit{Either} to \textit{Dyspnea} dominates. Resulting in \textit{Dyspnea} being predicted as `False'. It is worth noting that the endogenous variable \textit{X-ray anomalies} is not considered by C$^2$BM's inference since it is not an ancestor of the defined task.

\begin{figure*}[h]
    \centering
    \includegraphics[width=\linewidth*2/3]{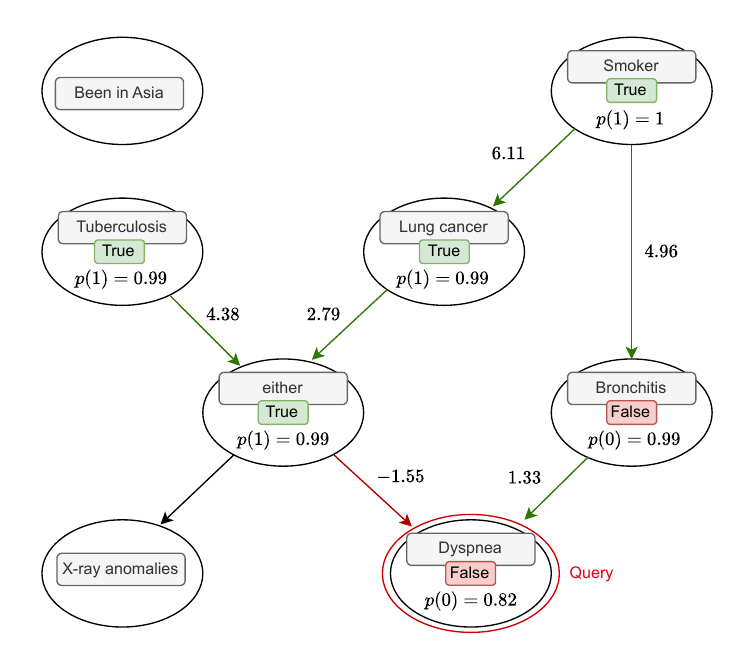}
\caption{Visualization example of 
 information propagation (Asia dataset). The parent weight (predicted by the hypernetwork) is visualized next to each edge.}\label{fig:interpretability}
\end{figure*}

\section{Dataset details}\label{app:datasets}
\subsection{cMNIST}
The MNIST dataset \citep{lecun2010} is a large collection of freely available grayscale images of handwritten digits. It consists of $60000$ training images and $10000$ test images, both drown from the same distribution. Each image is labeled with the digit it represents. For our experiments, we download the original train and test dataset using the \textit{torchvision} library~\citep{marcel2010} and reserve $10\%$ of the training set for validation. Additionally, we colorize each image based on its digit. Specifically, in the in-distribution setting (experiment in Sec.\ref{sec:exp-accuracy-and-rel}), color is randomly assigned to each image in the training, validation, and test sets with equal probability over red, green, and blue. In the out-of-distribution setting (experiment in Sec.\ref{sec:exp-ood}), images in the training and validation sets with odd digits are colored green, while the remaining images are colored red or blue with equal probability. In the test set, images with even digits are colored green, and the remaining ones are colored red or blue with equal probability. For both the versions of this dataset, the following concepts are considered: \emph{Number}, \emph{Color}, and \emph{Parity} (task). Finally, the images are preprocessed using a pre-trained ResNet-18 model with default weights from the torchvision library.

\subsection{Bayesian networks}
For our experiments, we use synthetic datasets sampled from discrete Bayesian networks available in the \texttt{bnlearn} repository~(\url{https://www.bnlearn.com/bnrepository/}). A \emph{Bayesian network} is a probabilistic graphical model consisting of a DAG, where nodes correspond to random variables, and each node is associated with a conditional probability distribution (CPD). This CPD defines the probability of the node's value, given the values of its parent nodes in the network \citep{sharma2020}.
From the \texttt{bnlearn} repository, we select Bayesian networks with different dimensions and domains: \textbf{Asia}~\citep{asia}, a small network focused on lung disease with $8$ nodes and $8$ edges; \textbf{Sachs}~\citep{sachs2005causal}, a widely-used network modeling the relationships between protein and phospholipid expression levels in human cells with $11$ nodes and $17$ edges; \textbf{Insurance}~\citep{binder1997},  a network for evaluating car insurance risks with $27$ nodes and $52$ edges; \textbf{Alarm}~\citep{alarm}, a network designed to provide an alarm message system for patient monitoring with $37$ nodes and $46$ edges; \textbf{Hailfinder}~\citep{abramson1996}, a network designed to forecast severe summer hail in northeastern Colorado with $56$ nodes and $66$ edges. For each network, we generate $10000$ samples and create training, validation, and test datasets using a $70\%-10\%-20\%$ split. 

While node values can be used as concept annotations ($\mathbf{v}$), input features ($\mathbf{x}$) are absent. To generate them and make the datasets applicable to concept-based architectures, we flatten the concept values and process them with a simple autoencoder (MSE loss) comprising 2 encoder layers and 2 decoder layers (the latent dimension is adjusted based on the number of nodes: \textit{Asia}-32, \textit{Sachs}-32, \textit{Insurance}-32, \textit{Alarm}-64, \textit{Hailfinder}-128). Embeddings are further transformed so that each sample is a mixture composed of $50\%$ original data and $50\%$ noise, with the noise drawn from a standard normal distribution. Finally, the output is standardized. The goal of these transformations is to make the inputs non-trivial representations of the concepts (the network nodes), forcing architectures to learn how to identify and retrieve the underlying concepts from the preprocessed data.
%In out-of-distribution settings, we modified Sachs Bayesian network by introducing spurious correlations between the variables \emph{Plcg} and \emph{P38}, \emph{Plcg} and \emph{PKC} and between \emph{Plcg} and \emph{PKA}. Specifically, we introduce the following CPDs in the form of conditional probability tables: \ari{inserire tables se necessario}

Modified versions of the \textit{Asia} and \textit{Alarm} datasets are considered, denoted as Asia$^*$ and Alarm$^*$, in which only a subset of the original concepts is retained (experiment in Sec.\ref{sec:exp-accuracy-and-rel}). Specifically, for \textit{Asia}$^*$, we keep only the concepts "Smoke" and "Dyspnea". For \textit{Alarm}$^*$, we retain only the concepts: "BP", "CO", "CATECHOL", "HR", "LVFAILURE", "STROKEVOLUME", "HYPOVOLEMIA".

\subsection{CelebA}
CelebA~\citep{liu2015} is a large-scale face attributes dataset with more than $200.000$ celebrity images divided into training, validation and test set with $40$ binary attribute annotations.  
For our experiments, we first downloaded all the splits from the project website \url{https://mmlab.ie.cuhk.edu.hk/projects/CelebA.html}. 
\begin{wrapfigure}{r}{0.50\textwidth}
    \centering
    \includegraphics[width=\linewidth]{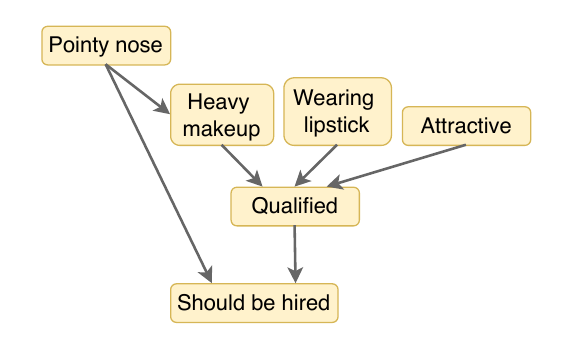}
    \caption{Subset of introduced causal relationships among original and newly created concepts in CelebA, used for our fairness analysis.}\label{fig:celeba_custom}
    \label{fig:celeba_custom}
\end{wrapfigure}
We then select a subset of attributes that we consider relevant for our analysis and apply this selection to all the splits. These attributes are: \emph{Attractive},
\emph{Big Lips}, \emph{Heavy Makeup}, \emph{High Cheekbones}, \emph{Male},
\emph{Mouth Slightly Open},
\emph{Oval Face}, 
\emph{Smiling}, \emph{Wavy Hair}, \emph{Wearing Lipstick} (experiment in Sec.\ref{sec:exp-accuracy-and-rel}). For our fairness analysis (experiment in Sec.\ref{sec:exp-fairness}), we selected the attributes: \emph{Attractive}, \emph{Heavy Makeup}, \emph{High Cheekbones}, \emph{Male},
\emph{Mouth Slightly Open},
\emph{Oval Face}, 
\emph{Pointy Nose}, \emph{Smiling}, \emph{Wavy Hair}, \emph{Wearing Lipstick} and \emph{Young}. 
We then introduced two additional new attributes \emph{Qualified} and \emph{Should be Hired}.
In this analysis, we consider a hypothetical scenario in which a person with a pointy nose is required for a specific role, e.g., in a movie. However, the hiring manager has a strong bias toward models who are considered attractive or heavily made up. Our aim is to intervene and mitigate such biases. The \emph{Qualified} attribute is defined as a binary variable indicating whether a person meets the qualifications for the job based on the hiring manager's biased criteria. It is therefore constructed using the logical expression: (\emph{Heavy Makeup} and \emph{Wearing Lipstick}) or \emph{Attractive}. The \emph{Should be Hired} attribute, on the other hand, indicates whether a person should be hired for the job, considering both the hiring manager’s preferences and the role's requirements (having a pointy nose). Therefore, it is defined as the logical "and" between \emph{Qualified} and \emph{Pointy Nose}.
Additionally, we applied standard preprocessing to both versions of the dataset, including downsampling and normalization of the images, followed by feature extraction using a pre-trained ResNet-18 model.

\subsection{Siim-pneumothorax}
This dataset is derived from the publicly available chest radiograph dataset provided by the National Institutes of Health (NIH) and it contains chest x-ray images with binary annotations indicating the presence or the absence of Pneumothorax. For our experiments, we use in particular the training annotations available on Kaggle (\url{https://www.kaggle.com/competitions/siim-acr-pneumothorax-segmentation}) and the corresponding training images from \url{https://www.kaggle.com/datasets/abhishek/siim-png-images}.  The dataset is then split into training, validation, and test sets using a traditional $70\%-10\%-20\%$ partition. As it does not include concept annotations, we followed a procedure inspired by the methodology in \citet{oikarinen2023} to generate concepts and their corresponding annotations. More details can be found in App.~\ref{app:concept_discovery}. %In this work, we begin by generating a set of concepts using an LLM, with the specific prompt provided in App.~\ref{app:prompts}. These concepts are subsequently filtered following the general methodology described in \citet{oikarinen2023}, with modifications tailored to our specific application. Specifically, we utilize the \emph{CXR-CLIP} model—pretrained on medical imaging datasets~\citep{johnson2019, irvin2019, wang2017, you2023}—to encode both concepts and images. The LLM first produces candidate concepts, which are then refined based on criteria outlined in~\citet{oikarinen2023}. Next, we compute cosine similarities between the image and concept embeddings using \emph{CXR-CLIP}. Finally, we apply a k-means clustering algorithm to discretize these similarity scores into binary values.

%\ari{dire che in questo modo i nostri concetti non saranno mai sempre spenti e quindi il passaggio 5 dei label-free non e'  necessario o lasciare cosi? parlare della configurazione usata per il preprocessing?}

\subsection{CUB$_C$}
This dataset is derived from the publicly available Caltech-UCSD Birds-200-2011 (CUB)~\citep{he2019fine} dataset, which is widely used in the CBM community \citep{koh2020, zarlenga2022concept}. It contains 11,788 images across 200 bird categories, with 5,994 images for training and 5,794 images for testing. Each image is annotated in detail, including 312 binary attributes. For our experiments, we downloaded the dataset from \url{https://data.caltech.edu/records/65de6-vp158} and further split the training set such that $10\%$ is used for validation. We then selected the 112 most frequently activated binary attributes to serve as our concepts of interest, as considered in \cite{zarlenga2022concept}.

To explore deeper causal relationships between concepts, we introduce four new ones: \emph{camouflage}, \emph{flight adaptation}, and \emph{hunting ability}, which are derived from existing attributes using logical rules, as well as the multi-valued concept \emph{survival}, whose activation is in turn derived from these three binary concepts. The rules used for these derivations are detailed in the table below (Table \ref{table:cub_concepts}). We use \emph{survival} as our downstream task. The images are instead further downsampled, normalized, and processed using a ResNet-50 architecture, following a procedure similar to that used for the CelebA dataset.
%if at least one concept in the corresponding row of the table below is activated.

\input{tables/cub_custom}    

\section{Experimental details}\label{app:exp-details}
\textbf{Python code and instructions for reproducing results across all datasets and methods are available within the code provided alongside the submission as supplementary material.} We detail below the configurations and hyperparameters used for models' instantiation and training. For a fair comparison across concept-based models, we standardize the concept encoder to a single-hidden-layer MLP. All models are trained using the Adam optimizer~\citep{kingma2014adam} for a maximum of 500 epochs, with early stopping based on a 30-epoch patience. We employ \textit{LeakyReLU} as the activation function throughout.

The batch size is set to 512 for most datasets, with the exception of \textit{Siim-pneumothorax} and SCBM, where it is reduced to 128 due to memory constraints. Following~\cite{koh2020}, we regularize the task loss to encourage concept learning, using a weighted sum of task and concept losses:
$$
L = (1 - \alpha) \cdot L_{\text{task}} + \alpha \cdot L_{\text{concepts}}, \quad \text{with } \alpha = 0.8
$$
where $L_{\text{task}}$ is a cross-entropy over the task whereas $L_{\text{concepts}}$ is the summation of cross-entropy losses over the concepts.

Additionally, we apply random training-time interventions as proposed by~\cite{zarlenga2022concept}, with an intervention probability of 0.25. Regarding SCBM, we used the authors' implementation at \url{https://github.com/mvandenhi/SCBM}. More precisely, we implemented the \textit{global} variation using the configuration proposed by the authors.

Key hyperparameters, including learning rate, MLP hidden size, and dropout rate, are selected via grid search. A complete list of hyperparameters for C$^2$BM and all baseline models can be found in the provided YAML configuration files within the code.

All experiments are conducted on NVIDIA GeForce RTX 3080 and NVIDIA RTX A5000 GPUs.

\subsection{Metric: custom Structural Hamming Distance}\label{app:hamming}
To assess the quality of the causal graphs generated by our pipeline and baseline discovery models, we employ two metrics: a variant of the structural Hamming distance (SHD)\footnote{Although we refer to this as a "distance," it is technically an asymmetric scoring function.} that operates on CPDAGs, and the number of incorrect edges identified. The number of incorrect edges serves as a standard measure of the quality of the graph, while the SHD allows us to customize weights for different types of errors. Tab.~\ref{tab:shd_penalties} provides a schematic of our SHD scores: 
\begin{table}[h!]
    \centering
    \begin{tabular}{ccc}
        \toprule
        \textbf{Ground Truth} & \textbf{Predicted} & \textbf{SHD Penalty} \\
        \midrule
        /         & $i \rightarrow j$ & 1 \\
        /         & $i - j$  & $1/2$ \\
        $i \rightarrow j$ & $i \leftarrow j$ & $1/3$ \\
        $i \rightarrow j$ & /         & $1/4$ \\
        $i - j$  & /         & $1/4$ \\
        $i - j$  & $i \rightarrow j$ & $1/4$ \\
        $i \rightarrow j$ & $i - j$  & $1/5$ \\
        \bottomrule
    \end{tabular}
    \caption{SHD penalties for various discrepancies between ground truth and predicted edges. '$i \rightarrow j$': oriented edge, '$i - j$': unoriented edge, '$/$': no edge}
    \label{tab:shd_penalties}
\end{table}

The rationale behind the scores is that the insertion of new edges is ‘riskier’ than the removal of existing ones. The introduction of non-existing edges may induce spurious correlations that strongly affect the reliability of the model and related metrics, such as counterfactual fairness or accuracy in ood tasks. On the contrary, removing an existing edge is a more conservative operation: while it still affects the model accuracy, is not strongly impacting reliability. For the same reason, introducing an incorrect edge orientation is more penalized than removing an orientation.

\section{Additional experiments}
\label{app:extraexp}
\subsection{Label accuracy}
\label{app:concept_accuracy}
We report here the average label accuracy (task + concepts) for each of the datasets analyzed in Tab.~\ref{table:accuracy}. In this setting, since we are evaluating the accuracy across both the task and the concepts, we use a different opaque neural baseline, $\text{OpaqNN}_{M}$, which jointly predicts all concepts and the task for each dataset. As shown in the Tab.~\ref{table:accuracy_concepts}, C$^2$BM achieves comparable results to non-causal models in terms of concept prediction. Notably, the performance differences observed in Section~\ref{sec:exp-accuracy-and-rel} for task accuracy are attenuated here due to averaging over both tasks and concepts.
\input{tables/accuracy_labels}

\subsection{Task accuracy using true graph}\label{app:true-graph}
%Providing a well-structured causal graph is crucial for C$^2$BM, as it facilitates the selection of concepts related to the task and enhances the reasoning process. 
To further assess the quality of the generated graph, we compare downstream task performance using the inferred graph with performance using the true graph, available in Bayesian Network synthetic datasets.
\input{tables/true_graph}

Results in Tab.~\ref{table:true-graph} indicate that the two settings yield comparable performance, demonstrating both the quality of the learned graph and the robustness of the proposed pipeline. This robustness can be attributed to the employed exogenous latent embeddings, which mitigate the impact of concept incompleteness. Consequently, even when the inferred graph is not perfectly aligned with the true causal structure, C$^2$BM maintains strong task performance. These findings highlight the model’s resilience and its ability to generalize effectively in real-world scenarios where true causal structures are often unavailable.

\subsection{Ablation study on Causal Discovery}\label{app:ablation_cd}
In this section, we evaluate the sensitivity of the causal discovery component in our pipeline to the choice of method for constructing a causal graph. To do so, we compare our method of choice for causal discovery, i.e., \textit{Greedy Equivalence Search} (GES)~\citep{chickering2002}, with different widely-used causal discovery algorithms that recover a CPDAG. Each of these methods is evaluated with and without refinement via the retrieval-augmented generation (RAG)-enhanced language model (LLM) employed in our study. Furthermore, we compare all these methods against the use of the LLM alone, as well as our retrieval-augmented LLM (LLM + RAG) applied directly to discover the whole graph. Specifically, the LLM is prompted with the causal prompt shown in App.~\ref{app:prompts}, with additional retrieved context appended to the query in the LLM+RAG setting.

Specifically, we evaluate the following methods from the causal discovery literature:
\begin{itemize}
    \item \textit{GES} (score-based): The algorithm selected in our study. A detailed description is provided in App.~\ref{app:causal_discovery};
    \item \textit{PC Algorithm}~\citep{spirtes2000causation} (independence-based): A classical independence-based method that first estimates the undirected structure of the causal relations through a sequence of conditional independence tests, then orients edges using a set of predefined orientation rules~\citep{colombo2014}.  In our implementation, we use the PC algorithm provided by the \texttt{causal-learn} library~\citep{zheng2024}, with a chi-squared independence test and a significance level of 0.05;
    \item \textit{Fast GES}~\citep{ramsey2017} (FGES) (score-based): A computationally efficient variant of GES that improves performance by storing intermediate evaluations and parallelizing expensive operations, enabling its application to large datasets~\citep{andrews2019, ramsey2017}. In our implementation, we use the FGES algorithm provided by the \texttt{py-tetrad} library~\citep{ramsey2023py}, employing the \emph{Bayesian Dirichlet Equivalent Uniform} (BDeu) score for structure evaluation~\citep{heckerman1995learning}.
\end{itemize}
The results are evaluated on all the \texttt{bnlearn} datasets, for which the true causal graphs are known. Results are presented in Tab.~\ref{table:cd_ablation}.   

\input{tables/ablation_cd}

As shown in the results, GES refined with the RAG-augmented LLM is the overall best-performing method in the majority of cases, both in terms of structural Hamming distance and number of mistaken edges. In instances where it is not the best, it still consistently ranks among the top methods. Notably, both the standard causal discovery method (GES) and the use of the RAG-augmented LLM contribute positively to performance. Due to the substantial computational time required by the LLM+RAG-based causal discovery approach, we excluded experiments on datasets for which execution exceeded 24 hours.

\subsection{Ablation study on LLM type}\label{app:ablation-llm}
The LLM or the design of the LLM prompt could condition the causal graph refinement. To assess this, we present an ablation study in which we fix the incomplete causal graph generated by the causal discovery algorithm (GES), and later evaluate different LLMs, with different prompts, for the graph refinement step. Specifically, we evaluate the impact of four different prompting strategies with three LLMs (GPT-4o, 200M parameters; and GPT-4o-mini, 8M parameters). The considered strategies are as follows:
\begin{itemize}
    \item \textit{Minimal} prompting: simply asks the LLM to identify causal relations without additional guidance;
    \item \textit{Instruction} prompting: provides a more detailed explanation of what constitutes a causal relation;
    \item \textit{Few-shot} prompting: proposed in~\cite{ye2024investigating}, combines instruction prompting with a few examples;
    \item \textit{Chain-of-Thought} (CoT): proposed in~\cite{wei2022chain}, encourages the model to generate intermediate logical steps before generating the answer. This is the approach we used in the original manuscript.
\end{itemize}
Tab.~\ref{table:llm_ablation} compares the number of mistaken edges and the custom Hamming distance between the true and the learned DAG. Causal reliability of standard, flat, CBMs is also reported for reference. The results show that, using GPT-5, causal reliability improves on average w.r.t. GPT-4o, and GPT-4o-mini, especially on the Sachs dataset. Despite moderate variation in the Sachs dataset, the performance is generally very robust to the prompt strategy (made an exception for naive minimal strategy), and causal reliability is largely superior to standard CBMs in all cases.

\input{tables/ablation_llm}

\subsection{Ablation study on RAG}\label{app:ablation-rag}
Tab.~\ref{table:rag_ablation} illustrates the influence of the context supplied by RAG in correcting the causal graph generated by the causal discovery algorithm. In this analysis, we used \texttt{GPT-4o} for both the experiments with and without RAG, aiming to evaluate the effect of context on answering causal queries.

\input{tables/ablation_rag}

Although the context provided by RAG appears to have no significant impact on the final causal graph for the Alarm dataset, it proves essential for correctly handling the undirected edges in the causal graph for the Sachs dataset. We hypothesize that this is due to the absence of protein related documents used in training the LLM, which leaves it with insufficient prior knowledge to address specific questions on the topic. RAG helps mitigate this limitation by supplying the LLM with the relevant information, thereby compensating for the lack of prior knowledge. In conclusion, while an LLM with the ability to fully comprehend complex queries is crucial for causal discovery, the additional context provided by RAG is vital for overcoming the LLM's prior knowledge gaps.

%\subsection{Sensitivity study on test-intervention noise}

%\subsection{Sensitivity study on training intervention probability}

\subsection{Sample complexity of graph discovery}\label{app:sensitivity-data-size}
We study the effect of dataset size and graph size on the quality of C$^2$BM's graph construction pipeline. We explored this with a sensitivity study, running the causal graph pipeline (causal discovery with GES + LLM refinement with CoT prompt) across all datasets with an available ground-truth graph and comparing the number of mistaken edges between the true and the learned DAG. For each dataset, we varied the number of data points $N$ from $100$ to $10000$.
\input{tables/sensitivity_data_size}
Results are presented in Tab.~\ref{tab:sensitivity-data-size}. 

A few considerations emerge:
\begin{itemize}
    \item C$^2$BM's causal graph is consistently more causally reliable than the flat structure implicitly assumed by standard CBMs, regardless of the dataset size.
    \item As expected, increasing the number of data points leads to better alignment between the estimated and true causal graphs. Causal reliability tends to remain stable at larger data sizes.
    \item We observed that the data size threshold for reliable performance does not strictly depend on graph size. We speculate this is due to the varying impact of LLM-based refinement across datasets. In some cases, an effective background knowledge can compensate for limited data.
\end{itemize}

\subsection{Sensitivity to graph corruption}\label{app:sensitivity-graph_corruptions}
In this section, we empirically assess the robustness of C$^2$BM to graph misspecification and corruption. Although C$^2$BM is theoretically a universal approximator for the final prediction task, independent of the specific causal graph (see Appendix~\ref{app:universal-approx}), we complement this result with an empirical validation.
\begin{itemize}
    \item \textbf{Adversarial Graph Corruptions.} We first evaluate robustness by altering a percentage $p$ of graph edges, chosen randomly, with one of the following operations: \textit{edge flipping}, \textit{addition}, or \textit{removal}. The resulting performance across datasets and corruption levels is reported in Tab.~\ref{tab:graph_corruption_edges}.
    
    \item \textbf{Progressive Flattening into Standard CBMs.} As a second evaluation, we progressively transform the graph into the flat structure assumed by standard CBMs, by connecting a percentage $p$ of nodes directly to the prediction task while removing their outgoing edges. Results are shown in Tab.~\ref{tab:graph_flattening}.
\end{itemize}

\input{tables/sensitivity_corruption_operations}
\input{tables/sensitivity_corruption_cbms}

Across both corruption strategies, we find that task accuracy \textbf{remains stable}, even when the causal graph is heavily perturbed. These empirical results support the theoretical claim that C$^2$BM is robust to graph misspecification.

\subsection{Effect of single-concept interventions}\label{app:single-interventions}
\begin{figure}[h]
    \centering
    \includegraphics[width=0.49\linewidth]{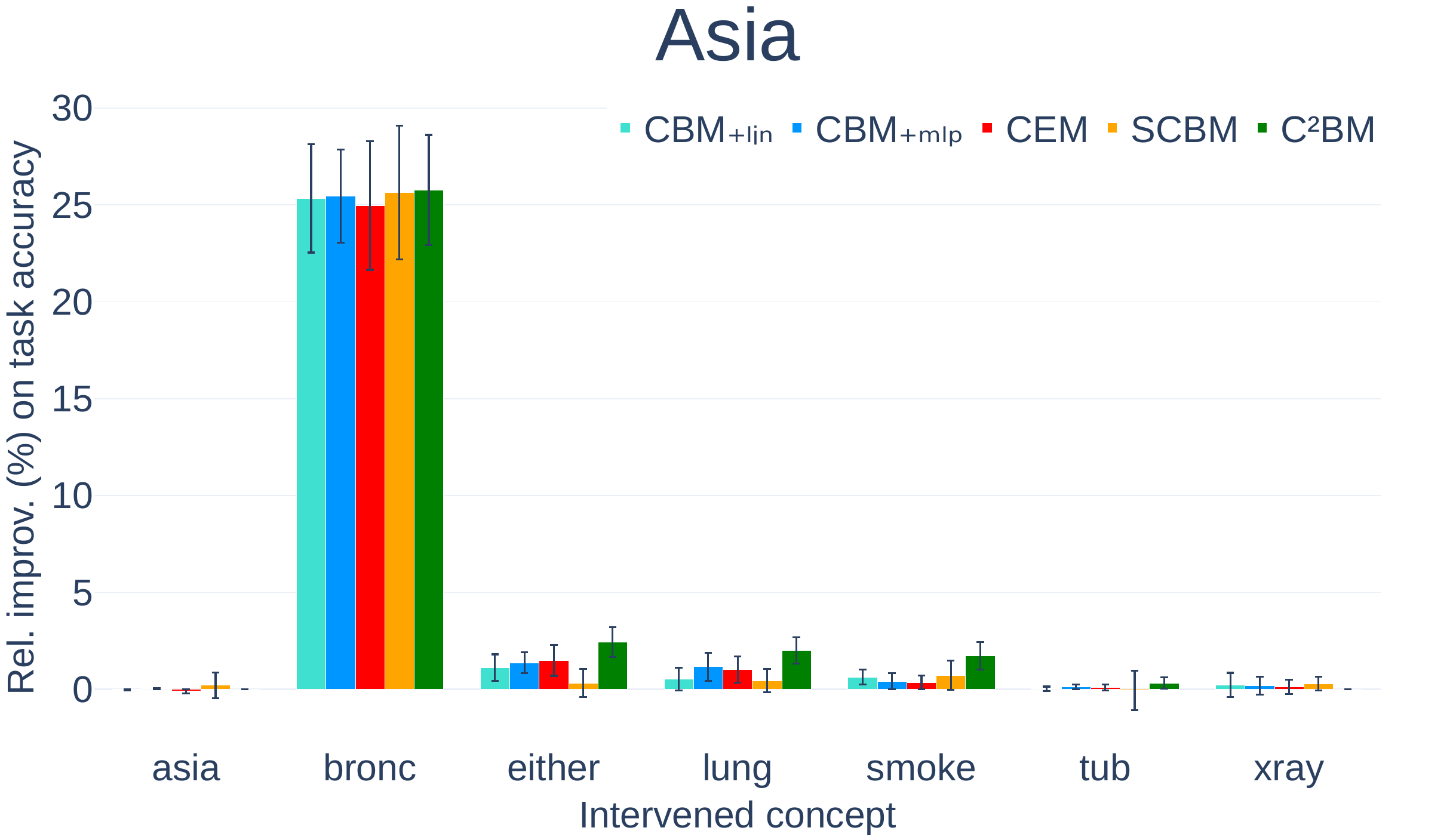}
    \includegraphics[width=0.49\linewidth]{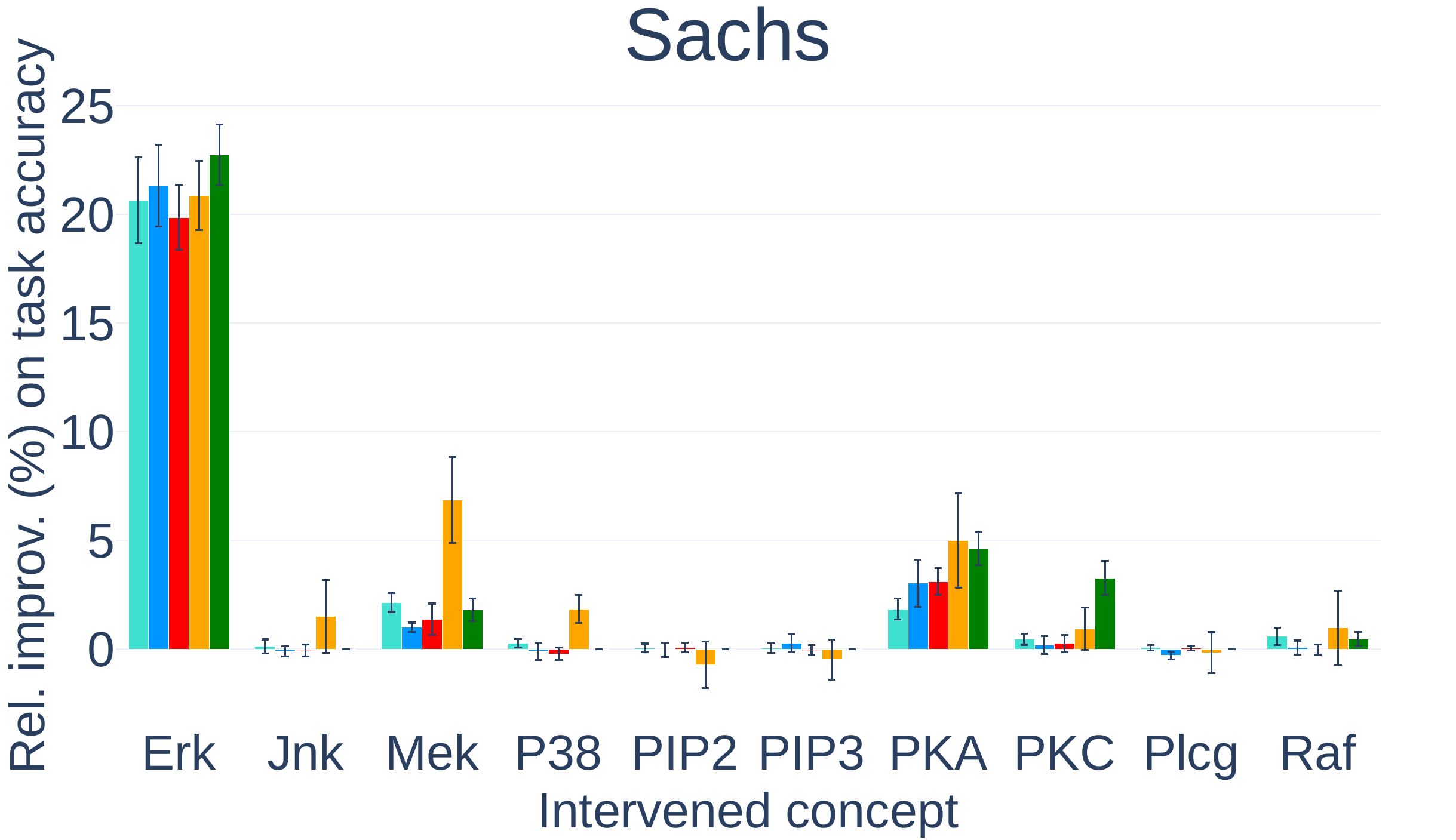}
    \includegraphics[width=0.49\linewidth]{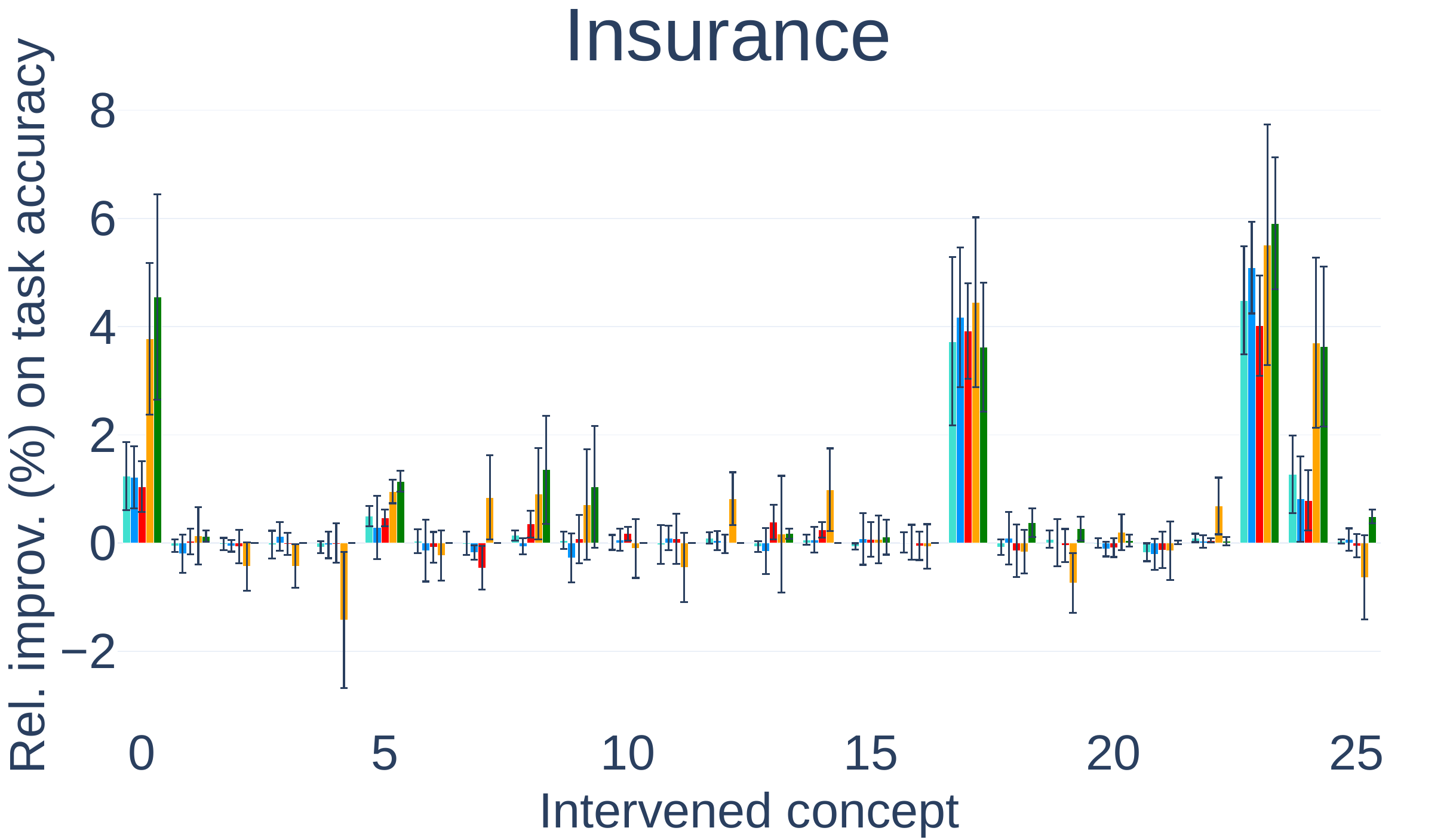}
    \includegraphics[width=0.49\linewidth]{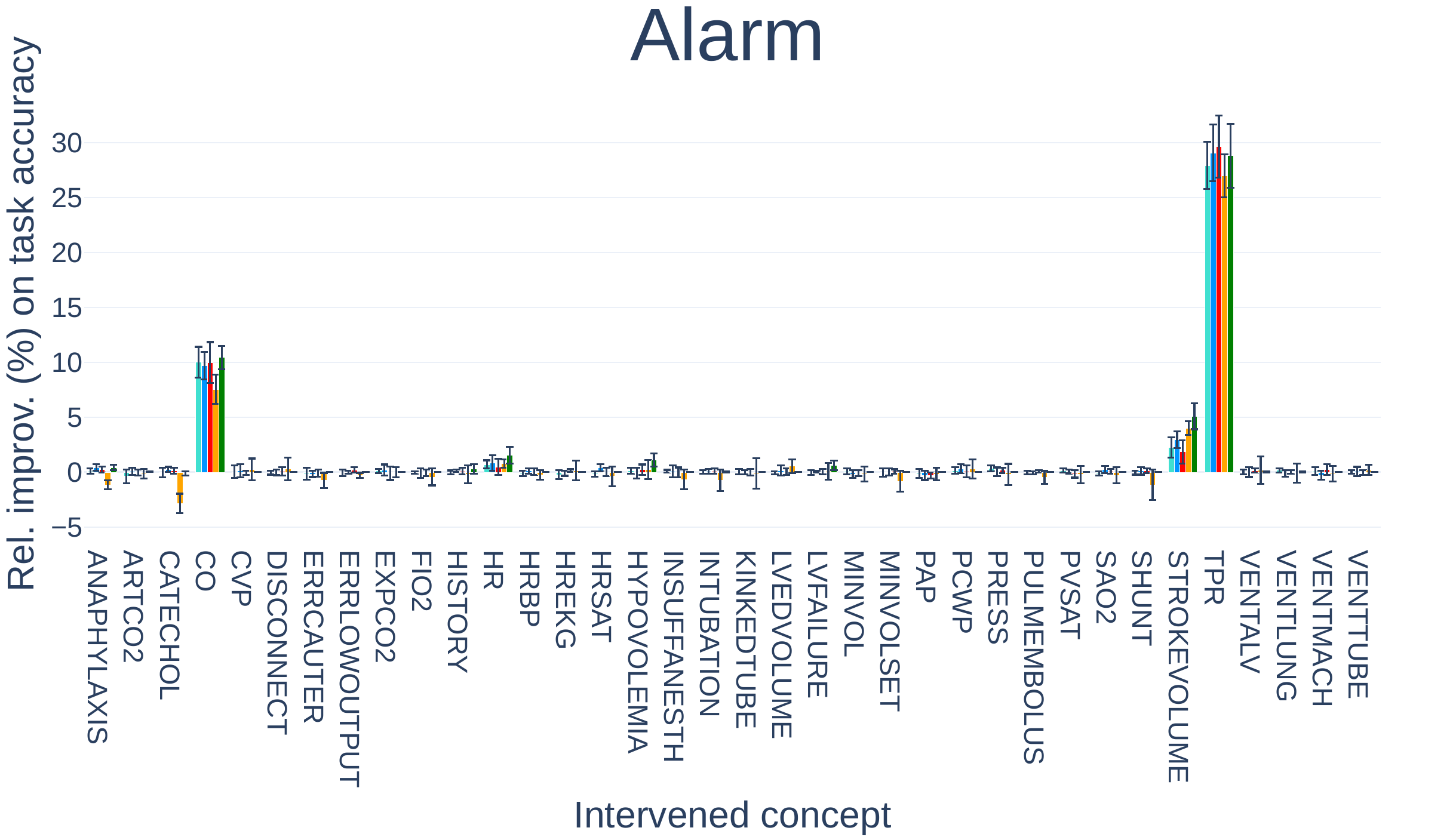}
    \includegraphics[width=0.49\linewidth]{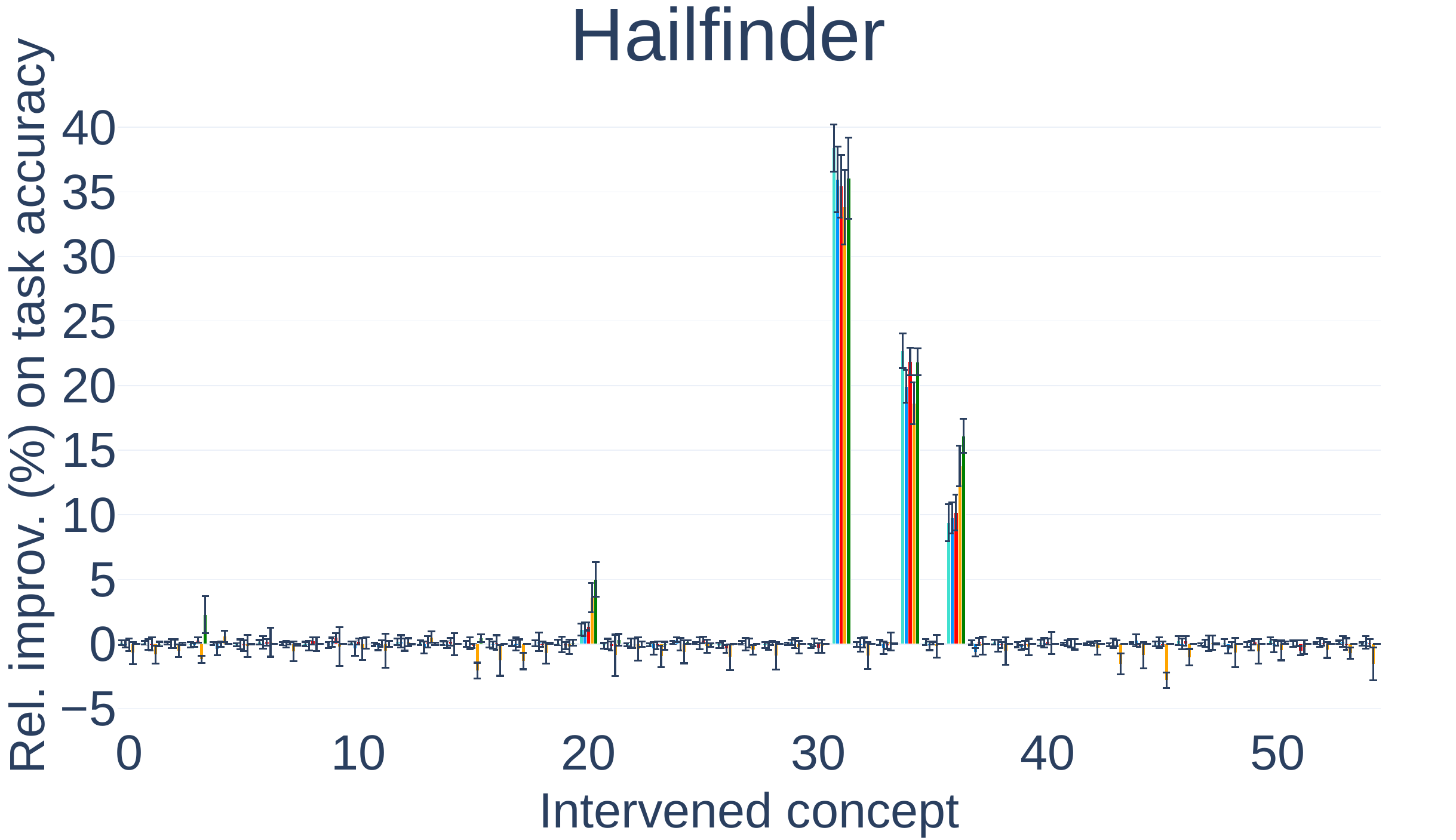}
    \includegraphics[width=0.49\linewidth]{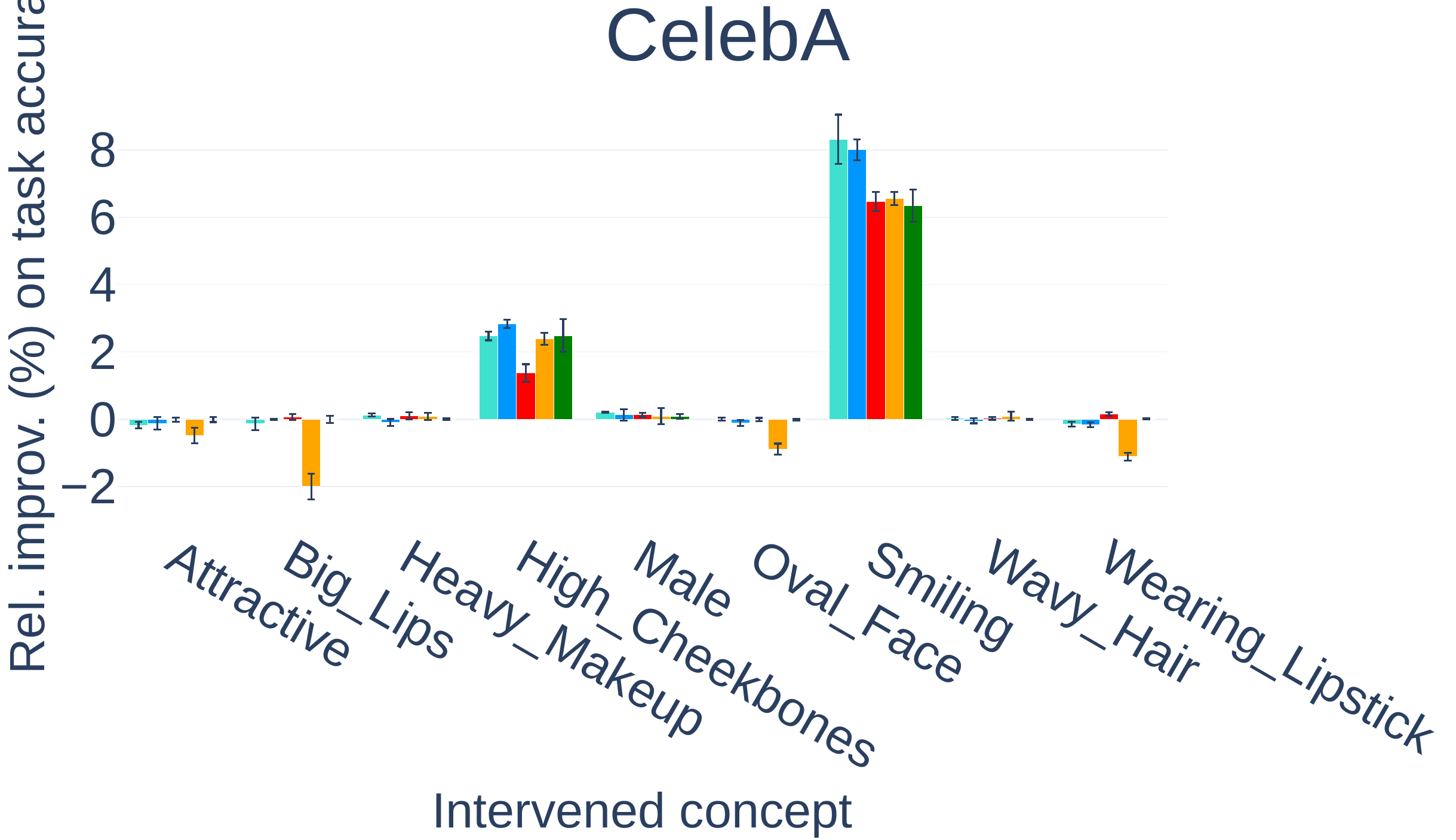}
    \includegraphics[width=0.49\linewidth]{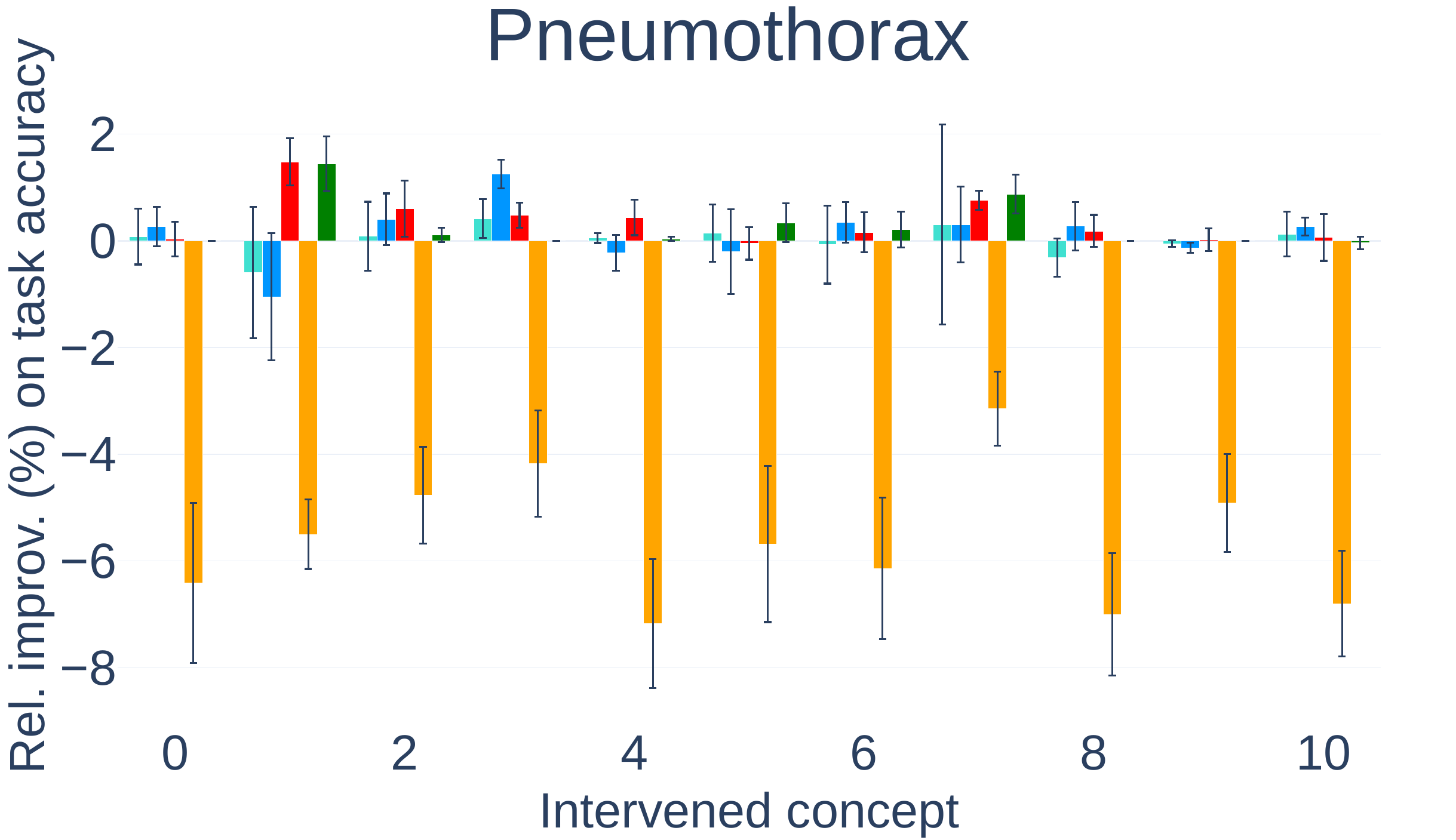}
    \includegraphics[width=0.49\linewidth]{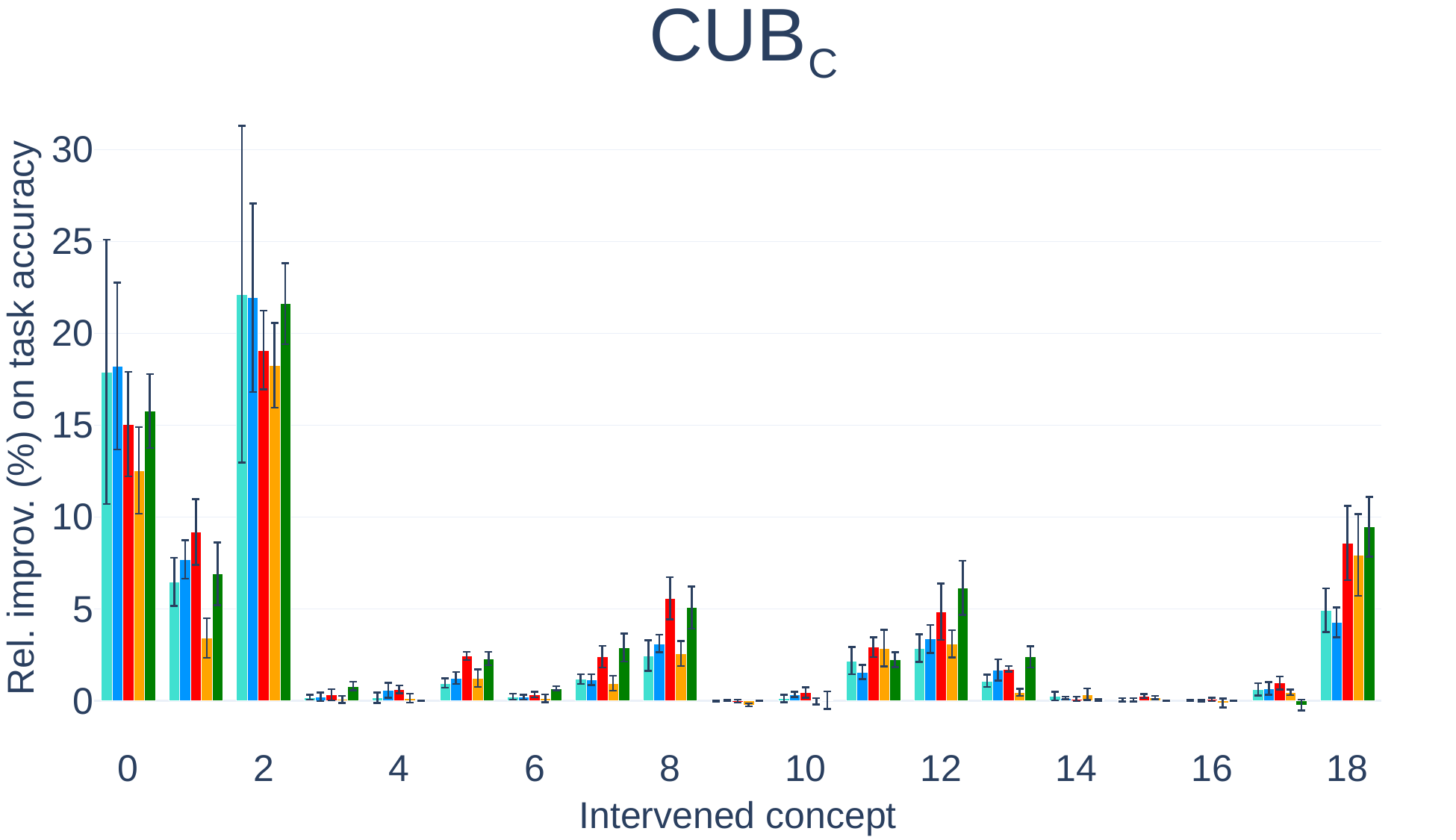}
    \caption{Relative improvement (\%) in task accuracy when intervening on specific concepts.}
    \label{fig:single_intervention_results}
\end{figure}
Fig.~\ref{fig:single_intervention_results} presents the relative improvement in task accuracy after intervening on individual concepts across all datasets. A key observation is that \textbf{C$^2$BM responds to interventions on the same key concepts as the baselines, despite the fundamental difference in how information propagates}. In CBM-based models and CEM, all concepts are directly connected to the task, enabling direct influence. In contrast, C$^2$BM enforces information flow through the causal graph, constraining the interactions. Yet, the task performance improvements remain consistent across models. These results highlight that C$^2$BM preserves the intervention effects observed in traditional concept-based models while providing a more structured and interpretable causal representation of the underlying relationships.

\subsection{Decomposing interventional accuracy}\label{app:decomposition}
Fig.~\ref{fig:interventions} in the main paper illustrates the improvement in cumulative relative interventional accuracy across all downstream, non-intervened concepts, including both intermediate concepts and the final task. To further analyze these effects, Fig.~\ref{fig:triplets_label_accuracy} decomposes this metric into two separate evaluations: one focusing solely on the task node and another considering only intermediate concepts. The concept accuracy plots highlight C$^2$BM’s unique ability to enhance performance of intermediate (downstream) concepts, a property not observed in competing models.

\begin{figure}[h]
    \centering
    \includegraphics[width=1\linewidth]{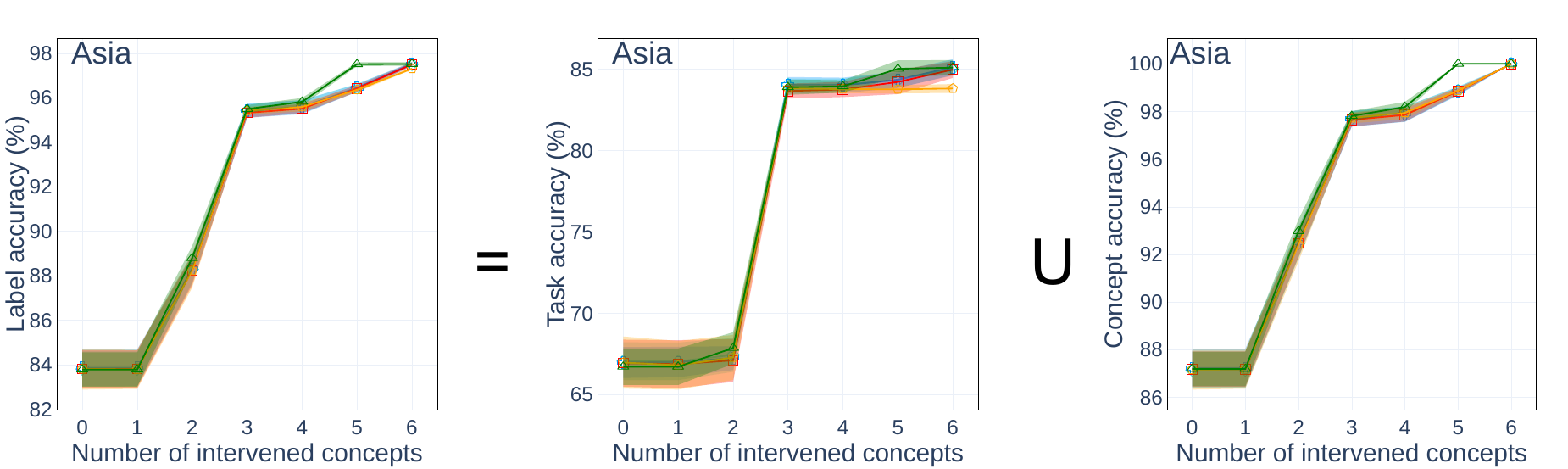}
%    \caption{Caption}
%    \label{fig:enter-label}
\end{figure}

\begin{figure}[h]
    \centering
    \includegraphics[width=1\linewidth]{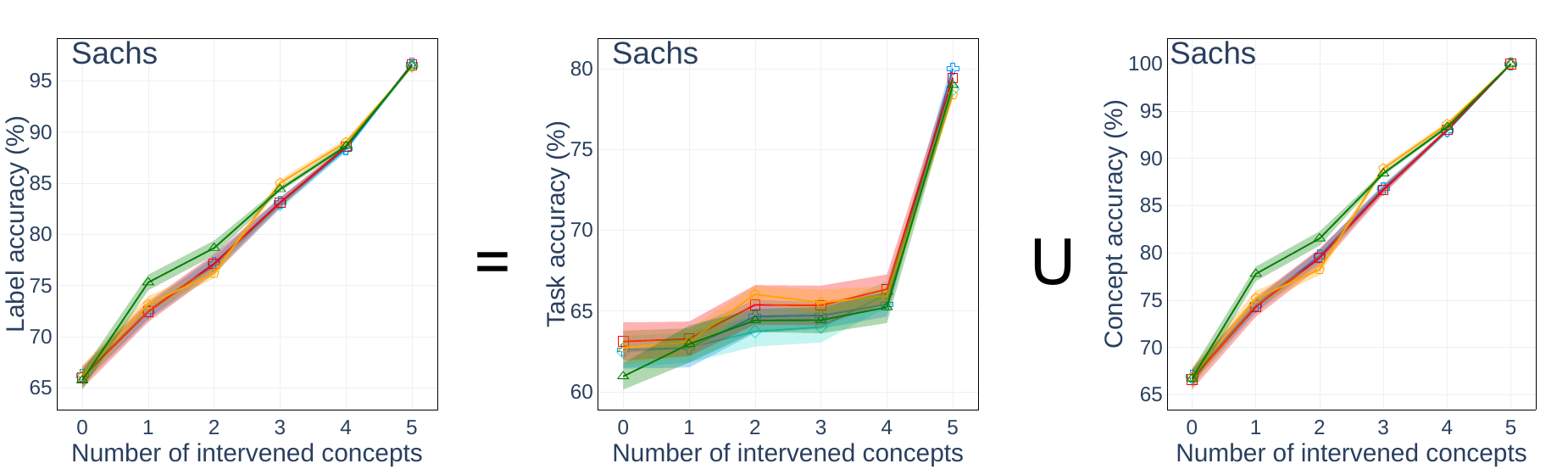}
%    \caption{Caption}
%    \label{fig:enter-label}
\end{figure}

\begin{figure}[h]
    \centering
    \includegraphics[width=1\linewidth]{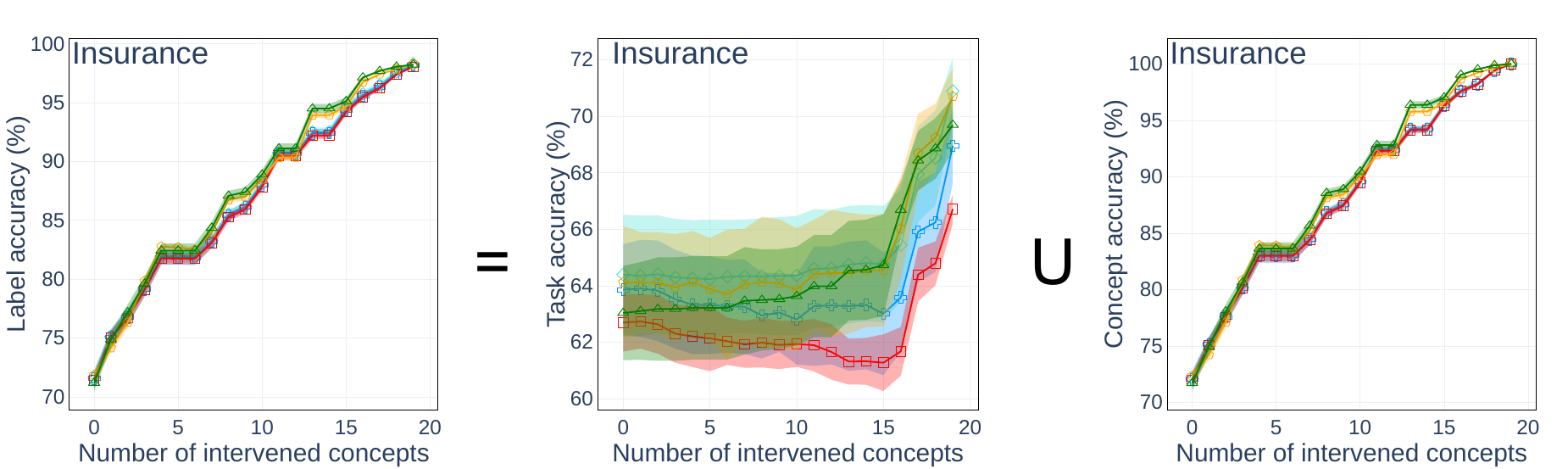}
%    \caption{Caption}
%    \label{fig:enter-label}
\end{figure}

\begin{figure}[h]
    \centering
    \includegraphics[width=1\linewidth]{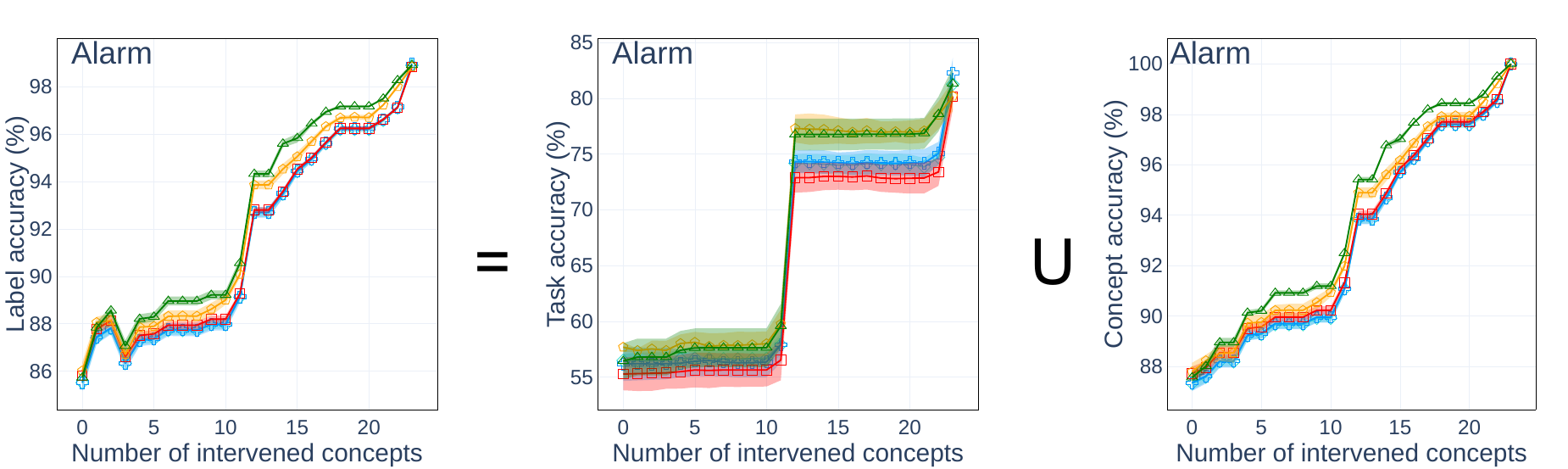}
%    \caption{Caption}
%    \label{fig:enter-label}
\end{figure}

\begin{figure}[h]
    \centering
    \includegraphics[width=1\linewidth]{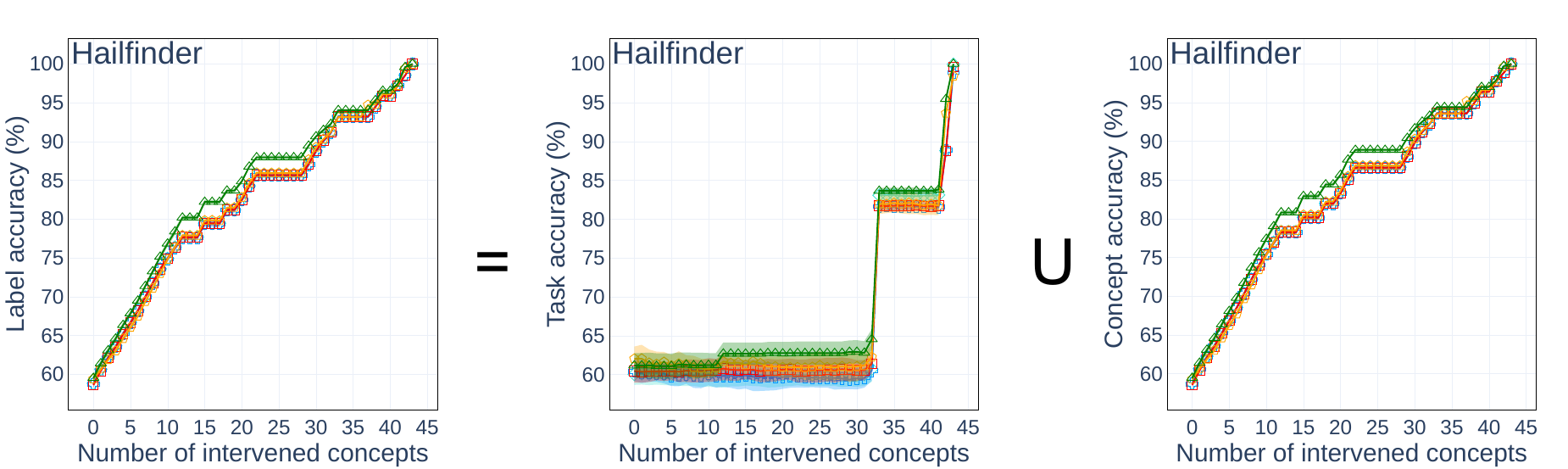}
%    \caption{Caption}
%    \label{fig:enter-label}
\end{figure}

\begin{figure}[h]
    \centering
    \includegraphics[width=1\linewidth]{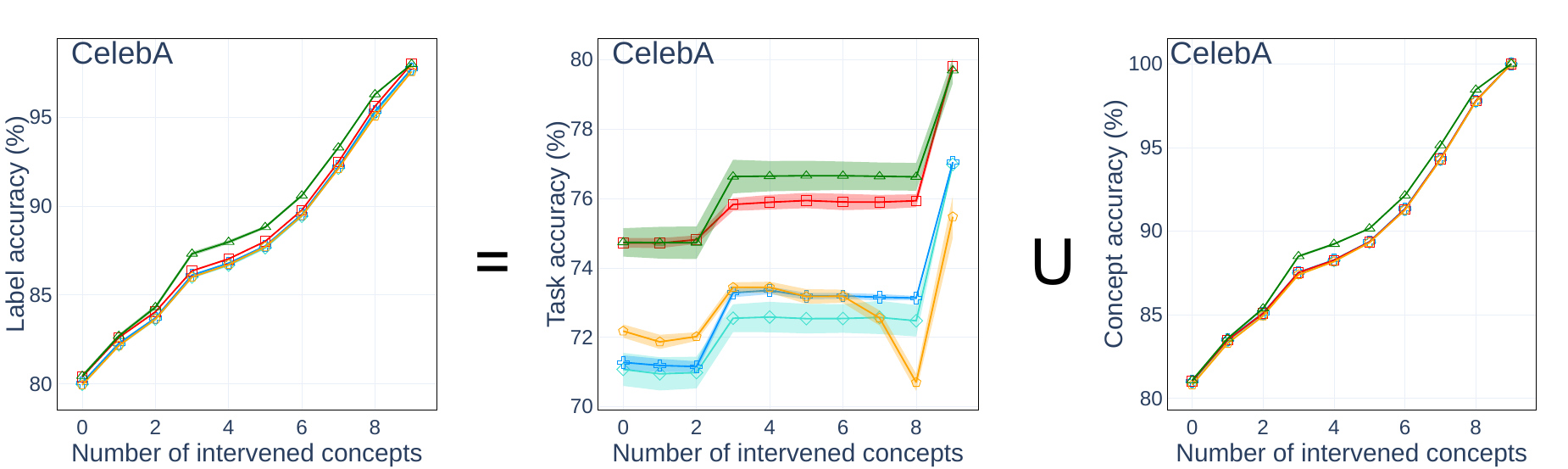}
%    \caption{Caption}
%    \label{fig:enter-label}
\end{figure}

\begin{figure}[h]
    \centering
    \includegraphics[width=1\linewidth]{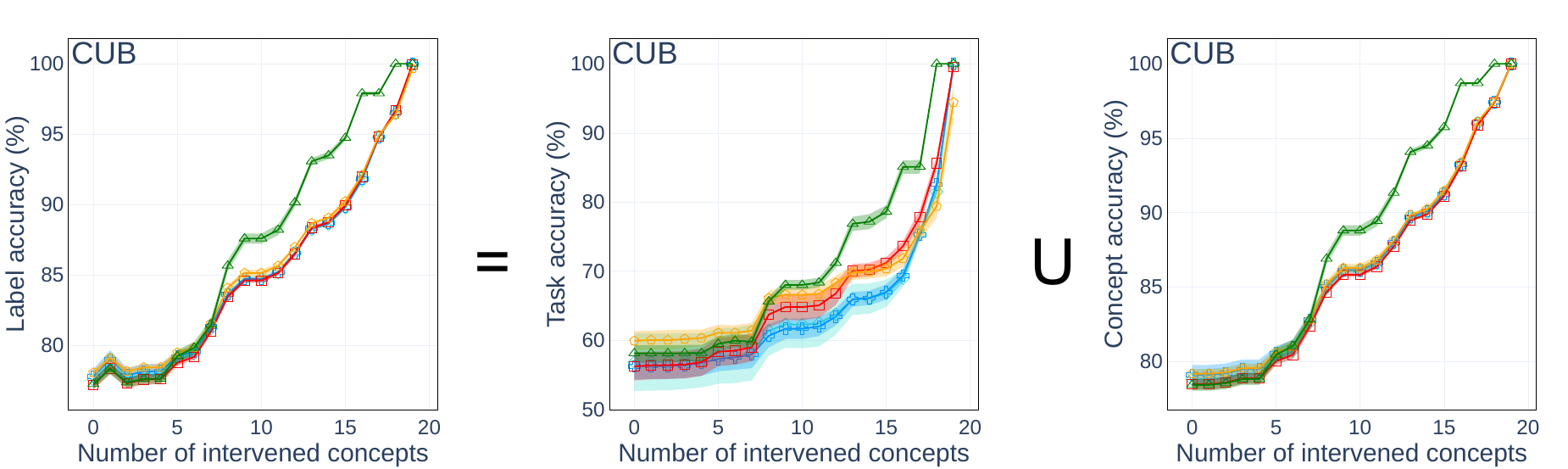}
%    \caption{Caption}
%    \label{fig:enter-label}
\end{figure}

\begin{figure}[h]
    \centering
    \includegraphics[width=1\linewidth]{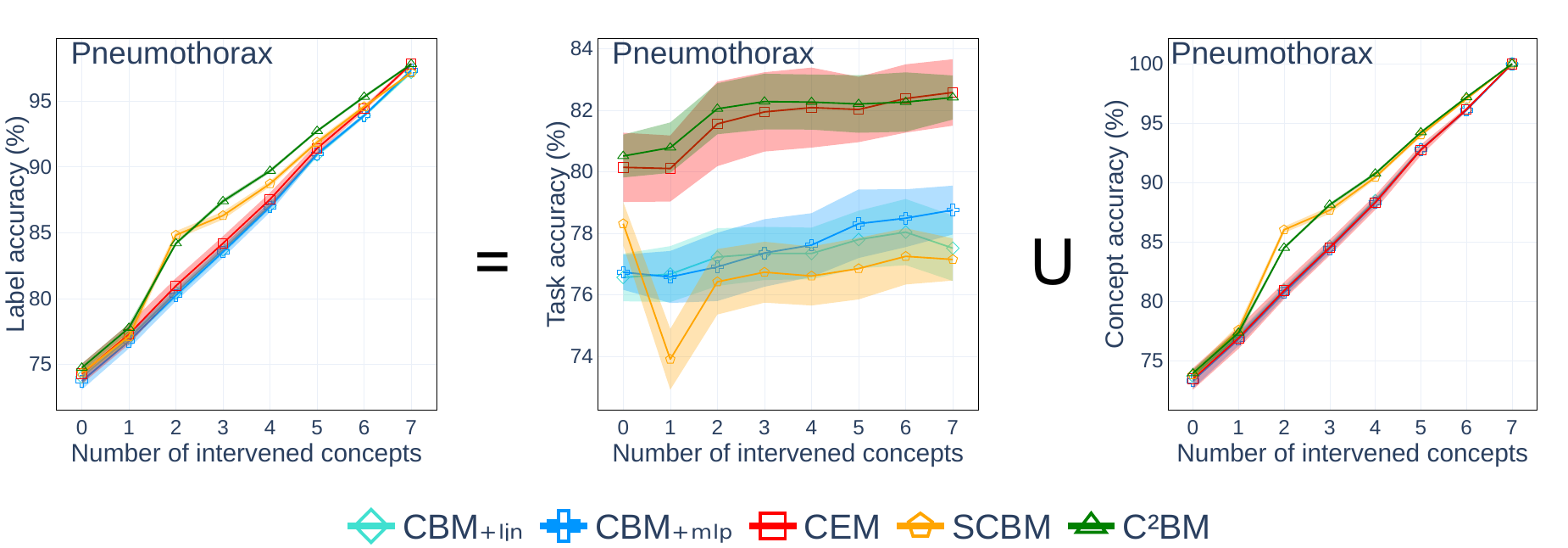}
    \caption{Task accuracy improvement (\%) in predicting downstream variables after intervening on groups of concepts up to progressively deeper levels in the graph hierarchy. The metric is averaged across all downstream variables (Left), the task only (Middle), and all concepts (Right) Total label accuracy. Uncertainties represent $2$ sample mean $\sigma$ across 5 runs.}
    \label{fig:triplets_label_accuracy}
\end{figure}

%%%%%%%%%%%%%%%%%%%%%%%%%%%%%%%%%%%%%%%%%%%%%%%%%%%%%%%%%%%%

\end{appendices}

\end{document}

%% file: tiks_figs.tex
\newlength{\nodesep}
\newlength{\innersep}
\newlength{\sep}
\tikzset{
    observed/.style={circle, draw, thick, fill=gray!50, minimum size=.8cm, text width=.8cm, inner sep=0pt, align=center},
    observeddirac/.style={circle, double, draw, thick, fill=gray!50, minimum size=.8cm, text width=.8cm, inner sep=0pt, align=center},
    observedred/.style={circle, draw, thick, fill=red!20, minimum size=.7cm},
    unobserved/.style={circle, draw, thick, minimum size=.8cm, text width=.8cm, inner sep=0pt, align=center},
    blank/.style={circle, draw, thick, dotted, fill=gray!50, minimum size=.7cm, opacity=0.5},
    plate/.style={draw, minimum size=.3cm},
    square1/.style={draw, fill=gray!20, minimum size=.3cm},
    square2/.style={draw, fill=gray!50, minimum size=.3cm},
    square3/.style={draw, fill=gray!100, minimum size=.3cm},
    true/.style = {circle, draw=none, thick, fill=green!50, minimum width=.5cm, minimum height=.5cm, inner sep=0pt},
    false/.style = {circle, draw=none, thick, fill=red!50, minimum width=.5cm, minimum height=.5cm, inner sep=0pt},
    plain/.style = {circle, draw, thick, fill=none, minimum width=.5cm, minimum height=.5cm, inner sep=0pt},
    arrowstyle/.style={->, thick, rounded corners},
    arrowstylered/.style={->, thick, rounded corners, red, dashed},
    arrowstylegray/.style={->, thick, rounded corners, gray, dashed},
    arrowstyledash/.style={->, thick, rounded corners, dashed pattern=on 2pt off .6pt, gray},
    adjweight/.style={->, dashed, line width=.1mm, color=blue!50},
    adjweightbig/.style={->, line width=.6mm, color=blue!50},
    cembtrue1/.style={draw, fill=green!20, minimum size=.3cm},
    cembtrue2/.style={draw, fill=green!50, minimum size=.3cm},
    cembtrue3/.style={draw, fill=green!100, minimum size=.3cm},
    cembfalse1/.style={draw, fill=red!20, minimum size=.3cm},
    cembfalse2/.style={draw, fill=red!50, minimum size=.3cm},
    cembfalse3/.style={draw, fill=red!100, minimum size=.3cm},
    operation/.style={draw, thick, fill=blue!20, minimum size=0.6cm,  rounded corners=5pt, inner sep=1pt, outer sep=0pt, align=center},
    box/.style={rectangle,draw,fill=green!20,text centered,rounded corners,thick},
}

\newcommand{\TWOGRAPHS}{

\begin{tikzpicture}[
    node distance = .3cm,
]

\node[anchor=south west] at (-1.4cm, -1.4cm) {\textbf{(a) Causal transparent graph}};
\node[unobserved] (e1) at (0,0) {$N$};
\node[unobserved] (c1) [above=1cm of e1] {$C$};
\node[unobserved] (n1) [right=1cm of e1] {$P$};

\draw[arrowstyle] (e1) to [] (n1);
\draw[arrowstyle] (c1) to [] (n1);

\node[anchor=south west] at (5.7cm, -1.4cm) {\textbf{(b) Causal reliable graph}};
\node[unobserved] (e2) at (7cm,0) {$N$};
\node[unobserved] (c2) [above=1cm of e2] {$C$};
\node[unobserved] (n2) [right=1cm of e2] {$P$};

\draw[arrowstyle] (e2) to [] (n2);

\end{tikzpicture}
}

%% file: acronyms.tex
\usepackage[acronym, nohypertypes={acronym}]{glossaries} 

\newacronym{mlp}{MLP}{Multilayer-Perceptron}

\glsdisablehyper

\newglossaryentry{asia}{name=Asia,description=}

%% file: tables/accuracy.tex
\begin{table*}[t]
    \renewcommand*{\arraystretch}{1.2}
        \caption{ Task accuracy (\%). Task concepts are as follows: dysp (Asia), Akt (Sachs), PropCost (Insurance), BP (Alarm), R5Fcst (Hailfinder), parity (cMNIST), mouth slightly open (CelebA), Survival (CUB$_C$), pneumothorax (Pneumoth.). $^*$ refers to reduced concept bottlenecks. Methods matching the performance of OpaQNN and showing a significant improvement over the other considered methods are highlighted in bold. %Methods that match the performance of OpaQNN with significant improvement w.r.t. other considered methods are in \textbf{bold}. 
        Uncertainties represent $2$ sample mean $\sigma$ across 5 runs.}
        \vspace{-0.3cm}
        \label{table:accuracy}
        \vskip 0.15in
        \begin{center}
            \begin{small}
                \begin{sc}
                    \resizebox{\textwidth}{!}{%
                    \begin{tabular}{l | c c | c c c c c c c c c | c c}
                    \toprule
                    Model & Semantic & Causal & Asia & Sachs & Insurance & Alarm & Hailfinder & cMNIST & CelebA & CUB$_C$ & Pneumoth. & Asia$^*$ & Alarm$^*$\\
                     & Transp. & Rel. & \multicolumn{9}{c}{} \\
                    \midrule 
                    OpaqNN           & {\color{BrickRed}\ding{56}} & {\color{BrickRed}\ding{56}} & $71.0_{\pm1.4}$ & $65.83_{\pm.71}$ & $66.8_{\pm1.5}$ & $62.8_{\pm1.5}$ & $72.0_{\pm1.9}$ & $91.24_{\pm.72}$ & $74.97_{\pm.08}$ & $60.3_{\pm0.9}$ & $80.0_{\pm1.5}$ & $71.0_{\pm1.4}$ & $62.8_{\pm1.5}$ \\
                    \midrule
                    CBM$_{+lin}$     & {\color{ForestGreen}\ding{52}} & {\color{BrickRed}\ding{56}} & $71.2_{\pm1.6}$ & $65.44_{\pm.93}$ & $67.1_{\pm1.7}$ & $62.7_{\pm1.3}$ & $72.2_{\pm2.3}$ & $93.92_{\pm .37}$ & $71.07_{\pm.48}$ & $56.8_{\pm4.2}$ & $76.6_{\pm0.8}$ & $56.0_{\pm8.3}$ & $52.7_{\pm1.0}$ \\
                    CBM$_{+mlp}$    & {\color{ForestGreen}\ding{52}} & {\color{BrickRed}\ding{56}} & $71.2_{\pm1.4}$ & $65.68_{\pm.84}$ & $66.7_{\pm1.4}$ & $62.3_{\pm1.8}$ & $70.9_{\pm2.2}$ & $93.55_{\pm.31}$ & $71.27_{\pm.20}$ & $56.2_{\pm1.9}$ & $76.7_{\pm0.6}$ & $58.6_{\pm2.7}$ & $52.8_{\pm1.2}$ \\
                    CEM              & {\color{ForestGreen}\ding{52}} & {\color{BrickRed}\ding{56}} & $71.1_{\pm1.8}$ & $65.93_{\pm.72}$ & $66.7_{\pm1.6}$ & $60.8_{\pm1.1}$ & $71.5_{\pm1.9}$ & $93.72_{\pm.26}$ & $\mathbf{74.72}_{\pm.14}$ & $56.2_{\pm2.0}$ & $\mathbf{80.1}_{\pm1.1}$ & $\mathbf{69.7}_{\pm2.0}$ & $\mathbf{61.8}_{\pm1.2}$ \\ 
                    SCBM             & {\color{ForestGreen}\ding{52}} & {\color{BrickRed}\ding{56}} & $70.7_{\pm1.6}$ & $66.30_{\pm.55}$ & $67.1_{\pm1.7}$ & $63.4_{\pm1.5}$ & $73.4_{\pm2.1}$ & $94.02_{\pm.23}$ & $72.15_{\pm.15}$ & $59.9_{\pm1.4}$ & $78.4_{\pm0.6}$ & $61.8_{\pm1.2}$ & $53.5_{\pm0.9}$ \\      
                    \textbf{C$^2$BM} & {\color{ForestGreen}\ding{52}} & {\color{ForestGreen}\ding{52}} & $71.4_{\pm1.7}$ & $65.33_{\pm1.1}$ & $66.4_{\pm1.5}$ & $62.5_{\pm1.4}$ & $74.1_{\pm1.8}$ & $94.18_{\pm.03}$ & $\mathbf{74.73}_{\pm.41}$ & $58.1_{\pm1.1}$ & $\mathbf{80.5}_{\pm0.7}$ & $\mathbf{70.8}_{\pm1.7}$ & $\mathbf{60.5}_{\pm1.4}$ \\
                    \bottomrule
                    \end{tabular}
                    }
                \end{sc}
            \end{small}
        \end{center}
    \end{table*}

%blackbox    $74.08_{\pm.41}$
%cbm_linear  $72.87_{\pm.28}$
%cbm_mlp     $72.78_{\pm.52}$
%cem         $74.36_{\pm.07}$
%crm         $74.34_{\pm.16}$

%blackbox    $92.08_{\pm.50}$   53.96 ± 0.13
%cbm_linear  $91.50_{\pm.39}$   53.36 ± 0.27
%cbm_mlp     $91.41_{\pm.41}$   53.09 ± 0.35
%cem         $91.91_{\pm.41}$   53.04 ± 0.26
%crm         $92.10_{\pm.23}$   52.62 ± 0.36

%label accuracy
%               asia_true    sachs_true insurance_true    alarm_true hailfinder_true          asia         sachs     insurance         alarm    hailfinder
%blackbox    87.94 ± 0.97  72.12 ± 0.35   78.73 ± 0.29  89.96 ± 0.18    68.21 ± 0.19  & $86.08_\pm{1.13}$ & $72.12_{\pm.35}$ &  $78.72_{\pm.23}$ &  $88.83_{\pm.17}$ & $64.96_{\pm.24}$ &
%cbm_linear  88.00 ± 0.97  70.43 ± 0.49   76.04 ± 0.19  88.65 ± 0.36    53.09 ± 0.27  & $86.14_\pm{1.13}$ & $70.43_{\pm.49}$ &  $76.37_{\pm.18}$ &  $87.17_{\pm.44}$ & $48.65_{\pm.34}$ &
%cbm_mlp     87.81 ± 0.94  70.22 ± 0.21   76.11 ± 0.22  88.55 ± 0.39    52.48 ± 0.51  & $85.93_\pm{1.10}$ & $70.22_{\pm.21}$ &  $76.44_{\pm.28}$ &  $87.08_{\pm.49}$ & $48.03_{\pm.47}$ &
%cem         87.89 ± 0.96  71.89 ± 0.37   79.52 ± 0.21  90.29 ± 0.09    69.96 ± 0.27  & $86.02_\pm{1.12}$ & $71.89_{\pm.37}$ &  $79.51_{\pm.13}$ &  $89.23_{\pm.11}$ & $66.68_{\pm.26}$ &
%crm         87.93 ± 0.93  72.00 ± 0.32   79.54 ± 0.27  90.38 ± 0.11    69.45 ± 0.18  & $86.09_\pm{1.12}$ & $71.75_{\pm.53}$ &  $79.44_{\pm.21}$ &  $89.23_{\pm.06}$ & $67.25_{\pm.25}$ &

%% file: tables/hamming.tex
\begin{table}[t]
    \renewcommand*{\arraystretch}{1.2}
        \caption{Structural Hamming distance (App.~\ref{app:hamming}) and number of mistaken edges between true and learned DAG. Reliability of standard flat CBMs is reported for reference. The total number of edges is in parentheses.}
        \label{table:hamming}
        \vskip 0.15in
        \begin{center}
            \begin{small}
                \begin{sc}
                    \resizebox{\columnwidth}{!}{%
                    \begin{tabular}{l | c | c c c c c c}
                    \toprule
                    Metric & after & cMNIST & Asia & Sachs & Insur. & Alarm & Hailf.  \\
                    \midrule
                    Hamming & flat CBM & 1.0 & 6.5 & 11.75 & 36.5 & 45.0 & 69.0  \\
                     & CD & 0.2 & 0.7 & 3.4 & 6.4 & 5.4 & \textbf{11.0}  \\
                     & CD + \tiny{LLM} & \textbf{0} & \textbf{0.3} & \textbf{1.8} & \textbf{6.3} & \textbf{5.0} & \textbf{11.0}  \\
                    \midrule
                    Incorrect    & flat CBM & 1 (1) &  11 (8) & 23 (17) & 74 (52)  & 78 (46) & 117 (66) \\
                    edges & CD & 1 (1) &  3 (8) & 17 (17) & 19 (52)  & 13 (46) & \textbf{22} (66) \\
                    (True edges) & CD + \tiny{LLM} & \textbf{0} (1) & \textbf{1} (8) & \textbf{7} (17) & \textbf{18} (52) & \textbf{10} (46) & \textbf{22} (66) \\
                    \bottomrule
                    \end{tabular}
                    }
                \end{sc}
            \end{small}
        \end{center}
\end{table}

%% file: tables/fairness_wrap.tex
    \renewcommand*{\arraystretch}{1.2}
        \caption{CelebA dataset. Causal Concept Effect (CaCE, \%) between a sensitive concept (\textit{Attractive}) and a target concept (\textit{Should be hired}), before and after blocking the path between the two variables with do-interventions.}
        \label{table:fairness}
        %\vskip 0.15in
        \begin{center}
            \begin{small}
                \begin{sc}
                    \resizebox{0.6\columnwidth}{!}{%
                    \begin{tabular}{l | c c c c c }
                    \toprule
                    CaCE metric & CBM$_{+lin}$ & CBM$_{+mlp}$ & CEM & SCBM & \textbf{C$^2$BM} \\
                    \midrule
                    Before int. & $12.3_{\pm1.1}$ & $12.5_{\pm3.1}$ & $19.0_{\pm3.6}$ & $30.9_{\pm4.1}$ & $25.1_{\pm1.8}$ \\
                    After int.  & $21.9_{\pm1.3}$ & $11.8_{\pm6.5}$ & $8.2_{\pm2.2}$ & $14.8_{\pm3.6}$ & $\textbf{0.0}_{\pm0.0}$ \\
                    \bottomrule
                    \end{tabular}
                    }
                \end{sc}
            \end{small}
        \end{center}

%% file: tables/cub_custom.tex
\begin{table}[h!]
    \centering
    \caption{Newly introduced CUB concepts and the rules used to determine their values.}
    \label{table:cub_concepts}
    \vskip 0.15in
    \begin{tabular}{l | p{6cm}}
        \toprule
        New Concept & Logical rules \\
        \midrule
        camouflage &  has\_tail\_pattern\_spotted $\lor$
        \newline
        has\_tail\_pattern\_striped $\lor$
        \newline
         has\_tail\_pattern\_multi-colored $\lor$ \newline
        has\_back\_pattern\_spotted $\lor$
        \newline
        has\_back\_pattern\_striped $\lor$
        \newline
        has\_back\_pattern\_multi-colored\\
        \midrule
        flight\_adaptation &  has\_tail\_shape\_rounded\_tail $\lor$ \newline has\_wing\_shape\_rounded-wings 
        $\lor$ \newline has\_size\_medium\\
        \midrule
        hunting\_ability & has\_bill\_shape\_curved
        $\lor$ \newline
        has\_bill\_shape\_needle
        $\lor$ \newline
    has\_bill\_shape\_spatulate
        $\lor$ \newline
        has\_bill\_shape\_all-purpose
        $\lor$ \newline
has\_bill\_shape\_longer\_than\_head
        $\lor$ \newline has\_bill\_shape\_shorter\_than\_head \\
        \midrule
        survival & max(camouflage + flight\_adaptation + hunting\_ability, 2)
        \\
        \bottomrule
    \end{tabular}
    \vspace{-0.3cm}
\end{table}

%% file: tables/accuracy_labels.tex
\begin{table*}[h]
    \renewcommand*{\arraystretch}{1.2}
        \caption{ Label accuracy (\%). Task concepts are as follows: dysp (Asia), Akt (Sachs), BP (Alarm), PropCost (Insurance), R5Fcst (Hailfinder), parity (cMNIST), mouth slightly open (CelebA), Survival (CUB$_C$), pneumothorax (Pneumoth.). Uncertainties represent $2$ sample mean $\sigma$ across 5 runs. 
        }
        \label{table:accuracy_concepts}
        \vskip 0.15in
        \begin{center}
            \begin{small}
                \begin{sc}
                    \resizebox{\textwidth}{!}{%
                    \begin{tabular}{l | c c | c c c c c c c c c}
                    \toprule
                    Model & Semantic & Causal & Asia & Sachs & Insurance & Alarm & Hailfinder & cMNIST & CelebA & CUB$_C$ & Pneumoth. \\
                     & Transp. & Rel. & \multicolumn{7}{c}{} \\
                    \midrule 
                    OpaqNN$_M$           & {\color{BrickRed}\ding{56}} & {\color{BrickRed}\ding{56}} & $87.1_{\pm1.1}$ & $72.4_{\pm1.0}$ & $78.76_{\pm0.73}$ & $90.13_{\pm0.30}$ & $66.29_{\pm0.60}$ & $91.85_{\pm.44}$ & $80.30_{\pm0.04}$ & $77.66_{\pm0.49}$ & $74.99_{\pm0.34}$ \\
                    CBM$_{+lin}$     & {\color{ForestGreen}\ding{52}} & {\color{BrickRed}\ding{56}} & $87.0_{\pm1.0}$ & $71.9_{\pm1.2}$ & $78.65_{\pm0.68}$ & $89.88_{\pm0.30}$ & $69.99_{\pm0.67}$ & $91.85_{\pm.44}$ & $79.98_{\pm0.14}$ & $77.51_{\pm0.43}$ & $74.03_{\pm0.44}$ \\
                    CBM$_{+mlp}$     & {\color{ForestGreen}\ding{52}} & {\color{BrickRed}\ding{56}} & $87.1_{\pm0.9}$ & $72.2_{\pm1.2}$ & $78.64_{\pm0.70}$ & $89.80_{\pm0.36}$ & $69.78_{\pm0.79}$ & $91.46_{\pm.40}$ & $80.06_{\pm0.03}$ & $77.75_{\pm0.54}$ & $73.72_{\pm0.65}$ \\
                    CEM              & {\color{ForestGreen}\ding{52}} & {\color{BrickRed}\ding{56}} & $87.1_{\pm0.9}$ & $72.2_{\pm1.1}$ & $78.60_{\pm0.77}$ & $89.77_{\pm0.28}$ & $70.14_{\pm0.76}$ & $91.42_{\pm.37}$ & $80.39_{\pm0.02}$ & $77.16_{\pm0.33}$ & $74.27_{\pm0.77}$ \\
                    SCBM             & {\color{ForestGreen}\ding{52}} & {\color{BrickRed}\ding{56}} & $87.0_{\pm0.9}$ & $72.2_{\pm1.0}$ & $78.62_{\pm0.75}$ & $90.08_{\pm0.32}$ & $69.54_{\pm0.65}$ & $91.84_{\pm.37}$ & $79.93_{\pm0.03}$ & $78.03_{\pm0.23}$ & $74.31_{\pm0.29}$ \\
                    \textbf{C$^2$BM} & {\color{ForestGreen}\ding{52}} & {\color{ForestGreen}\ding{52}} & $87.2_{\pm1.0}$ & $72.0_{\pm0.9}$ & $78.37_{\pm0.72}$ & $89.83_{\pm0.35}$ & $71.85_{\pm0.80}$ & $92.19_{\pm.05}$ & $80.44_{\pm0.10}$ & $77.22_{\pm0.36}$ & $74.73_{\pm0.31}$ \\
                    \bottomrule
                    \end{tabular}
                    }
                \end{sc}
            \end{small}
        \end{center}
    \end{table*}

%% file: tables/true_graph.tex
\begin{table}[h]
    \renewcommand*{\arraystretch}{1.2}
        \caption{Task accuracy (\%) using the \textit{true} and \textit{predicted} graph. Task concepts are as follows: dysp (Asia), Akt (Sachs), PropCost (Insurance), BP (Alarm), R5Fcst (Hailfinder). Uncertainties represent $2$ sample mean $\sigma$ across 5 runs.}
        \label{table:true-graph}
        \vskip 0.15in
        \begin{center}
                \begin{sc}
                    \resizebox{\textwidth}{!}{%
                    \begin{tabular}{l | c c c c c c }
                    \toprule
                    Model &  &  Asia    &     Sachs   &  Insurance   &      Alarm  &  Hailfinder \\
                    \midrule
                    C$^2$BM & True graph     & $71.0_{\pm2.3}$ & $65.0_{\pm1.3}$ & $66.4_{\pm1.8}$ & $62.1_{\pm1.5}$ & $73.0_{\pm1.6}$   \\
                            & Predicted graph  & $71.4_{\pm1.7}$ & $65.3_{\pm1.1}$ & $66.4_{\pm1.5}$ & $62.5_{\pm1.4}$ & $74.1_{\pm1.8}$   \\
                    \bottomrule
                    \end{tabular}
                    }
                \end{sc}
        \end{center}
\end{table}

%% file: tables/ablation_cd.tex
\begin{table}[h!]
    \renewcommand*{\arraystretch}{1.2}
    \vspace{-.2cm}
    \caption{Structural Hamming distance and number of mistaken edges between the true and learned causal graphs for each tested method. The maximum number of errors (Max errors) and the average number of errors make by a random classifier (Random Classifier) are also provided for reference.}
    \vspace{-.2cm}
    \label{table:cd_ablation}
    \vskip 0.15in
    \begin{center}
        \begin{scriptsize} % Reduced font size further
            \resizebox{0.8\columnwidth}{!}{%
                \begin{tabular}{l | c | c c c c c}
                    \toprule
                    Metric & CD Method & Asia & Sachs & Insurance & Alarm & Hailfinder  \\

                    \midrule
                      Hamming    & LLM & 13 & 11.08 & 336 & 505.92 & 92.75\\
                                 & LLM + RAG & 12 & 6.83 & 201.91 &  Time limit &  Time limit \\
                                 & PC & 2.41 & 2.93 & 6.65 & 5 & 14.53  \\
                                 & FGES & 0.65 & 3.4 & 6.78 & 5.91 & 21.4  \\
                                 & GES & 0.65 & 3.4 & 6.43 & 5.41 & \textbf{11} \\
                                 & PC + LLM (+ RAG) & 2.41 & 2.91 & 6 & \textbf{3.92} & 15.42  \\
                                 & FGES + LLM (+ RAG)& \textbf{0.25} & 2.03 & \textbf{5.33} & 6 & 20.5  \\
                                 & GES + LLM (+ RAG) & \textbf{0.25} & \textbf{1.83} & 6.33 & 4.95 & \textbf{11}  \\
                                 
                    \midrule

                  Number of mistaken   & Max errors & 28 & 55 & 351 & 666 & 1540\\ 
                  edges      & Random Classifier & 9.33 & 18.33 & 117 & 222 & 513.33 \\
                      & LLM & 13 & 20 & 351 & 543 & 107 \\
                             & LLM + RAG & 12 & 12 & 226 & Time limit &  Time limit \\
                             & PC & 6 & 10 & 26 & 13 & 45  \\
                             & FGES & 3 & 17 & 23 & 14 & 39  \\
                             & GES & 3 & 17 & 19 & 13 & \textbf{22} \\
                             & PC + LLM (+ RAG) & 6 & 9 & 22 & \textbf{9} & 43  \\
                             & FGES + LLM (+ RAG)& \textbf{1} & 8 & \textbf{15} & 11 & 36  \\
                             & GES + LLM (+ RAG)& \textbf{1} & \textbf{7} & 18 & 10 & \textbf{22}  \\
                    \bottomrule
                \end{tabular}
            }
        \end{scriptsize}
    \end{center}
    \vspace{-.3cm}
\end{table}

%% file: tables/ablation_llm.tex
\begin{table}[h!]
\centering
\caption{Structural Hamming distance and number of Mistaken Edges between true and learned DAG across datasets, using different LLMs and prompting strategies. Uncertainty represents $2$ sample mean $\sigma$ across 3 runs.}
\label{table:llm_ablation}

\begin{scriptsize}

\begin{tabular}{l|l|c|ccccc}
\toprule
Metric & LLM & Prompting Strategy & Asia & Sachs & Insurance & Alarm & Hailfinder \\
\midrule

\multirow{12}{*}{Hamming} 
  & \multirow{4}{*}{GPT-4o-mini} 
    & Minimal     & 0.25{\tiny{$\pm$.00}} & 2.67{\tiny{$\pm$.00}} & 8.04{\tiny{$\pm$.52}} & 3.63{\tiny{$\pm$.94}} & 10.83{\tiny{$\pm$.51}} \\
  & & Instruction & 0.25{\tiny{$\pm$.00}} & 2.75{\tiny{$\pm$.06}} & 8.04{\tiny{$\pm$.52}} & 4.3{\tiny{$\pm$1.2}} & 10.83{\tiny{$\pm$.51}} \\
  & & Few-shot    & 0.25{\tiny{$\pm$.00}} & 2.79{\tiny{$\pm$.05}} & 8.04{\tiny{$\pm$.52}} & 4.3{\tiny{$\pm$1.2}} & 10.83{\tiny{$\pm$.51}} \\
  & & CoT         & 0.25{\tiny{$\pm$.00}} & 2.21{\tiny{$\pm$.08}} & 7.67{\tiny{$\pm$.48}} & 4.1{\tiny{$\pm$1.1}} & 10.83{\tiny{$\pm$.51}} \\
\cmidrule(lr){2-8}
  & \multirow{4}{*}{GPT-4o} 
    & Minimal     & 0.25{\tiny{$\pm$.00}} & 2.83{\tiny{$\pm$.00}} & 8.04{\tiny{$\pm$.52}} & 3.75{\tiny{$\pm$.98}} & 10.83{\tiny{$\pm$.51}} \\
  & & Instruction & 0.25{\tiny{$\pm$.00}} & 2.63{\tiny{$\pm$.16}} & 8.04{\tiny{$\pm$.52}} & 4.7{\tiny{$\pm$1.3}} & 10.83{\tiny{$\pm$.51}} \\
  & & Few-shot    & 0.25{\tiny{$\pm$.00}} & 2.46{\tiny{$\pm$.05}} & 8.04{\tiny{$\pm$.52}} & 4.8{\tiny{$\pm$1.3}} & 10.83{\tiny{$\pm$.51}} \\
  & & CoT         & 0.25{\tiny{$\pm$.00}} & 2.25{\tiny{$\pm$.06}} & 7.54{\tiny{$\pm$.52}} & 4.7{\tiny{$\pm$1.3}} & 10.83{\tiny{$\pm$.51}} \\
\cmidrule(lr){2-8}
  & \multirow{4}{*}{GPT-5} 
    & Minimal     & 0.25{\tiny{$\pm$.00}} & 1.29{\tiny{$\pm$.05}} & 7.92{\tiny{$\pm$.48}} & 4.2{\tiny{$\pm$1.1}} & 10.83{\tiny{$\pm$.51}} \\
  & & Instruction & 0.25{\tiny{$\pm$.00}} & 1.21{\tiny{$\pm$.02}} & 7.29{\tiny{$\pm$.43}} & 4.4{\tiny{$\pm$1.2}} & 10.83{\tiny{$\pm$.51}} \\
  & & Few-shot    & 0.25{\tiny{$\pm$.00}} & 1.04{\tiny{$\pm$.05}} & 7.29{\tiny{$\pm$.43}} & 4.0{\tiny{$\pm$1.1}} & 10.83{\tiny{$\pm$.51}} \\
  & & CoT         & 0.25{\tiny{$\pm$.00}} & 0.92{\tiny{$\pm$.09}} & 7.17{\tiny{$\pm$.39}} & 4.6{\tiny{$\pm$1.3}} & 10.83{\tiny{$\pm$.51}} \\
\midrule

\multirow{12}{*}{\shortstack{Mistaken \\ Edges}} 
  & \multirow{4}{*}{GPT-4o-mini} 
    & Minimal     & 1.00{\tiny{$\pm$.00}} & 9.00{\tiny{$\pm$.00}} & 23.0{\tiny{$\pm$1.4}} & 9.0{\tiny{$\pm$1.8}} & 21.50{\tiny{$\pm$.90}} \\
  & & Instruction & 1.00{\tiny{$\pm$.00}} & 9.50{\tiny{$\pm$.18}} & 23.0{\tiny{$\pm$1.4}} & 10.0{\tiny{$\pm$2.2}} & 21.50{\tiny{$\pm$.90}} \\
  & & Few-shot    & 1.00{\tiny{$\pm$.00}} & 9.50{\tiny{$\pm$.18}} & 23.0{\tiny{$\pm$1.4}} & 10.0{\tiny{$\pm$2.2}} & 21.50{\tiny{$\pm$.90}} \\
  & & CoT         & 1.00{\tiny{$\pm$.00}} & 8.00{\tiny{$\pm$.36}} & 21.5{\tiny{$\pm$1.3}} & 9.0{\tiny{$\pm$1.8}} & 21.50{\tiny{$\pm$.90}} \\
\cmidrule(lr){2-8}
  & \multirow{4}{*}{GPT-4o} 
    & Minimal     & 1.00{\tiny{$\pm$.00}} & 11.00{\tiny{$\pm$.00}} & 23.0{\tiny{$\pm$1.4}} & 9.5{\tiny{$\pm$2.0}} & 21.50{\tiny{$\pm$.90}} \\
  & & Instruction & 1.00{\tiny{$\pm$.00}} & 10.00{\tiny{$\pm$.72}} & 23.0{\tiny{$\pm$1.4}} & 10.0{\tiny{$\pm$2.2}} & 21.50{\tiny{$\pm$.90}} \\
  & & Few-shot    & 1.00{\tiny{$\pm$.00}} & 9.50{\tiny{$\pm$.18}} & 23.0{\tiny{$\pm$1.4}} & 11.0{\tiny{$\pm$2.3}} & 21.50{\tiny{$\pm$.90}} \\
  & & CoT         & 1.00{\tiny{$\pm$.00}} & 8.50{\tiny{$\pm$.18}} & 21.0{\tiny{$\pm$1.4}} & 10.0{\tiny{$\pm$2.2}} & 21.50{\tiny{$\pm$.90}} \\
\cmidrule(lr){2-8}
  & \multirow{4}{*}{GPT-5} 
    & Minimal     & 1.00{\tiny{$\pm$.00}} & 4.50{\tiny{$\pm$.18}} & 22.5{\tiny{$\pm$1.3}} & 9.5{\tiny{$\pm$2.0}} & 21.50{\tiny{$\pm$.90}} \\
  & & Instruction & 1.00{\tiny{$\pm$.00}} & 4.00{\tiny{$\pm$.00}} & 20.0{\tiny{$\pm$1.1}} & 9.0{\tiny{$\pm$1.8}} & 21.50{\tiny{$\pm$.90}} \\
  & & Few-shot    & 1.00{\tiny{$\pm$.00}} & 3.50{\tiny{$\pm$.18}} & 20.0{\tiny{$\pm$1.1}} & 9.0{\tiny{$\pm$1.8}} & 21.50{\tiny{$\pm$.90}} \\
  & & CoT         & 1.00{\tiny{$\pm$.00}} & 3.00{\tiny{$\pm$.36}} & 19.5{\tiny{$\pm$.9}} & 9.5{\tiny{$\pm$2.0}} & 21.50{\tiny{$\pm$.90}} \\
\bottomrule
\end{tabular}

\end{scriptsize}
\end{table}

%% file: tables/ablation_rag.tex
\begin{table}[h!]
    \renewcommand*{\arraystretch}{1.2}
    \vspace{-.2cm}
    \caption{Structural Hamming distance and number of mistaken edges between true and learned DAG when using either RAG to provide context to the LLM or just the LLM.}
    \vspace{-.2cm}
    \label{table:rag_ablation}
    \vskip 0.15in
    \begin{center}
        \begin{scriptsize} % Reduced font size further
            \resizebox{0.5\columnwidth}{!}{%
                \begin{tabular}{l | c | c c}
                    \toprule
                    Metric & Context & Sachs & Alarm  \\
                    \midrule
                    Hamming & No context & 3.1 & \textbf{5.0} \\
                             & RAG context & \textbf{1.8} & \textbf{5.0} \\
                    \midrule
                    Mistaken & No context & 12 & \textbf{10} \\
                    edges ratio & RAG context & \textbf{7} & \textbf{10} \\
                    \bottomrule
                \end{tabular}
            }
        \end{scriptsize}
    \end{center}
    \vspace{-.3cm}
\end{table}

%% file: tables/sensitivity_data_size.tex
\begin{table}[h]
    \renewcommand*{\arraystretch}{1.2}
    \centering
    \caption{Number of mistaken edges between the true DAG and the learned graph using the C$^2$BM's graph construction pipeline, evaluated across datasets and data sizes. For reference, the reliability of standard flat CBMs is also reported.}
    \begin{tabular}{l c c c c c c c}
        \toprule
        {Dataset} & Flat CBMs & & \multicolumn{5}{c}{{C$^2$BM's causal graph}} \\
        \cline{4-8}
         & & Data size ($N$) $\rightarrow$ & {100} & {500} & {1000} & {5000} & {10000 (Paper)} \\
        \midrule
        Asia       & 11  & & 6  & 3  & 1  & 1  & 1  \\
        Sachs      & 23  & & 15 & 12 & 11 & 7  & 7  \\
        Insurance  & 74  & & 47 & 35 & 24 & 22 & 18 \\
        Alarm      & 78  & & 31 & 16 & 9  & 9  & 9  \\
        Hailfinder & 117 & & 56 & 47 & 44 & 22 & 22 \\
        \bottomrule
    \end{tabular}
    \label{tab:sensitivity-data-size}
\end{table}

%% file: tables/sensitivity_corruption_operations.tex
\begin{table}[ht]
    \renewcommand*{\arraystretch}{1.1}
    \centering
    \caption{Task accuracy (\%) under edge-level adversarial corruptions. A percentage $p$ of edges is altered through flipping, addition, or removal. Task concepts are as follows: dysp (Asia), Akt (Sachs), PropCost (Insurance), BP (Alarm), R5Fcst (Hailfinder). Uncertainties represent $2$ sample mean $\sigma$ across 3 runs.}
    \label{tab:graph_corruption_edges}
    \begin{center}
        \begin{scriptsize} % Reduced font size further
            \resizebox{\columnwidth}{!}{%
    \begin{tabular}{l c c c c c c c}
        \toprule
        Dataset / $p$ & 0.05 & 0.1 & 0.2 & 0.4 & 0.6 & 0.8 & 1.0 \\
        \midrule
        Asia       & 71.8$_{\pm 0.6}$ & 71.6$_{\pm 0.5}$ & 71.4$_{\pm 0.6}$ & 71.7$_{\pm 1.3}$ & 70.2$_{\pm 1.9}$ & 71.4$_{\pm 1.2}$ & 70.2$_{\pm 0.9}$ \\
        Sachs      & 65.1$_{\pm 1.9}$ & 65.6$_{\pm 2.0}$ & 64.9$_{\pm 2.0}$ & 64.6$_{\pm 1.2}$ & 65.5$_{\pm 1.5}$ & 65.2$_{\pm 1.5}$ & 65.0$_{\pm 2.0}$ \\
        Insurance  & 67.2$_{\pm 1.6}$ & 67.0$_{\pm 3.1}$ & 67.5$_{\pm 2.3}$ & 66.1$_{\pm 1.1}$ & 67.0$_{\pm 2.6}$ & 66.8$_{\pm 2.6}$ & 67.2$_{\pm 3.1}$ \\
        Alarm      & 61.9$_{\pm 2.6}$ & 61.7$_{\pm 2.1}$ & 61.5$_{\pm 2.5}$ & 61.9$_{\pm 2.7}$ & 60.6$_{\pm 1.9}$ & 60.9$_{\pm 1.6}$ & 61.4$_{\pm 1.7}$ \\
        Hailfinder & 74.0$_{\pm 1.5}$ & 73.1$_{\pm 3.0}$ & 72.3$_{\pm 2.2}$ & 72.3$_{\pm 3.2}$ & 73.0$_{\pm 2.1}$ & 71.8$_{\pm 1.6}$ & 72.4$_{\pm 1.4}$ \\
        \bottomrule
    \end{tabular}
            }
        \end{scriptsize}
    \end{center}
\end{table}

%% file: tables/sensitivity_corruption_cbms.tex
\begin{table}[ht]
    \renewcommand*{\arraystretch}{1.1}
    \centering
    \caption{Task accuracy (\%) under progressive graph flattening into standard CBMs. A percentage $p$ of nodes is directly connected to the task output. Task concepts are as follows: dysp (Asia), Akt (Sachs), PropCost (Insurance), BP (Alarm), R5Fcst (Hailfinder). Uncertainties represent $2$ sample mean $\sigma$ across 3 runs.}
    \label{tab:graph_flattening}
    \begin{center}
        \begin{scriptsize} % Reduced font size further
            \resizebox{\columnwidth}{!}{%
    \begin{tabular}{l c c c c c c c}
        \toprule
        Dataset / $p$ & 0.05 & 0.1 & 0.2 & 0.4 & 0.6 & 0.8 & 1.0 \\
        \midrule
        Asia       & 72.0$_{\pm 1.0}$ & 72.0$_{\pm 1.0}$ & 71.5$_{\pm 0.3}$ & 71.7$_{\pm 0.8}$ & 71.7$_{\pm 0.5}$ & 71.2$_{\pm 1.6}$ & 71.1$_{\pm 1.1}$ \\
        Sachs      & 65.5$_{\pm 2.2}$ & 64.8$_{\pm 3.0}$ & 65.8$_{\pm 1.6}$ & 64.7$_{\pm 1.5}$ & 65.2$_{\pm 1.2}$ & 65.1$_{\pm 2.7}$ & 64.9$_{\pm 2.2}$ \\
        Insurance  & 67.4$_{\pm 2.6}$ & 66.2$_{\pm 1.7}$ & 66.3$_{\pm 2.5}$ & 65.6$_{\pm 3.3}$ & 65.3$_{\pm 2.4}$ & 66.9$_{\pm 1.8}$ & 64.7$_{\pm 3.1}$ \\
        Alarm      & 62.6$_{\pm 2.9}$ & 61.6$_{\pm 2.6}$ & 61.6$_{\pm 2.0}$ & 60.5$_{\pm 2.6}$ & 61.2$_{\pm 1.8}$ & 60.6$_{\pm 3.2}$ & 61.3$_{\pm 1.9}$ \\
        Hailfinder & 73.1$_{\pm 1.1}$ & 73.2$_{\pm 1.9}$ & 72.5$_{\pm 0.9}$ & 72.7$_{\pm 2.9}$ & 73.0$_{\pm 2.6}$ & 73.0$_{\pm 1.7}$ & 72.9$_{\pm 1.9}$ \\
        \bottomrule
    \end{tabular}
            }
        \end{scriptsize}
    \end{center}
\end{table}